\documentclass[a4paper,fleqn]{cas-dc}

\usepackage{pgfplots}
\pgfplotsset{compat=1.18}
\usepackage{moreverb,url}

\usepackage{adjustbox}
\usepackage{xcolor}
\usepackage{color}
\usepackage{caption}

\usepackage{amsfonts}

\usetikzlibrary{positioning}

\newcommand\BibTeX{{\rmfamily B\kern-.05em \textsc{i\kern-.025em b}\kern-.08em
T\kern-.1667em\lower.7ex\hbox{E}\kern-.125emX}}

\usepackage{float}
\usepackage{amsthm}
\usepackage{subcaption}
\usepackage{amsmath}
\usepackage{graphicx}

\usepackage{nomencl}
\makenomenclature

\usepackage[numbers]{natbib}

\def \SE {\textbf{SE}(3)}
\def \SO {\textbf{SO}(3)}
\def \se {\textbf{se}(3)}
\def \so {\textbf{so}(3)}
\def \R  {\mathbb{R}}
\def \g  {\mathfrak{g}}
\def \q  {\mathfrak{q}}
\def \h  {\mathfrak{h}}

\def \OOs {\Omega_\orbit}
\def \m  {\mathfrak{m}}
\def \E  {\mathfrak{e}}
\def \t  {\mathfrak{t}}
\def \V {\mathcal{V}}

\def \Ad {\textbf{Ad}}
\def \ad {\textbf{ad}}
\def \Add {\mathfrak{Ad}}
\def \L {\mathcal{L}}
\def \A {\mathcal{A}}

\def \sin {\text{sin}}
\def \cos {\text{cos}}

\def \orbit {\odot}

\def \diag {\text{diag}}
\def \bO {\mathbb{O}}

\def\tsc#1{\csdef{#1}{\textsc{\lowercase{#1}}\xspace}}
\tsc{WGM}
\tsc{QE}
\tsc{EP}
\tsc{PMS}
\tsc{BEC}
\tsc{DE}

\newtheorem{remark}{Remark}
\theoremstyle{remark}
\newtheorem{definition}{Definition}
\newtheorem{theorem}{Theorem}
\newtheorem{lemma}{Lemma}

\newtheorem{Assumption}{Assumption}
\renewcommand{\qedsymbol}{$\blacksquare$}

\begin{document}
\sloppy
\let\WriteBookmarks\relax
\def\floatpagepagefraction{1}
\def\textpagefraction{.001}
\shorttitle{Lagrange-Poincar\'{e}-Kepler Equations of Disturbed Space-Manipulator Systems in Orbit}
\shortauthors{Borna Monazzah Moghaddam and Robin Chhabra}

\title [mode = title]{Lagrange-Poincar\'{e}-Kepler Equations of Disturbed Space-Manipulator Systems in Orbit}
\tnotemark[1]
\tnotetext[1]{This work was partially supported by the Natural Sciences and Engineering Research Council of Canada and the Canada Research Chair Program.}

\author{Borna Monazzah Moghaddam}

% \address{Autonomous Space Robotics and Mechatronics Laboratory, Carleton University, Ottawa, ON K1S 5B6, Canada}

\author{Robin Chhabra}[orcid=0000-0001-7511-2910]

\cormark[1]
\ead{Email: robin.chhabra@torontomu.ca}

\cortext[cor1]{Corresponding author at: Department of Mechanical, Industrial, and Mechatronics Engineering, Toronto Metropolitan University, Toronto, ON M5B 2K3, Canada, Tel.:	 416-979-5000 ext. 554097}
%https://orcid.org/0000-0003-0023-1928

\begin{abstract}
This article presents an extension of the Lagrange-Poincaré Equations (LPE) to model the dynamics of spacecraft-manipulator systems operating within a non-inertial orbital reference frame. Building upon prior formulations of LPE for vehicle-manipulator systems, the proposed framework—termed the Lagrange-Poincaré-Kepler Equations (LPKE)—incorporates the coupling between spacecraft attitude dynamics, orbital motion, and manipulator kinematics.
The formalism combines the Euler-Poincaré equations for the base spacecraft, Keplerian orbital dynamics for the reference frame, and reduced Euler-Lagrange equations for the manipulator’s shape space, using an exponential joint parametrization. Leveraging the Lagrange-d’Alembert principle on principal bundles, we derive novel closed-form structural matrices that explicitly capture the effects of orbital disturbances and their dynamic coupling with the manipulator system.
The LPKE framework also systematically includes externally applied, symmetry-breaking wrenches,  allowing for immediate integration into hardware-in-the-loop simulations and model-based control architectures for autonomous robotic operations in the orbital environment . To illustrate the effectiveness of the proposed model and its numerical superiority, we present a simulation study analyzing orbital effects on a 7-degree-of-freedom  manipulator mounted on a spacecraft. 
\end{abstract}

%\begin{highlights}\item The Guidance, Navigation and Control (GNC) methodologies used in various phases of in-orbit robotic missions performed by a spacecraft-manipulator system are reviewed.\item  The GNC methodologies are presented in a unifying manner to provide a systematic comparison ground and list their advantages and disadvantages.\item Two new families of approaches are introduced that can noticeably improve GNC of spacecraft-manipulator systems.\end{highlights}

\begin{keywords}
Guidance, Navigation and Control\sep Space Robotics \sep On-Orbit Servicing\sep Artificial Intelligence\sep Geometric Mechanics
\end{keywords}

\maketitle

\section{Introduction}
 On-Orbit Servicing (OOS) (i.e. satellite repair, refueling, assembly, and debris mitigation) stands to benefit substantially from dexterous space manipulators (Fig.\ref{fig:orbitalmanipulator}), provided spacecraft mobility is managed robustly in the harsh orbital environment \citep{moghaddam2021guidance,zhang2013output,zong2020concurrent}. While in-space demonstrations of space manipulators remain limited, OOS is transitioning from concept to capability: operations with Canadarm2 aboard the ISS have validated microgravity manipulation on a very massive and inertia-dominant base, conditions under which base–arm coupling is comparatively subdued \citep{gibbs2002canada}.  Ongoing programs (e.g., planned European orbital servicing initiatives such as EROSS ~\citep{dubanchet2020eross,andiappane2019mission,roa2024eross,dubancheta2021eross}) signal rapid maturation but also underscore the need for principled dynamics and control frameworks tailored to on-orbit interaction \citep{janousek2021eros,flores2014review}. Compared with other proximity-operation technologies, robotic manipulators (predominantly teleoperated or supervised at the moment) offer predictable interaction control for docking, capture, and payload handling \citep{king2001space,aikenhead1983canadarm,sheridan1993}, yet safe performance hinges on accounting for orbital disturbances and non-inertial coupling between orbital motion, base attitude, and arm kinematics \citep{wee1997articulated,nenchev1992analysis,zong2021optimal,huang2006tracking,moghaddam2021guidance}.  
 In contrast to manipulators in micro-gravity (e.g. CanadArm), free-flying space manipulators experience strong reaction dynamics and pronounced orbit–attitude–arm coupling, motivating modeling tools that remain reliable in the absence of the stabilizing inertia of large platforms.

 Hence, to facilitate the robust and reliable operation of space manipulators, the interaction between the manipulator, the spacecraft, and the orbital motion must be carefully captured to ensure stability and precision \citep{sakawa1999trajectory}. A wide range of modeling and control methodologies aim to capture realistic orbital disturbances, e.g. relative motion with respect to an orbital reference frame, gravitational perturbations, and solar radiation pressure \citep{prussing1993orbit, kaplan1976modern}. 
Clohessy-Wiltshire equations (CW) are the most widely used model-based approach for capturing relative orbital dynamics (especially in rendezvous and synchronization scenarios \citep{Padial,kawano2001result,clohessy1960terminal,benninghoff2014autonomous}). Historically, orbital motion has been treated as a separate layer in modeling of space manipulator dynamics: either via linearized relative dynamics (CW and variants) \citep{clohessy1960terminal} or by direct two-body propagation in parallel with manipulator/base dynamics \citep{xu2024orbit,carr2022coupled,sincarsin1983gravitational,jiang2017integrated,curtis2013orbital}.
\begin{figure}
	\centering
		\includegraphics[scale=.45]{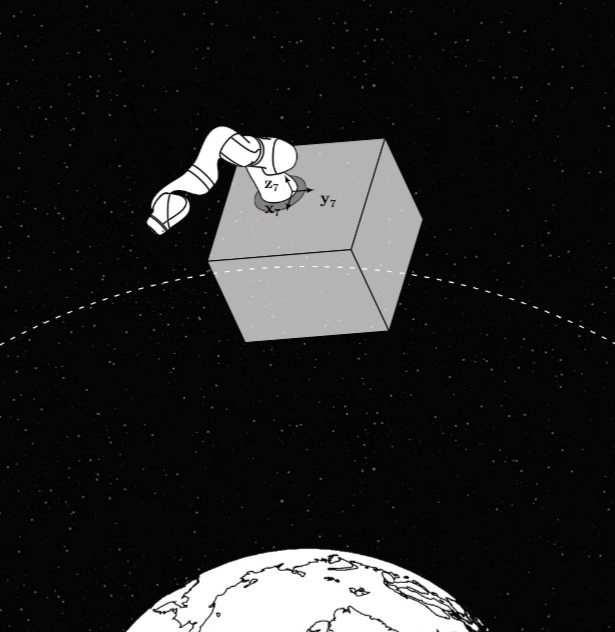}
	\caption{Space Manipulator in Orbit}
	\label{fig:orbitalmanipulator}
\end{figure}For in-orbit proximity operations, virtual-fixed-base viewpoints shaped early space-manipulator modeling and control approaches, including: the Generalized Jacobian Matrix (GJM) for base-arm coupling by Yoshida and Umetani \citep{umetani1987continuous}, and the virtual manipulator approach  by Vafa and Dubowsky \citep{vafa1987dynamics,dubowsky1993kinematics}. Classical modeling of the dynamics of mobile-manipulator systems has relied on recursive Newton-Euler algorithms \citep{muralidharan2019rendezvous,huang2017vehicle}, Lagrangian formulations \citep{dubowsky1993kinematics,antonelli2004adaptive}, and quasi-Lagrangian approaches \citep{olguin2013passivity,sarkar2001coordinated}. 
For the kinematics of open-chain multi-body manipulators the Product-of-Exponentials (POE) was introduced by Brockett, developed by Murray, Li, and Sastry \citep{brockett1984robotic,murray1994mathematical,park1994computational}, embedded into recursive Newton–Euler and Euler–Lagrange dynamics by Park \citep{park1994computational,park1995lie,park1994recursive,lee2005newton}, and later generalized to multi–DoF joints by Chhabra and Emami \citep{chhabra2014generalized}.  The exponential coordinates in these classical formulations \citep{ball1998treatise,borri2000representations,selig2005geometric} can often be susceptible to parametrization singularities for multi-DoF joints (i.e. spacecraft's pose) \citep{chhabra2014generalized}, though providing a minimal global representation for single-DoF joints of the manipulator \citep{herve1999lie,stramigioli2002geometry}. 
 %Space manipulators are multi-body systems composed of a 6-degree-of-freedom (DoF) free-floating base and a serial chain of single-DoF joints. 
This motivated the utilization of global, coordinate-independent, and singularity-free frameworks such as the homogeneous transformation matrix \citep{gibson1990geometry,gibson1990geometry2,aspragathos1998comparative}, at the cost of added complexity in deriving dynamics. Lie groups provide a natural language to arrive at such singularity-free models. Duindam \textit{et al.} developed a local diffeomorphism based on the exponential map of the Special Euclidean group ($\SE$) that enabled singularity-free dynamics for systems with general multi-DoF joints \citep{duindam2008singularity}, facilitating the development of the Boltzmann–Hamel equations for constrained multibody systems \citep{stramigioli2001modeling,from2010singularity}.Through exploitation of Lie-group symmetries (i.e. invariance of the Lagrangian/Hamiltonian energy metric of the system under certain group actions) in reduction of mechanical systems, singularity-free Euler-Poincar\'{e} equations were developed that can describe a free-floating body's motion on the $\SE$ \cite{blochbook,bloch2001controlled}. Marsden and Ratiu established the foundational Poisson, Symplectic and Lagrangian reduction methods \citep{marsden1986reduction,marsden2007hamiltonian,marsden1990reduction,marsden2013introduction}. Scheurle–Marsden developed a Lagrangian quasi-velocity framework, yielding reduced Euler–Lagrange equations on the tangent bundle \cite{scheurle1993reduced,murray1997nonlinear,bloch1994reduction}\citep{scheurle1993reduced}. Bloch and McClamroch extended the framework to nonholonomic dynamics \cite{bloch1996nonholonomic,bloch1992control}. Hamel unified the Euler–Lagrange and Euler–Poincaré formalisms on local principal bundles, introducing the key notion of quasi-velocities \cite{hamel1904lagrange,hamel2013theoretische}.  Lastly, Cendra \textit{et al.} formulated the Lagrange–Poincaré equations (LPE) on $\SE$, yielding group-free, singularity-avoiding models that separate group (spacecraft) and shape (arm) dynamics without explicit pose coordinates \citep{cendra2001geometric,blochbook}.
From \textit{et al.} further advanced this framework by introducing minimal quasi-coordinates, producing singularity-free formulations applied to complex mechanisms, including underwater and space manipulators \citep{from2012explicit,from2012explicit2,from2011singularityspace,from2010singularityauv}.
 %This global, coordinate-free pose representation on $\SE$ traces to classical screw theory (Chasles, Ball) and its modern expositions in robotics and mechanics~\citep{ball1998treatise,borri2000representations,selig2005geometric}.
In space robotics, Chhabra–Emami applied symplectic/Chaplygin reductions for planning and control \citep{chhabra2015symplectic,chhabra2014nonholonomic}; and Mishra \textit{et al.} employed LPE for free-floating spacecraft–manipulator dynamics, with block-diagonal inertia structure exploited for simulation and observation \citep{mishra2020inertia,de2021relative}. However, rigorous treatments that \emph{explicitly} couple orbital motion, base attitude, and arm kinematics for GNC-relevant operating conditions have emerged only recently \citep{da2017attitude,xu2024orbit}. There have been Geometric formulations (e.g. LPE) for similar mobile manipulators such as underwater vehicle-manipulator systems \citep{antonelli2004adaptive,sarkar2001coordinated}, which offer methodological analogies for space manipulators (reduction, mechanical connection, momentum maps), while the governing physics differ from orbital settings (e.g., hydrodynamics and buoyancy vs. microgravity, high orbital rates, and symmetry-breaking orbital perturbations).

\medskip
Though existing geometric treatments of free-floating space manipulators preserve Lie‐group structure, they almost always incorporate orbital motion as an external layer (via CW-type linearizations or separate two-body propagation) rather than embedding it in the reduced equations themselves \citep{da2017attitude,xu2024orbit}. By contrast, we propose a combined set of \emph{Lagrange-Poincar\'e-Kepler Equations} that: (i) injects closed-form Keplerian evolution directly into the LPE , (ii) clearly links to classical space-robotics mappings (e.g., GJM) in momentum-conserving conditions, and (iii) works within a quasi-inertial/orbital frame. The formulation remains coordinate–free on principal bundles and uses standard reduction tools \citep{moghaddam2022,murray1994mathematical,lynch2017modern,blochbook,scheurle1993reduced}, while its orbital coupling is expressed via closed–form anomalies rather than separate orbital ODE integration \citep{curtis2013orbital}.%We make the connection to the widely used GJM precise by introducing a total Jacobian $\mathcal{J}$ that recovers the GJM structure in the momentum-conserving limit, (\emph{cf.} the Generalized Jacobian Matrix of Umetani and Yoshida~\citep{umetani1987continuous}). while clarifying similarities and differences when orbital forcing is present. 

 The major contributions of this work are as follows.
\begin{enumerate}
  \item \textbf{We rigorously embed Kepler's equations of orbital mechanics into the reduced LPE equations for free-floating manipulators in the form of explicit symmetry-breaking disturbances.}
  Thus, we address a longstanding gap in space manipulator modeling: where explicit utility of orbital motion has typically been limited to rendezvous and docking \citep{zong2020optimal,shi2022coupled,mao2018reachable}, coupled orbit/attitude dynamics of satellites without manipulators \citep{xu2024orbit,carr2022coupled,sincarsin1983gravitational,jiang2017integrated}.

  \item \textbf{We formally specialize the LPE on $\textbf{SE}(2)\subset\SE$ to incorporate planar orbital motion \citep{curtis2013orbital,viswanathan2012dynamics}, }
   deriving a complete matrix-form of the disturbed equations \citep{curtis2013orbital,viswanathan2012dynamics}.

  \item \textbf{We derive a numerically well-conditioned, non-stiff and singularity-robust LPKE, } through (i) avoiding singularities due to base parameterization \citep{moghaddam2022,murray1997nonlinear,marsden1974reduction,stramigioli2007geometric,mishra2022reduced,from2011singularityspace}, (ii) mitigating the stiffness due to the large disparity between orbital and manipulator motion by formulating the dynamics in a quasi–inertial/orbital frame,\citep{moghaddam2022,murray1997nonlinear,marsden1974reduction,stramigioli2007geometric,mishra2022reduced,from2011singularityspace}, and (iii) incorporating a more robust  analytic Keplerian evolution over a full orbital ODE propagation \citep{curtis2013orbital,xu2024orbit,sincarsin1983gravitational}.
\end{enumerate}
%While coupled orbit/attitude dynamics of satellites (rather than space manipulators) has been predominantly investigated for station-keeping under disturbances \citep{xu2024orbit,carr2022coupled,sincarsin1983gravitational,jiang2017integrated}, we incorporate effects of orbital motion into the LPKE formalism of space manipulators \citep{moghaddam2022} moving relative to an orbital frame, which emerge as distinct actions of orbit-attitude coupling, orbital disturbances, and Coriolis effects. 
The proposed LPKE complements prior geometric reductions for free‐floating manipulators while preserving the adjoint and metric factorizations emphasized in geometric mechanics \citep{moghaddam2022,murray1994mathematical,blochbook}. In doing so, the proposed connection $\mathcal{A}$ provides a precise bridge to the Generalized Jacobian Matrix (GJM) used in space robotics, retaining its interpretability while explicitly accounting for symmetry‐breaking orbital inputs. Finally, the coordinate‐free, principal‐bundle formulation maintains compatibility with modern geometric robotics toolchains \citep{lynch2017modern}, facilitating integration with planning and control frameworks posed on reduced manifolds. The result is a single, singularity–free reduced model that \emph{natively} couples orbit, attitude, and arm dynamics.

\medskip
This paper is structured as follows: 
We provide preliminaries on the kinematics of relative rigid motion in Section \ref{foundations}. In Section \ref{kinematics}, we extend the kinematic formulation in \citep{moghaddam2022} to orbital space manipulators. The LPKE of spacecraft-manipulator systems with respect to a moving frame is rigorously detailed in Section \ref{sec4}. Validation framework, simulation results, and discussions are provided in Section \ref{noninertial: Numerical Study}. Concluding remarks are included in Section \ref{conclusion}.

\section{Preliminaries of Relative Motion} \label{foundations}

In this section we provide the fundamental concepts of relative motion in multi-body systems (for further detail, refer to \citep{moghaddam2022}). The relative pose between Body $i$ and Body $j$ ($i,j \in \{0,1,\cdots,n\}$) in a multi-body system is expressed as $g^j_i={\begin{bmatrix} R^j_i & {}^jp^j_i\\ \mathbb{O}_{1\times3} & 1\end{bmatrix}} \in \SE$, where \(R^j_i \in \SO\) and \({}^jp^j_i \in\mathbb{R}^3\) represent the relative orientation and position. The symbol \(\mathbb{O}\) denotes a zero matrix. The Special Euclidean (\(\SE\)) group and the Special Orthogonal (\(\SO\)) group are matrix Lie groups providing global representations of pose and orientation. The relative velocity of Body $i$ with respect to Body $j$ in Body $i$'s frame is given by \({}^i\hat{V}^j_i={\begin{bmatrix}{}^i\tilde{\omega}^j_i & {}^iv^j_i\\ \bO_{1\times 3} & 0\end{bmatrix}}=(g^j_i)^{-1}\dot{g}^j_i \in\se\), where \({}^i{\omega}^j_i\) and \({}^iv^j_i\) are the relative angular and linear velocities. The \textit{tilde} operator transforms a vector \({}^i{\omega}^j_i\in\R^3\) to a skew-symmetric matrix \({}^i\tilde{\omega}^j_i\in \so\), the Lie algebra of \(\SO\). The \textit{hat} operator maps between \(\R^6\) and \(\se\), the Lie algebra of \(\SE\), such that ${}^iV^j_i=\begin{bmatrix} ({}^i{\omega}^j_i)^T & ({}^iv^j_i)^T\end{bmatrix}^T\in\R^6$. 
The \textit{Adjoint matrix} $\Ad_{g_i^ j}:={\begin{bmatrix}R_i^ j & (^ j\tilde{p}_i^ j) R_i^ j \\\mathbb{O}_{3 \times 3}& R_i^ j\end{bmatrix}}$ transforms velocity vectors from frame $i$ to frame $j$, such that \( ^j{V}_i^j=\Ad_{g^j_i}{}^i{V}_i^j\). The \textit{adjoint operator} $\textbf{ad}_{{}^jV^j_i}(\cdot):=[{^j\hat{V}^j_i},\hat{(\cdot)}]^\vee={\begin{bmatrix}^j\tilde{\omega}_i^{j} & ^j\tilde{v}_i^{j} \\\mathbb{O}_{3\times 3}& ^j\tilde{\omega}_i^{j}\end{bmatrix}}$, and the map \(\tilde{\textbf{ad}}_{P_j}:={\begin{bmatrix}\mathbb{O}_{3\times 3} & \tilde f_j\\\tilde f_j & \tilde \tau_j\end{bmatrix}}:\R^6\rightarrow (\R^6)^*\) for every $P_j=[f_j^T~~\tau_j^T]^T\in (\R^6)^*$ such that \(\tilde\ad_{P_j}{}^jV_i^j= \ad_{{}^jV_i^j}^TP_j\). The relative pose between successive Bodies $i$ and $i-1$ connected via a 1-DoF joint can be parameterized using the exponential map of $\SE$:
\begin{equation}
g_i^{i-1}=e^{{}^{i-1}\hat{\xi}^{i-1}_iq_i}\bar{g}^{i-1}_i\label{expparam}
\end{equation}
where \({}^{i-1}\xi^{i-1}_i\in \R^6\) is the joint twist in Body $i-1$, \(\bar g_i^{i-1}\in\SE\) is the initial relative pose, and \(q_i\) is the joint parameter \citep{murray1994mathematical}.

\section{Kinematics of Spacecraft-Manipulator Systems in Orbit}
\label{kinematics}
This section summarizes a kinematic framework for spacecraft-manipulator systems in orbital environment based on \citep{moghaddam2022}, developed by integrating Lie group representations of multi-body systems with the POE formula \citep{chhabra2014generalized}. We model the space manipulator as a serial-link open-chain system, observed relative to a moving orbital frame $\orbit$. The system consists of $n+1$ bodies, indexed by $\{0,1,\cdots, n\}$. Motion of the orbital frame is formulated relative to a quasi-inertial (non-accelerated) frame $I$. %, with $0$ as the base spacecraft and $n$ as the end-effector, and $n+1$ joints connecting consecutive bodies.
Reference frames are attached to each link at its preceding joint, and inertial parameters are expressed in these local frames. The pose of the orbital frame $\orbit$ relative to the quasi-inertial frame $I$ is described by a displacement sub-group joint (Rigorously defined in Definition 1 of \citep{moghaddam2022}), with the $3$-dimensional restricted relative configuration manifold $Q^I_\orbit$ \citep{chhabra2014generalized}. The matrix Lie group $\textbf{SE}(2)$ provides a global identification of this manifold, through an isomorphism $\iota_\orbit$ and left translation by a fixed relative pose $\bar{g}^I_\orbit \in Q^I_\orbit$, with members of the form $\breve{g}^I_\orbit \in \textbf{SE}(2)$. We define $\breve{\iota}(.) : \textbf{SE}(2) \rightarrow \SE$ such that $\breve{\iota}(\breve{g}^I_\orbit)=g^I_t \in Q_\orbit^I$.  We will directly find the evolution of $\breve{g}^I_\orbit$ as a function of time, and thus, will directly work with members of $\textbf{SE}(2)$ for formulation of kinematics and dynamics to avoid parametrization singularities. The pose of the base spacecraft, indexed by $0$, relative to the inertial frame $g^I_0=g^I_\orbit g^\orbit_0 \in Q^I_0$ with $g^\orbit_0 \in Q^\orbit_0$ where $Q^\orbit_0$ is the 6-dimensional relative configuration manifold corresponding to the motion of the spacecraft relative to orbit. 
%We choose $\bar{g}^I_0{}=\mathbb{I}_4$ as this fixed pose as, representing the initial overlap of the orbital frame and the quasi-inertial frame during the orbital operation. The pose of the base spacecraft relative to the orbital frame $\orbit$ is described by $g^\orbit_0 \in Q_0$. % a displacement sub-group joint (Definition \ref{def:joint} in \citep{moghaddam2022}), with the $b$-dimensional restricted relative configuration manifold $Q^I_0$ \citep{chhabra2014generalized}. 
%To avoid singularities in parameterization, we directly use elements of $\Theta$ when formulating the system's kinematics and dynamics. 
We have $Q_\mathfrak{m}$ represent the $n$-dimensional configuration manifold of the manipulator, with elements consisting of sets of joint parameters $q_\mathfrak{m}=(q_1,\cdots,q_n)\in Q_\mathfrak{m}$. Consequently, the $(6+n)$-dimensional configuration manifold of the space manipulator relative to the quasi-inertial frame is a Cartesian product of the two aforementioned manifolds $Q:= Q^I_0 \times Q_\mathfrak{m}$ \citep{chhabra2014generalized}.
%The relative configuration manifolds $Q^{i-1}_i$ of consecutive manipulator links are parameterized using the exponential map \citep{murray1994mathematical}:

\subsection{ Orbital Motion}\label{sec:orbitmotion}
Here, we brief the orbital mechanics fundamentals utilized for this study (For rigorous explanation refer to \citep{curtis2013orbital}). The  Earth-centered  Perifocal coordinate frame $\oplus$ is fixed at the Center-of-Mass (CoM)  of Earth, with its axes $x_\oplus$ and $y_\oplus$ in the plane of the orbit (perifocal frame), and $z_\oplus$ perpendicular to the plane of the orbit (in the direction of the orbital angular velocity). The axis $x_\oplus$ is directed from the center of Earth towards the orbit periapsis, and $y_\orbit$ complete the frame (see Figure \ref{fig:orbitalframes}). This frame matches the initial pose of the orbital frame so that $R^I_\orbit(t=0)=\mathbb{I}_3$. The quasi-inertial frame $I$ is an inertial constant-velocity frame, initially matching the pose of the orbital frame $g^I_\orbit(t=0)=\mathbb{I}_4$, and moving with a constant linear velocity equal to the initial linear velocity of orbital frame ${}^\oplus v_I^\oplus={}^{\oplus}{v}^{\oplus}_\orbit(t=0)=const.$ and ${}^\oplus w_I^\oplus=0$ (see Figure \ref{fig:orbitalframes}). The orbital frame (indexed by $\orbit$) has its $x_\orbit$ and $y_\orbit$ axes in the perifocal plane, with $x_\orbit$ directed radially away from the Earth %(directed from center of the earth toward the origin of the orbital frame) 
and $z_\orbit=z_\oplus$ (see Figure \ref{fig:orbitalframes}). %The classical orbital elements utilized here are $X_{OEc} = \begin{bmatrix}a& e& i& \omega& w& \theta\end{bmatrix}^T$,where $a$ is the semi-major axis of the orbit, $e$ is the eccentricity the orbit, $i$ is the inclination of the plane of orbit, $\omega$ is the longitude of ascending node, $w$ is the argument of perigee, and $\theta$ is the true anomaly. The inertial states are $X_{I} = \begin{bmatrix}x&y&z&\dot{x}&\dot{y}&\dot{z}\end{bmatrix}^T$,
Within the perifocal frame, $\breve{p}_\orbit={}^\oplus \breve{p}^\oplus_\orbit=\begin{bmatrix}x_\orbit&y_\orbit\end{bmatrix}^T$ and $\theta=\theta_\orbit^I$ are the perifocal linear and angular position of the orbital frame, respectively. Similarly, $\breve{v}_\orbit={}^\oplus \breve{v}^\oplus_\orbit=\begin{bmatrix}v_{\orbit x}&v_{\orbit y}\end{bmatrix}^T$ and ${}^\oplus \breve{w}^\oplus_\orbit=\dot{\theta}$ are the perifocal Cartesian linear and angular velocity of the orbit, respectively. Knowing the initial position and velocity of the orbit, orbit's angular momentum can be found $\bar{\mu}_\orbit=\breve{p}_\orbit\times{}\breve{v}_\orbit=\begin{bmatrix}0&0&\mu_\orbit\end{bmatrix}^T$. The orbit eccentricity $\breve{e}_\orbit=\frac{1}{\mathcal{G}m_\oplus} \left(\left(|\breve{v}_\orbit|^2-\frac{\mathcal{G}m_\oplus}{|\breve{p}_\orbit|}\right)\breve{p}_\orbit-|\breve{p}_\orbit| \left((\breve{p}_\orbit)^T\breve{v}_\orbit\right) \breve{v}_\orbit\right)$ is a geometric descriptor of ellipticity of the orbit and $e_\orbit=|\breve{e}_\orbit|$, and the true anomaly $\bar{\theta}=\theta(t=0)=\cos^{-1}\left(\frac{\textbf{e}_\orbit^T \breve{p}_\orbit}{e_\orbit |\breve{p}_\orbit|}\right)$ represents the angular distance between $x_\oplus$ and $x_\orbit$\citep{curtis2013orbital}. 
%${}^\oplus p^\oplus_\orbit=\begin{bmatrix}x&y&0\end{bmatrix}^T$ and ${}^\oplus v^\oplus_\orbit=\begin{bmatrix}0&0&{w_z}_\orbit\end{bmatrix}^T$ are the inertial Cartesian position and velocity of the orbit, respectively. Knowing the initial position and velocity of the orbit, the orbital momentum is found $\bar{\mu}_\orbit={}^\oplus p^\oplus_\orbit\times{}^\oplus v^\oplus_\orbit=\begin{bmatrix}0&0&\mu_\orbit\end{bmatrix}^T$. The orbit eccentricity $e_\orbit=|\textbf{e}_\orbit|=\left|\frac{1}{\mathcal{G}m_\oplus} \left[\left(|{}^\oplus v^\oplus_\orbit|^2-\frac{\mathcal{G}m_\oplus}{|{}^\oplus p^\oplus_\orbit|}\right){}^\oplus p^\oplus_\orbit-|{}^\oplus p^\oplus_\orbit| \left(({}^\oplus p^\oplus_\orbit)^T{}^\oplus v^\oplus_\orbit\right) {}^\oplus v^\oplus_\orbit\right]\right|$ is a geometric property of ellipticity of the orbit essential to calculation of its evolution, and the true anomaly $\theta(t=0)=\cos^{-1}\left(\frac{\textbf{e}_\orbit^T {}^\oplus p^\oplus_\orbit}{e_\orbit |{}^\oplus p^\oplus_\orbit|}\right)$ represents the angular distance between $x_\oplus$ and $x_\orbit$\citep{curtis2013orbital}.
%Perifocal Velocity of the orbital frame can be found from:\begin{equation} V_p=\frac{\mathcal{G}m_\orbit}{\mu_\orbit}\begin{bmatrix}-\textbf{sin}\theta & e+\textbf{cos}\theta & 0 &0 & 0 & \frac{\mathcal{G}m_\orbit (1+e_\orbit \cos(\theta))^2}{\mu_\orbit^2}\end{bmatrix}^T\end{equation}

%********************LEMMMA ORBIT***************************************
\begin{lemma}
(Evolution of the orbital parameter $\theta$) \label{orbitdyn}
The True Anomaly $\theta$ of an undisturbed elliptic orbit in the perifocal frame with the eccentricity $e_\orbit$ and constant orbital angular momentum $\mu_\orbit$ as a function of time is found from the well-known Kepler's equations as \citep{curtis2013orbital}:
\begin{equation*}
    \theta(t)=2 \tan^{-1}\left(\frac{1-e_\orbit}{1+e_\orbit}\tan(E(t))\right), 
\end{equation*}
\begin{equation}
    E(t)=\frac{1}{2}\textbf{MA}^{-1}\left(E(\bar{\theta})-e_\orbit \sin(E(\bar{\theta}))\right)+\frac{(\mathcal{G}m_\oplus)^2}{\mu_\orbit^3}t),\label{noninertial:trueanomalyevolution}
\end{equation}
where the Mean Anomaly function $\textbf{MA}(E):=E-e_\orbit \sin(E)$ is the fraction of orbital period since periapsis, $E=2\tan^{-1}\left(\frac{1+e_\orbit}{1-e_\orbit}\tan\left(\frac{1}{2}\theta\right)\right)$ is the eccentric anomaly, %$\mathfrak{n}=$ is the constant rate of progression of $\textbf{MA}$ as a function of time, 
$\mathcal{G}m_\oplus$ is the gravitational constant of the Earth, and $\bar{\theta}=\theta|_{t=0}$.
\end{lemma}
\begin{proof}
The proof is provided in  Appendix \ref{app:orbitkinematics} . 
\end{proof}
%\begin{lemma}(Evolution of the orbital parameter $\theta$) \label{orbitdyn}The True Anomaly $\theta \in \Theta$ of an undisturbed elliptic orbit in the perifocal frame with the eccentricity $e_\orbit$ and constant total orbital momentum $\mu_\orbit$ as a function of time is found from the well-known Kepler's equations as \citep{curtis2013orbital}:\begin{equation}\theta(t)=2 \tan^{-1}\left(\frac{1-e_\orbit}{1+e_\orbit}\tan\Big(\frac{1}{2}\textbf{MA}^{-1}(\textbf{MA}(E(t=0))+\mathfrak{n}t)\Big)\right),\label{noninertial:trueanomalyevolution}\end{equation}where $\textbf{MA}$ is the mean anomaly map $\textbf{MA}(E):=E-e_\orbit \sin(E)$ defined based on \eqref{meananomaly} in \ref{app:orbitkinematics}, $E=2\tan^{-1}\left(\frac{1+e_\orbit}{1-e_\orbit}\tan\left(\frac{1}{2}\theta\right)\right)$ is the mean anomaly of the orbital frame, $\mathfrak{n}=\frac{\sqrt{\mu_\orbit^3}}{\mathcal{G}m_\oplus}$ is the constant rate of progression of $\textbf{MA}(t)$ in time, and  $\mathcal{G}m_\oplus$ is the gravitaional constant of the Earth.\end{lemma}\begin{proof}The proof is provided in \ref{app:orbitkinematics}. \end{proof}
\begin{figure}
    \centering
    \includegraphics[width=0.99\columnwidth]{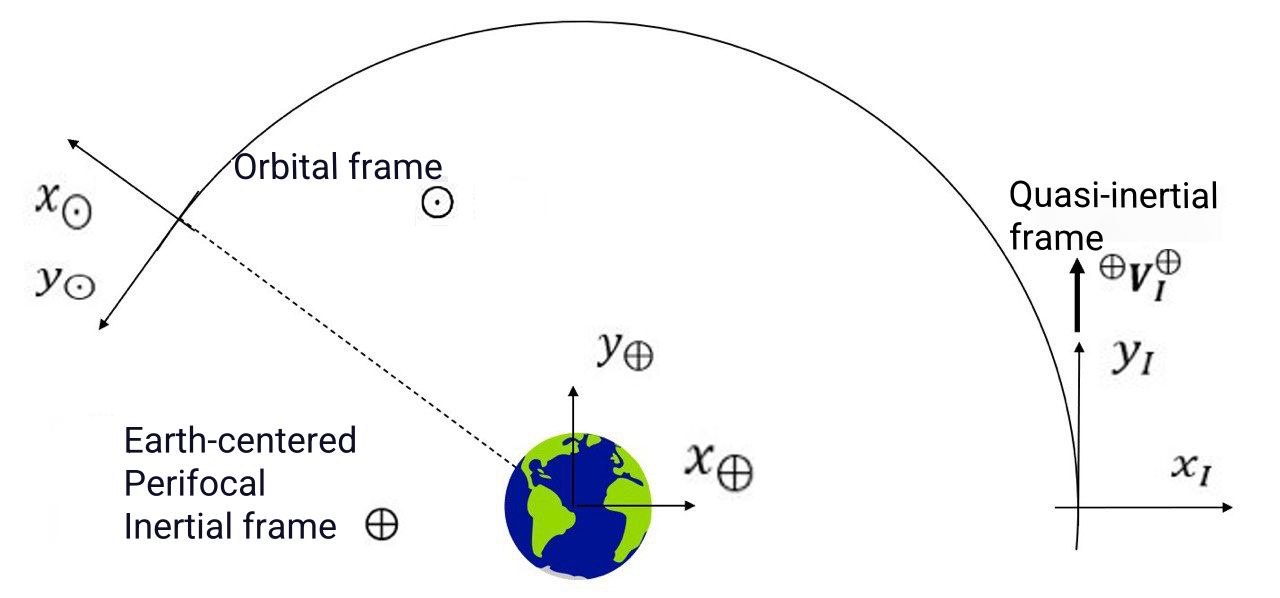}
    \caption[Orbital Coordinate Frames]{Quasi-Inertial  Frame $I$, Perifocal frame $\oplus$, and Orbital frame $\orbit$}
    \label{fig:orbitalframes}
\end{figure}

% LEMMA pose and Vel ******************************
\begin{lemma} \label{orbitkinlemma}(Evolution of orbit pose and velocity in the quasi-inertial frame) Knowing $\theta(t)$ at each instant, orbit  eccentricity $e_\orbit$, orbit angular momentum $\mu_\orbit$ and velocity of the quasi-inertial frame ${}^\oplus \V^\oplus_I=[{v}_x~ {v}_y ~0]^T=const.$, the absolute pose of the orbital frame with respect to the quasi-inertial frame ${\breve{g}}^I_\orbit \in \textbf{SE}(2)$ is calculated as ${\breve{g}}^I_\orbit(\theta)=\begin{bmatrix} \breve{R}^I_\orbit & {}^I \breve{p}^I_\orbit\\ \mathbb{O}_{1\times2} & 1\end{bmatrix}$ ,  %=\iota^\oplus_\orbit\breve{g}^\oplus_\orbit(\theta) %\iota^\oplus_\orbit
with \citep{curtis2013orbital}
\begin{equation*}
    \breve{R}^I_\orbit=\begin{bmatrix}\cos(\theta-\bar{\theta}) & \sin(\theta-\bar{\theta})\\ -\sin(\theta-\bar{\theta}) & \cos(\theta-\bar{\theta})\end{bmatrix} %\quad \& ~ =\frac{\mu_\odot ^2}{\mathcal{G}m_\oplus}\begin{bmatrix}\frac{1}{1+e_\odot \cos(\theta )}-\frac{\cos(\theta-\bar{\theta})}{1+e_\odot \cos(\bar{\theta} )}\\ \frac{\sin(\theta-\bar{\theta})}{1+e_\odot \cos(\bar{\theta} )}\\0\end{bmatrix}
\end{equation*}
\begin{multline}
        {}^I \breve{p}^I_\orbit=\frac{\mu_\odot ^2}{\mathcal{G}m_\oplus}\begin{bmatrix}\cos(\bar{\theta}) &\sin(\bar{\theta})\\ -\sin(\bar{\theta}) & \cos(\bar{\theta})\end{bmatrix}\\
        \begin{bmatrix}\frac{\cos(\theta)}{1+e_\odot \cos(\theta )}-\frac{\cos(\bar{\theta})}{1+e_\odot \cos(\bar{\theta} )}-\sin(\bar{\theta})t\\ \frac{\sin(\theta)}{1+e_\odot \cos(\theta )}-\frac{\sin(\bar{\theta})}{1+e_\odot \cos(\bar{\theta} )}-\left(e+\cos(\bar{\theta})\right)t\end{bmatrix}\label{eq:orbitpose}.
\end{multline}

Orbit's velocity ${}^\orbit \V^I_\orbit$ with respect to the quasi-inertial frame is:
\begin{equation}
    {}^\orbit\V_{\orbit}^I(\theta)=\frac{\mathcal{G}m_\oplus }{\mu_\orbit}\begin{bmatrix}  -\sin(\theta-\bar{\theta})\\ 1-\cos(\theta-\bar{\theta})\\\frac{\mathcal{G}m_\oplus} {\mu_\orbit^2}(1+e_\orbit \cos(\theta))^2\end{bmatrix}^T. \label{orbitvel}
\end{equation}
with $\bar{\theta}=\theta|_{t=0}$. Lastly, we can see $g^I_\orbit=\breve{\iota}(\breve{g}^I_\orbit)$ and ${}^\orbit V^I_\orbit=\iota_\orbit {}^\orbit \V^I_\orbit$.%, and ${}^\orbit \bar{\V}^I_\orbit=\begin{bmatrix}\bar{v}_x & \bar{v}_y & \end{bmatrix}^T=\begin{bmatrix}-\sin \bar{\theta} & e+\cos \bar{\theta} & 0\end{bmatrix}^T$.
\end{lemma}
\begin{proof}
Proof is provided in  Appendix \ref{app:relposevel} .
\end{proof}
%\begin{lemma} (Evolution of orbit pose and velocity in the quasi-inertial frame)Knowing $\theta(t)$ at each instant, the absolute pose of the orbital frame with respect to the quasi-inertial frame ${g}^I_\orbit(\theta) \in \textbf{se}(2)$ is found as \citep{curtis2013orbital}:\begin{equation}{g}^I_\orbit(\theta)=\begin{bmatrix} R^I_\orbit & {}^\orbit p^I_\orbit\\ \mathbb{O}_{1\times2} & 1\end{bmatrix}, %=\iota^\oplus_\orbit\breve{g}^\oplus_\orbit(\theta) %\iota^\oplus_\orbit\label{eq:orbitpose}\end{equation}with\begin{equation} R^I_\orbit=\begin{bmatrix}\cos(\theta-\bar{\theta}) & \sin(\theta-\bar{\theta}) & 0\\ -\sin(\theta-\bar{\theta}) & \cos(\theta-\bar{\theta}) & 0\\ 0 & 0 & 1\end{bmatrix} \quad \& ~ {}^\orbit p^I_\orbit=\frac{\mu_\odot ^2}{\mathcal{G}m_\oplus}\begin{bmatrix}\frac{1}{1+e_\odot \cos(\theta )}-\frac{\cos(\theta-\bar{\theta})}{1+e_\odot \cos(\bar{\theta} )}\\ \frac{\sin(\theta-\bar{\theta})}{1+e_\odot \cos(\bar{\theta} )}\\0\end{bmatrix}\end{equation}and the orbit's velocity with respect to the quasi-inertial frame can be found:\begin{equation}{}^\orbit\V_{\orbit}^I(\theta)=\frac{\mathcal{G}m_\oplus }{\mu_\orbit}\begin{bmatrix}  -\sin(\theta-\bar{\theta})\\ 1-\cos(\theta-\bar{\theta})\\\frac{\mathcal{G}m_\oplus} {\mu_\orbit^2}(1+e_\orbit \cos(\theta))^2\end{bmatrix}. \label{orbitvel}\end{equation}with $\bar{\theta}=\theta(t=0)$.\end{lemma}\begin{proof}Proof is provided in \ref{app:orbitkinematics}\end{proof}

\subsection{Space Manipulator Kinematics}
\label{forwardkinematics}
Given a base spacecraft configuration $g^\orbit_0$ and an orbit configuration $g^I_\orbit$ such that $g^I_\orbit g^\orbit_0\in Q_0^I$, and a set of joint angles $q_\mathfrak{m} \in Q_\mathfrak{m}$, the pose of any joint/body relative to a chosen frame is derived using the cascade of relative poses between intermediate bodies \citep{park1994computational}. %For the representation of a spacecraft-manipulator system in an orbital environment, the reference frame can be the base spacecraft's frame, the orbital frame or the quasi-inertial frame.
The relative pose between Body $i \in {1,...,n}$ and the quasi-inertial frame $I$ is obtained as $g^I_{i}(g^I_0,q_\mathfrak{m})=g^I_{\orbit}g^\orbit_0g^0_{i}(q_\mathfrak{m})$, where the orbit's configuration is $g^I_\orbit=\breve{\iota}(\breve{g}^I_\orbit) \in Q^I_\orbit$ is found from \eqref{eq:orbitpose}. The POE formula in \eqref{expparam} leads to $g^0_i(q_\mathfrak{m})=e^{\hat{\xi}_{1}q_{1}}\ldots e^{\hat{\xi}_{i}q_{i}}\bar{g}^0_{i}$ \citep{murray1994mathematical,lynch2017modern}, where $\bar{g}^0_i=\bar{g}^0_1\cdots \bar{g}^{i-1}_i$ is the fixed pose of Body $i$ relative to the spacecraft (commonly chosen to be the initial pose). Also, 
\begin{equation}
  {\xi}_i:= \Ad_{\bar{g}^0_{i-1}}{}^{i-1}{\xi}^{i-1}_i\in\mathfrak{g}_0^\vee, \label{twist}
\end{equation}
is the manipulator's joints' twists in the spacecraft's initial frame \citep{murray1994mathematical}.
The end-effector $n$ forward kinematics is $g^I_n(g^I_0,q_\mathfrak{m})=g^I_\orbit g^\orbit_0e^{\hat{\xi}_{1}q_{1}}...e^{\hat{\xi}_{n}q_{n}}\bar{g}^0_{n}$. Pose of the CoM  of each Body $i$ with respect to $I$, denoted by $g^I_{cm,i}=g^I_i\bar{g}^i_{cm,i}=g^I_\orbit g^\orbit_0e^{\hat{\xi}_{1}q_{1}}...e^{\hat{\xi}_{i}q_{i}}\bar{g}^0_{cm,i}$, where $\bar{g}^i_{cm,i}$ and $\bar{g}^0_{cm,i}=\bar{g}^0_{i}\bar{g}^i_{cm,i}$  are the constant poses of the CoM of Body $i$ relative to their preceding joint and the spacecraft in its initial configuration, respectively.  
%---------------------------------------------------------------------
%\subsection{Differential Kinematics} \label{diffkin}
%We present the absolute body velocity of rigid bodies in a spacecraft-manipulator system expressed in the quasi-inertial, spacecraft or body frame, using appropriate Jacobian mappings. 
Let $ \mathfrak{se}(2)^\vee \cong \q^0_0 \subseteq \g_0$ be the Lie algebra of $\textbf{SE}(2) \cong Q_0:=(\bar{g}^I_0)^{-1}Q^I_0\subseteq G_0$ with the isomorphism $\iota_0:  \mathfrak{se}(2)^\vee\rightarrow \q_0$ with elements of the form $\V_\orbit \in \textbf{se}(2)^\vee$ \citep{chhabra2014generalized}. The differential kinematics of the spacecraft-manipulator system is based on the left-trivialization of $TQ^I_0$ and $TQ_\m$, resulting in elements of the form $(V_\orbit +V_0,q_\m,\dot{q}_\mathfrak{m})\in Q^I_0 \times  TQ^I_0\times TQ_{\mathfrak{m}}$. Here, we have defined $V_0={}^0V^\orbit_0$ and  
\begin{equation}V_\orbit={}^0 V_\orbit^I=\Ad_{g^\orbit_0}\iota_\orbit{}^\orbit\V_\orbit^I.
\end{equation} 
The inclusion map for the inertial restricted velocity is 
$$\iota_\orbit={\begin{bmatrix}1 & 0 &0 &0 & 0 & 0\\0 & 1 &0 &0 & 0 & 0\\0 & 0 &0 &0 & 0 & 1\\\end{bmatrix}}^T.$$
%Let $\g_0$ be the Lie algebra of $Q_0$, and $\vartheta\cong \mathfrak{q}_\orbit \subseteq \g_\orbit$ be the Lie algebra of $\Theta\cong Q_\orbit:=(\bar{g}^I_\orbit)^{-1}Q^I_\orbit\subseteq G_\orbit$ with the isomorphism $\iota_0: \vartheta^\vee \rightarrow \q_\orbit^\vee$ with elements of the form $\V_\orbit \in \vartheta^\vee$ \citep{chhabra2014generalized}. The differential kinematics of the spacecraft-manipulator system is based on the left-trivialization of $T\Theta$, $TQ_0$ and $TQ_\m$, resulting in elements of the form $(\theta,\V_\orbit, V_0,q_\m,\dot{q}_\mathfrak{m})\in \Theta \times \vartheta^\vee \times TQ_0\times TQ_{\mathfrak{m}}$. %Note that through $\iota_0$ every element of $\vartheta$ maps to an orbit velocity in the orbital frame itself\citep{chhabra2014unified}.  %to the the base spacecraft's constrained velocity and the time-derivatives of the manipulator's generalized coordinates, the parametrization of the tangent space? 

\begin{lemma}
In matrix form, the velocity of Body $i$ relative to the quasi-inertial frame can be expressed in the quasi-inertial, spacecraft or Body $i$ coordinate frame as \citep{murray1994mathematical}

\begin{align}
 &^iV_i^I={}^iJ^I_i\begin{bmatrix}V_\orbit+V_0\\\dot{q}_\mathfrak{m}\end{bmatrix} ~ \& ~{}^IV_i^I= {}^IJ^I_i\begin{bmatrix}V_\orbit+V_0\\\dot{q}_\mathfrak{m}\end{bmatrix} \\
 &^0V_i^I={}^0J^I_i\begin{bmatrix}V_\orbit+V_0\\\dot{q}_\mathfrak{m}\end{bmatrix}, \label{bodyvelocity}
\end{align}
where the spacecraft Jacobian matrix is
\begin{equation}
    ^0J^I_i(q_\m)\!=\!\begin{bmatrix} \mathbb{I}_6\!\!\!  &  {\xi}_{1}\,\cdots\, \Ad_{e^{\hat{\xi}_{1}q_{1}}\cdots e^{\hat{\xi}_{i-1}q_{i-1}}}{\xi}_i\!\!\! & \bO_{6\times (n-i)}\end{bmatrix},
    \label{baseJacobian}
\end{equation}
and the quasi-inertial, spacecraft and body Jacobian matrices are
\begin{align}
^I&J^I_i(g^I_\orbit,g^\orbit_0,q_\m)=\Ad_{g^I_\orbit(\theta)g^\orbit_0} {}^0J^I_i\nonumber\\  
    ^\orbit &J^I_i(g^\orbit_0,q_\m)=\Ad_{g^\orbit_0} {}^0J^I_i\nonumber\\ 
    {}^i&J^I_i(q_\m)=\Ad_{(g^0_i)^{-1}}{}{}^0J^I_i,
    \label{noninertial:bodyJacobian}
\end{align}
where the twists ${\xi}_i$ ($i=1,\ldots,n$) have been defined by \eqref{twist}.
\end{lemma}
\begin{proof}[Proof ] For rigorous proof, please refer to \citep{moghaddam2022}.
\end{proof}
% Since these velocities will be paired with body inertia matrices in calculating the kinetic energy, we have represented the absolute velocity of each body in both the inertial and body coordinate frames. %Therefore, the body Jacobian is presented:
% as a result we can say
% \begin{multline}
%     {}^iJ^o_i(q)=Ad_{(g^b_i)^{-1}}{}\begin{bmatrix} I_{6\times 6}  & \xi_{1}\cdots Ad_{e^{\xi_{1}q_{1}}\cdots e^{\xi_{i-1}q_{i-1}}}\xi_i & 0\cdots 0\end{bmatrix}\\
%     =\begin{bmatrix}Ad_{g^b_i(0)^{-1}e^{-\xi_{i}q_{i}}\cdots e^{-\xi_1q_1}}I_{6 \times 6} & Ad_{g^b_i(0)^{-1}e^{-\xi_{i}q_{i}}\cdots e^{-\xi_1q_1}}\xi_{1}& \cdots &Ad_{ g^b_i(0)^{-1}e^{-\xi_iq_i}}\xi_i & 0& \cdots& 0\end{bmatrix} 
%     \label{bodyJacobian}
% \end{multline}

\section{Space Manipulator Dynamics Relative to Orbital Frame} \label{sec4}
Built upon the moving base LPE formulation of spacecraft-manipulator dynamics in \citep{moghaddam2022}, we present the dynamic formulation of spacecraft-manipulator systems on the Cartesian product of $Q^I_0$ capturing the spacecraft motion and $Q_\m$ capturing the manipulator's configuration.  %In the derivation of the equations we use the principal bundle structure of the configuration manifold. The equations are in the form of LPE combining a reduced set of Euler-Lagrange equations for the arm motion on $Q_\mathfrak{m}$, a set of Euler-Poincar\'{e} equations for the spacecraft motion on $Q_0$, and a set of kinematic reconstruction of orbital velocity on $\Theta$. This introduces 
The Equations of Motion (EoM) are decomposed into three components: (i) singularity-free locked spacecraft-manipulator dynamics via Euler-Poincaré equations, (ii) internal manipulator dynamics, and (iii) orbital dynamics through extended Kepler's equations in quasi-inertial frame. %The decomposition also allows for studying the effects of internal and external forces and their couplings in an elucidated fashion. 
EoM of the spacecraft-manipulator system is found via the Lagrange-D'Alembert principle, based on the Lagrangian $\mathcal{L}\colon TQ \rightarrow \mathbb{R}$ defined as:
\begin{equation}
    \mathcal{L}=K-U,
    \label{noninertial:lagrangianfirst}
\end{equation}
and the applied forces to the system. Here, $K\colon TQ\rightarrow \mathbb{R}$ is the total kinetic energy, and $U\colon Q\rightarrow\mathbb{R}$ is the potential energy. The kinetic energy of the space manipulator is the sum of the kinetic energies of the bodies in the chain %. Using velocities in the body coordinate frames:
$K(g^I_0,\dot{g}^I_0,q_\mathfrak{m},\dot{q}_\mathfrak{m})=\frac{1}{2} \sum_{i=0}^n {{}^iV_i^I}^T(^i\mathcal{M}_i){}^iV_i^I$,
%\label{kineticenergy}
where the constant left-invariant metric ${}^i\mathcal{M}_i=\Ad^{-T}_{\bar{g}^i_{cm,i}}\begin{bmatrix}\mathbb{I}_3 m_i & \mathbb{O}_3\\\mathbb{O}_3 & \mathfrak{I}_i\end{bmatrix}\Ad^{-1}_{\bar{g}^i_{cm,i}} : \g_i\rightarrow \g_i^*$ is Body $i$'s inertia matrix in its preceding joint coordinate frame, with $m_i$ being the mass of Body $i$ and $\mathcal{I}_i$ representing the moments of inertia of Body $i$. Body velocities can be collected from \eqref{bodyvelocity} in vector form:
\begin{align}
    \begin{bmatrix}^0V_{0}^I \\ \vdots \\ {^nV_{n}^I}\end{bmatrix} &=%\begin{bmatrix}^0V_{\orbit}^I+{}^0V_{0}^\orbit \\ \vdots \\ {^nV_{\orbit}^I+{}^nV_{n}^\orbit}\end{bmatrix}=
    \begin{bmatrix}{}^0J^I_0(q_\mathfrak{m})\\ \vdots\\ {}^nJ^I_n(q_\mathfrak{m})\end{bmatrix} \begin{bmatrix}V_\orbit+V_0\\\dot{q}_\mathfrak{m}\end{bmatrix}\\
    &=:\diag_0^n\{\Ad_{\bar{g}^0_i}^{-1}\}\mathcal{J}(q_\mathfrak{m})\begin{bmatrix}V_\orbit+V_0\\\dot{q}_\mathfrak{m}\end{bmatrix},\label{noninertial:totalJ}
\end{align}
where we have introduced the total Jacobian matrix $\mathcal{J}$, and $\diag_0^n\{\cdot\}$ is the block diagonal matrix of its arguments for $i\!=\!0,\ldots,n$.
%----------------------------
It has been shown in \citep{moghaddam2022} that the Jacobian $\mathcal{J}$ in \eqref{noninertial:totalJ} can be decomposed into a constant matrix $\Xi$ and a $Q_\m$-dependent matrix $\mathfrak{L}$ as $\mathcal{J}(q_\mathfrak{m})=\mathfrak{L}(q_\mathfrak{m})\Xi$, where $\Xi=\begin{bmatrix}   \mathbb{I}_{6\times 6}& \mathbb{O}_{6\times n}\\\mathbb{O}_{6n\times 6}& \diag_1^n\{\xi_i\}\end{bmatrix}=:\begin{bmatrix} \mathbb{I}_{6\times 6}& \mathbb{O}_{6\times n}\\\mathbb{O}_{6n\times 6}& \Xi_\mathfrak{m}\end{bmatrix}$, with $\xi_i$ from \eqref{twist}, and $\mathfrak{L}=\begin{bmatrix}\mathbb{I}_{6\times6} &\mathbb{O}_{6\times 6n}\\\mathfrak{L}_{\m0}&\mathfrak{L}_\mathfrak{m}\end{bmatrix}$, where
\begin{equation}
    \mathfrak{L}_{\mathfrak{m}0}:=\begin{bmatrix} \Add_1^1\\ \Add^1_2 \\ \vdots \\ \Add^1_n \end{bmatrix}
    ~~~ \& ~~~ \mathfrak{L}_\mathfrak{m}:=\begin{bmatrix} \mathbb{I}_{6\times6}&\mathbb{O}_{6\times 6}&\cdots&\mathbb{O}_{6\times 6}\\ \Add^2_2 &\mathbb{I}_{6\times6}&\cdots&\mathbb{O}_{6\times 6}\\ \vdots &\vdots & \vdots &\vdots \\ \Add^2_n &\Add^3_n&\cdots&\mathbb{I}_{6\times6} \end{bmatrix} . \label{noninertial:Lcomponents}
\end{equation}
and
\begin{equation}
    \Add^j_i=\Ad_{e^{-\xi_{i}q_{i}}\cdots e^{-\xi_jq_j}}, \quad j>i \in \{1, \cdots , n\}. \label{noninertial:Lij}
\end{equation}
%---------------------------------
\begin{lemma}
For a symmetric potential energy under the action of $Q_0^I$,, the Lagrangian of a space-manipulator system on $TQ$ drops to a reduced Lagrangian $\ell^\orbit=\frac{1}{2}\begin{bmatrix}V_\orbit+V_0\\\dot{q}_\mathfrak{m}\end{bmatrix}^T M(q_\mathfrak{m}) \begin{bmatrix}V_\orbit+V_0\\\dot{q}_\mathfrak{m}\end{bmatrix}-u : \q_0^\vee \times TQ_\m\rightarrow \mathbb{R}$,
%\label{lagrangianbig}
where $u:=U(\mathbb{I}_0,q_{\m})\colon Q_\m\rightarrow \mathbb{R}$ is the reduced potential energy, and $M(q_{\m}): (\q_0^\vee  \times T_{q_\m}Q_\m)\times (\q_0^\vee  \times T_{q_\m}Q_\m)\rightarrow \mathbb{R}$ is the reduced mass metric, given by:
\begin{align}
M(q_\mathfrak{m})=&\mathcal{J}^T(q_\m)\bigg(\diag^n_0\{\mathfrak{M}_i\}\bigg)\mathcal{J}(q_\m),
    \label{totalM}\\%\Xi^TL^T (diag\{^i\mathcal{M}_i\})L \Xi
   \mathfrak{M}_i =&\Ad_{\bar{g}^0_i}^{-T}({}^i\mathcal{M}_i)\Ad_{\bar{g}^0_i}^{-1},~~~i=0.\ldots,n
     \label{noninertial:frakM}
\end{align}
where $\mathfrak{M}_i$ is the mass matrix of Body $i$ seen from the spacecraft's coordinate frame in the system's initial configuration.  Here,  $\mathbb{I}_0 \in Q^I_0$ is the identity element of $Q^I_0$.
\end{lemma} 
\begin{proof}[Proof] 
The proof is straightforward computation of $K$ using \eqref{noninertial:totalJ}, noting it's independence of the orbit and spacecraft poses. %Then, we can drop the metric to $\vartheta^\vee\times \g^0_0 \times TQ_\m$ (i.e. remove its dependency on the orbit and spacecraft's poses). \replace{}{
The potential $U$ can also be dropped to $Q_\m$ if it is invariant of the orbit and spacecraft poses.
\end{proof}

%\replace{}{Note that we will discuss the inclusion of $\Theta$-dependent potentials in Remark \ref{?}.}

It can be observed from  \ref{noninertial:Lcomponents}, \eqref{noninertial:totalJ}, and \eqref{noninertial:frakM} that $\ell^\orbit$ is $g^I_0$-independent:
\begin{multline}
        \ell^\orbit(V_\orbit+V _0,q_m,\dot{q}_m)=\\
        \frac{1}{2}\begin{bmatrix}V_\orbit+V_0\\\dot{q}_\mathfrak{m}\end{bmatrix}^T \begin{bmatrix} M_0 & M_{0\mathfrak{m}} \\ M_{0\mathfrak{m}}^T & M_\mathfrak{m} \end{bmatrix} \begin{bmatrix}V_\orbit+V_0\\\dot{q}_\mathfrak{m}\end{bmatrix}-u,
    \label{noninertial:lo}
\end{multline}
where: 
\begin{alignat}{3}
&M_0&&=\mathfrak{M}_0+\mathfrak{L}_{\mathfrak{m}0}^T (\diag^n_1\{\mathfrak{M}_i\})\mathfrak{L}_{\mathfrak{m}0},\label{noninertial:mass0inbase}\\%{}^0\mathcal{M}_0\\
  &M_{0\m}&&={\mathfrak{L}_{\mathfrak{m}0}}^T (\diag^n_1\{\mathfrak{M}_i\}) \mathfrak{L}_\mathfrak{m} \Xi_\mathfrak{m},\label{noninertial:mass0minbase}\\
  &M_{\mathfrak{m}}&&=\Xi_\mathfrak{m}^T\mathfrak{L}_{\m} ^T (\diag^n_1\{\mathfrak{M}_i\}) \mathfrak{L}_\mathfrak{m}\Xi_\mathfrak{m}.
  \label{noninertial:massminbase}
\end{alignat}

\begin{definition} \label{principalconnection}

We define the \textit{spatial mechanical connection}  $\mathcal{A}_I :Q_0\times \q^0_0 \times TQ_\m \rightarrow \q^I_I$ on the principal bundle $Q_0 \times Q_\m \rightarrow Q_\m$ based on the left translation action of $Q_0$ to be the map that assigns to each $(V_\orbit+V_0,\dot{q}_\m)$ the corresponding total velocity of the locked system relative to the orbit expressed in the inertial coordinate frame $\A_I(g^I_0,V_\orbit+V_0,q_\m,\dot{q}_\m)%\begin{bmatrix}V_0\\ \dot{q}_\mathfrak{m}\end{bmatrix}
:=\Ad_{g^\orbit_0}(V_\orbit+V_0+\A(q_\m)\dot{q}_\m)$, with the fiber-wise linear mapping $\A: TQ_\m \rightarrow \q^\vee_0$ defined as 
\begin{equation}
  \A(q_\m):=({M}_0)^{-1}(q_\m){M}_{0\m}(q_\m).
  \label{noninertial:specialconnectioncm}
\end{equation}
We define the \textit{body mechanical connection} $\A_0:Q^I_0\times\q^\vee_0\times TQ_\m \rightarrow \q^\vee_0$ on the principal bundle $Q^I_0 \times Q_\m \rightarrow Q_\m$ based on the right action of $Q^I_0$ as the map that assigns to each $(V_\orbit+V_0,\dot{q}_\m)$ the corresponding total velocity of the locked system relative to the orbital frame expressed in the spacecraft's coordinate frame $\A_0(V_\orbit+V_0,q_\m,\dot{q}_\m):=V_\orbit+V_0+\A(q_\m)\dot{q}_\m$.
Note that $\A$ is independent of the elements of $H$ and $Q_0$. \qed %
\end{definition} 

Let
\begin{equation}
    \OOs={}^0\Omega^I_{loc}:=V_\orbit+V_0+\A\dot{q}_\m
    \label{noninertial:lockedorbitalvelocitywrtorbit}
\end{equation} 
be the absolute angular velocity of the locked  system (the equivalent instantaneously locked rigid system) about its CoM  expressed in the spacecraft's coordinate frame.
Both these mechanical connections are defined on the tangent bundle of the relative-to-orbit configuration manifold $TQ_0 \times TQ_\m$.

\begin{lemma}
By the change of variables $(V_\orbit+V_0,\dot q_\m) \mapsto (\OOs ,\dot q_\m)$, the Lagrangian of a space-manipulator system $\ell^\orbit$ in \eqref{noninertial:lo} transforms into 
\begin{equation}
    \ell^{cm}=\frac{1}{2}\begin{bmatrix}\OOs\\\dot{q}_\m\end{bmatrix}^T \begin{bmatrix}  M_0 & \mathbb{O}_{6 \times n} \\ \mathbb{O}_{n \times 6} & \hat{M}_\mathfrak{m} \end{bmatrix} \begin{bmatrix} \OOs \\\dot{q}_\m\end{bmatrix}-u,
    \label{noninertial:lcm}
\end{equation}
with a block-diagonal mass matrix for expressing the kinetic energy of the system. Here, we have defined
\begin{equation}
    \hat{M}_\m={M}_\m-\A^T{M}_\orbit \A.\label{noninertial:massmatrix}
\end{equation}
\end{lemma}

\begin{proof}
The proof is a straightforward computation.
\end{proof}
Based on $l^{cm}$ in \eqref{noninertial:lcm} and $\OOs$ in \eqref{noninertial:lockedorbitalvelocitywrtorbit}, we can now introduce the generalized momentum of the locked system relative to the orbit about the space manipulator's CoM  $P \in {\q^0_0}^*$ represented in the spacecraft's coordinate frame:
\begin{align}
    P=\frac{\partial l^{cm}}{\partial \OOs}&={M}_0 \OOs \nonumber\\
    &=\underbrace{M_0 V_\orbit}_{P_\orbit} +\underbrace{M_0V_0 + M_{0} \A\dot{q}_\mathfrak{m}}_{P_0} \in (\q^\vee_0)^*\label{noninertial:momentum}.
\end{align}
 \begin{remark}
The base–arm contribution to the total momentum of the system observed from the spacecraft frame is labeled
$P_0$, and the orbital contributaions are grouped in $P_{\orbit}$.
Our formulation (LPKE equations in \eqref{noninertial:EoM}–\eqref{noninertial:vcurlyevolution}) will admit both the standard zero-momentum and nonzero-momentum initializations without modification. The commonly used “zero total momentum” setup for free-floating systems (with no external wrench and no symmetry-breaking orbital terms)
is recovered by initializing $P_\orbit=0$, and forcing the total spatial momentum in inertial frame $\mu$ to be conserved. In this case $P_0$ is a \emph{rotated observation} of the constant $\mu$ in the moving spacecraft frame. Non-zero initial
momenta are handled by initializing $(P_0,P_\orbit)$ accordingly: analytic Keplerian terms enter as symmetry-breaking inputs, $P_{\orbit}\neq 0$ and the base twist acquires a bias. Consequently $P_0$ evolves in time, reflecting orbit–attitude coupling and non-inertial effects. 
\end{remark}

By introducing the following assumption, we can use Kepler's equations \citep{curtis2013orbital} with a conserved orbital angular momentum $\mu_\orbit$ to find the evolution of $P_\orbit$:
\begin{Assumption}\label{noninertial:assumption}
The locked system's orbital motion is under no external forces, resulting in conservation of the locked orbital angular momentum of the system $\mu_\orbit$ relative to the quasi-inertial frame, where ${}^I\breve{P}^I_\orbit=[ \mu_{\orbit_x} ~~ \mu_{\orbit_y} ~~\mu_\orbit ]^T$.
\end{Assumption} 
\begin{lemma}\label{noninertial:lemmaPorbit}
For the space manipulator moving relative to an orbital frame with constant angular momentum $\mu_\orbit$, orbital frame pose $g^I_\orbit$ and velocity ${}^0\V_\orbit^I$ from Lemma \ref{orbitkinlemma}, the conserved the locked orbital angular momentum seen in spacecraft frame $P_\orbit$ evolves according to:
\begin{align}
    &P_\orbit(\theta,q_\m)=M_0  \frac{\mathcal{G}m_\oplus }{\mu_\orbit}\Ad_{g^\orbit_0}\iota_\orbit\begin{bmatrix}  -\sin(\theta-\bar{\theta})\\ 1-\cos(\theta-\bar{\theta})\\\frac{\mathcal{G}m_\oplus} {\mu_\orbit^2}(1+e_\orbit \cos(\theta))^2\end{bmatrix} \nonumber\\
    &\dot{P}_\orbit(\theta,q_\m)=\frac{\mathcal{G}m_\oplus}{\mu_\orbit}\Bigg(\left(\sum^n_{i=1}\frac{\partial M_0}{\partial q_i}\dot{q}_i\right)\Ad_{g^\orbit_0}\nonumber\\
    &~~~~~~~~~+M_0\left(\ad_{V_0}\Ad_{g^\orbit_0}\right)\Bigg)\iota_\orbit \begin{bmatrix}  -\sin(\theta-\bar{\theta})\\ 1-\cos(\theta-\bar{\theta})\\\frac{\mathcal{G}m_\oplus} {\mu_\orbit^2}(1+e_\orbit \cos(\theta))^2\end{bmatrix}\nonumber\\
    &~~~~~~~~~+\frac{\mathcal{G}m_\orbit^3}{\mu_\orbit^4}M_0\Ad_{g^\orbit_0}\iota_\orbit \left(1+e \cos (\theta)\right)^2\nonumber\\
    &~~~~~~~~~~~~~~~~~~~~~~~~~~~~~~~~\begin{bmatrix}
    -\cos(\theta-\bar{\theta})\\\sin(\theta-\bar{\theta})\\ -\frac{2 \mathcal{G}m_\oplus
    }{\mu_\orbit^2}e_\orbit sin(\theta)(1+e_\orbit cos(\theta))
    \end{bmatrix}.\label{noninertial:Porbit in time}
\end{align}
%\begin{equation}P_\orbit=M_0  \frac{\mathcal{G}m_\oplus }{\mu_\orbit}\Ad_{g^\orbit_0}\iota_\orbit\begin{bmatrix}  -\sin(\theta-\bar{\theta})& 1-\cos(\theta-\bar{\theta})&\frac{\mathcal{G}m_\oplus} {\mu_\orbit^2}(1+e_\orbit \cos(\theta))^2\end{bmatrix}^T \label{noninertial:Porbit in time}\end{equation}
%\begin{equation}P_\orbit=\Ad^T_{g^\oplus_0}\iota_\orbit\begin{bmatrix}  - \frac{\mathcal{G}m_\oplus m_{tot}}{\mu_\orbit}\sin(\theta) & \frac{\mathcal{G}m_\oplus m_{tot}}{\mu_\orbit}(e_\orbit +\cos(\theta))& \mu_\orbit            \end{bmatrix}^T \end{equation}
%\begin{multline}\dot{P}_\orbit(\theta,q_\m)=\frac{\mathcal{G}m_\oplus}{\mu_\orbit}\Bigg(\left(\sum^n_{i=1}\frac{\partial M_0}{\partial q_i}\dot{q}_i\right)\Ad_{g^\orbit_0}\iota_\orbit \V_\orbit+M_0\left(\ad_{V_0}\Ad_{g^\orbit_0}\right)\iota_\orbit \V_\orbit\Bigg)\\+M_0\Ad_{g^\orbit_0}\iota_\orbit \frac{\mathcal{G}m_\orbit^3}{\mu_\orbit^4}\left(1+e \cos (\theta)\right)^2\begin{bmatrix}-\cos(\theta-\bar{\theta})\\\sin(\theta-\bar{\theta})\\ -\frac{2 \mathcal{G}m_\oplus}{\mu_\orbit^2}e_\orbit sin(\theta)(1+e_\orbit cos(\theta))\end{bmatrix}. % \begin{bmatrix}\frac{(\mathcal{G} m_\oplus)^2(1+e_\orbit cos(\theta))^2}{\mu_\orbit^3 }cos(\theta-\bar{\theta})\\0\\ -\frac{2 \mathcal{G}m_\oplus e_\orbit sin(\theta)}{\mu_\orbit^2(1+e_\orbit cos(\theta))^4}\end{bmatrix}\end{multline}
\end{lemma}
\begin{proof}
    Proof is detailed in Appendix \ref{app:lemmaPorbit}. \qedsymbol
\end{proof}
\begin{remark}
    Since \( P_\odot \) and \( \dot{P}_\odot \) are computed analytically at each time-step, rather than being obtained through full dynamic integration over time, their associated numerical instability and computational complexities are effectively mitigated.
\end{remark}
Let us define the variation 
$\hat{\eta}_0:=({g^\orbit_0})^{-1}\delta g^\orbit_0 \in \g^0_0 $
%\label{noninertial:etadefinition0}
as the left translation of the variation of the curve $g^\orbit_0(t)\in Q_0$ to the lie algebra $\g^0_0$, with the conditions $\eta_0(t_0,\epsilon)=\eta_0(t_f,\epsilon)=0$ (for rigorous definition of the variation operator $\delta$ refer to \citep{moghaddam2022}).
%We can also define the variation\begin{equation}\hat{\eta}_\orbit:={g^I_\orbit}^{-1}\delta g^I_\orbit \in \q^\orbit_\orbit ,\label{noninertial:etadefinitiono}\end{equation}as the left translation of the variation of the curve $\theta(t)\in \Theta$ to the lie algebra $\q^\orbit_\orbit$, with the conditions\begin{equation*}\eta_\orbit(t_0,\epsilon)=\eta_\orbit(t_f,\epsilon)=0,\end{equation*}
We have presented the closed form solution for the evolution of the orbital pose and velocity in Lemma \ref{orbitdyn} and \eqref{orbitvel} based on the well-known Kepler's solution\citep{curtis2013orbital}. Therefore, since the exact evolution of $g^I_\orbit$ is known, the variation $\hat{\eta}_\orbit=(g^I_\orbit)^{-1}\delta(g^I_\orbit)=0$ is null for all the calculations.

A generalized force exerted at the coordinate frame attached to Body $j$ consists of a linear force $f_j \in \mathbb{R}^3$ and an angular moment $\tau_j \in \mathbb{R}^3$, where  $\begin{bmatrix} f_j^T & \tau_j^T
\end{bmatrix}^T\in(\g_j^\vee)^*$ is called a \textit{body wrench}. The power generated by a wrench can be calculated by the pairing with the body twist of the frame where the wrench is applied as $\left<\begin{bmatrix} f_j^T & \tau_j^T
\end{bmatrix}^T,{}^jV_j^i\right>$.

\begin{lemma}\label{noninertial:varlemma}
A space-manipulator system with the Lagrangian in \eqref{noninertial:lagrangianfirst} and the applied force $F=(F_0,F_m)\in T^*Q\cong  \times T^*Q_0\times T^*Q_\m$ satisfies the Hamilton-d'Alambert principle:
\begin{equation}
    \delta \int^{b}_{a} \L(q,\dot{q})dt-\int_{a}^{b} \left<F,\delta q\right>dt=0, %+\left<f_\E,{}^eV^I_e\right>\big)
    \label{noninertial:hamiltonforL}
\end{equation}
for variations of the type $\delta q$, if and only if $\ell^{cm}$ in \eqref{noninertial:lcm} satisfies the Hamilton-d'Alambert principle \citep{blochbook,moghaddam2022}:
\begin{equation}
    \delta \int\big(l^{cm}(q_\mathfrak{m},\dot{q}_\mathfrak{m},\OOs)- \big<F_{\eta_0},\eta_0\big> +\left<F_{\m},\delta q_\m\right>
    %+\left<F_{ee},\eta_{ee}\right>
    \big) dt=0
    \label{noninertial:hamiltonforlcm}
\end{equation}
for variations of the type:
\begin{equation}
    \delta (\OOs)= \dot{\eta_0}+ [{}^0V^\orbit_0,\eta_0] + \A\delta \dot{q}_\mathfrak{m}+\delta \A\dot{q}_\mathfrak{m},
    \label{noninertial:varomega0}
\end{equation}
and
\begin{equation}
    F_{\eta_0}=T_{\mathbb{I}}^*L_{g_0^\orbit}F_{0}\in(\{\g_0^0\}^\vee)^*.
\end{equation}
The mapping $T_{\mathbb{I}}^*L_{g_0^\orbit}\colon T^*_{g_0^0}Q_0\rightarrow (\{\q_0^0\}^\vee)^* $ is the dual of the tangent map corresponding to the left translation at the identity $\mathbb{I}\in Q_0$. %and the mapping $T_{e_0}^*L_{g_0^I}\colon T^*_{g_0^I}Q_0^I\rightarrow (\{\q_0^0\}^\vee)^* $ is the dual of the tangent map corresponding to the left translation at the identity $e_0\in Q_0^0$.
\end{lemma}
\begin{proof} 
The proof is provided in  Appendix \ref{app:variations}  along with the detailed derivation of the forcing term $f_\eta$ in \citep{moghaddam2022}.
\end{proof}
\begin{lemma}
The total body wrench collocated with the vehicle velocity ($F_\eta$) and the total force collocated with the joint velocities are found from: 
\begin{equation}
    F_\eta=f_{0}+J_{\E,0}^Tf_\E \in (\h^\vee)^*, \quad \& \quad 
    F_\m=f_\m+J_{\E,\m}^Tf_\E \in T^*Q_\m,\label{f.m}
\end{equation}
where the end-effector Jacobians are:
\begin{align}
    J_{\E,0}(q_\m)&= \Ad_{g^0_n}^{-1},\\ J_{\E,\m}(q_\m)&= \Ad_{g^0_n}^{-1}\begin{bmatrix}  {\xi}_{1}\cdots \Ad_{e^{\hat{\xi}_{1}q_{1}}\cdots e^{\hat{\xi}_{n-1}q_{n-1}}}{\xi}_n \end{bmatrix},
    \label{eeJacobian}
\end{align}
with vehicle wrenches collocated with the base vehicle's body velocity $f_0 \in ((\g^0_0)^\vee)^*$, %, including the external wrench ($f_{0,ext}$) and the control ($f_{0,c}$) wrench acting on the vehicle \begin{equation} f_0=f_{0,ext}+f_{0,c} \in (\q^0_0)^*, \label{f_0} \end{equation} the virtual work associated to which is calculated from:\begin{equation}     \delta W_0=\left<f_0,\eta\right>, \end{equation}
 arm forces collocated with the joint velocities $f_\m \in T^*Q_\m$, %including both the external wrenches ($f_{\m,ext}$) and actuation torques ($f_{\m,c}$) acting at the joints \begin{equation}    f_\m=f_{\m,ext}+f_{\m,c} \label{f_m} \in T^*Q_\m ,\end{equation} the virtual work associated to which is calculated from: \begin{equation}     \delta W_\m=\left<f_\m,\delta q_\m\right>, \end{equation}
and the external wrench acting at the end-effector $f_\E \in (\g_\E^\vee)^*$, 
\end{lemma}
\begin{proof}
    Proof is provided in Appendix D of \citep{moghaddam2022}.
\end{proof}

\begin{theorem} [\textbf{Lagrange-Poincar\'{e}-Kepler Equations for Spacecraft-Manipulator Systems in Orbit}]\label{noninertialtheorem:EoM} Given a space-manipulator system moving relative to an undisturbed  orbit defined by its constant angular momentum ${\mu}_\orbit$, eccentricity $e_\orbit$, gravitational constant $\mathcal{G}m_\oplus$, and initial orbital true anomaly $\bar{\theta}$ with the orbital configuration $g^I_\orbit$ and base configuration $g^\orbit_0$ such that $g^I_\orbit g^\orbit_0\in Q_0^I$, a set of joint angles $q_\mathfrak{m}\in Q_\m$, input wrenches $f_{0}\in (\q_0^0)^*$, %$f_{ee} \in (\q^n_n)^*$ 
$f_\mathfrak{m}\in T^*Q_\m$, and $f_\E\in (\q^\E_\E)^*$ that are collocated with the spacecraft's body velocity, the joint velocities, and the end-effector's body velocity, respectively, and the  $\L$agrangian $\L$ in \eqref{noninertial:lagrangianfirst} that is invariant with respect to the base spacecraft's configuration, the singularity-free  EoM  reads:
\begin{multline}
    \begin{bmatrix}
        \mathbb{I}_{b} & \!\!\!\mathbb{O}_{b\times n}\!\\ \mathbb{O}_{n\times b} & \!\!\!{\hat{M}_{\m}}\!
    \end{bmatrix}\begin{bmatrix}
            \dot{P}_0\\\ddot{q}_\m
        \end{bmatrix}\!+\!\begin{bmatrix}
        -\tilde\ad_{P_0} & \mathbb{O}_{b\times n}\\ \mathbb{O}_{n\times b} & \hat{C}_{\m}%~ \& ~centrifugal}
    \end{bmatrix}\begin{bmatrix}
            V_0\\ \dot{q}_\m
        \end{bmatrix}\!\!\\
        =\!\!\begin{bmatrix}
           F_{\eta_0}-\hat{F}_\orbit\\ \!\hat F_\m\!-\!\hat N_\m P_0-\hat{N}_\orbit\!\!
        \end{bmatrix}\!,\label{noninertial:EoM}
\end{multline}
\begin{equation}
    V_0=\big((g^\orbit_0)^{-1}\dot{g}^\orbit_0\big)^\vee=\hat{M}_0^{-1}(P_0-\A\dot{q}_\m) \label{noninertial:vcurlyevolution}
\end{equation}
%\begin{equation}   -\dot{P}_0+\textbf{ad}^T_{^0V^\orbit_0}P_0=f_{0}+(J_{\E,0})^Tf_\E+f_{\orbit,0}\label{noninertial:euler-poincare_eq}\end{equation}
%\begin{equation}{\hat{M}_{m}\ddot{q}_\mathfrak{m}+\hat{C}_{m}\dot{q}_\mathfrak{m}+\hat{N}_{m}+\frac{\partial u}{\partial q_\m}=f_\mathfrak{m}-\A_{\orbit,\m}^Tf_{0}+(J_{\E,\m}^T-\A_{\orbit,\m}^T(J_{\E,0})^T)f_\E} \label{noninertial:internaldynamics}\end{equation}
where the manipulator's mass matrix $\hat{M}_\m(q_\m)$ and the locked mass matrix $M_0$ are defined in \eqref{noninertial:massmatrix}, the map $\A$ is defined in Definition \eqref{principalconnection}. Further,  
\begin{align}
     &\hat{C}_{\m}=\sum_{i=1}^{n}\frac{\partial \hat{M}_\mathfrak{m}}{\partial{q}_i}\dot{q}_i-\frac{1}{2}\begin{bmatrix}\frac{\partial \hat{M}_\mathfrak{m}}{\partial q_1}\dot{q}_\m~\cdots~\frac{\partial \hat{M}_\mathfrak{m}}{\partial q_n}\dot{q}_\m\end{bmatrix}^T\label{noninertial:coriolis}\\
     &\hat{N}_\m=\begin{bmatrix}
       \dot{q}_\m^T\frac{\partial \A^T}{\partial q_1} \\ \vdots \\\dot{q}_\m^T\frac{\partial \A^T}{\partial q_n}
       \end{bmatrix}+\frac{1}{2}\begin{bmatrix}P_0^TM_0^{-1}\frac{\partial M_0}{\partial q_1}M_0^{-1}\\  \vdots\\ P_0^TM_0^{-1}\frac{\partial M_0}{\partial q_n} M_0^{-1}\\\end{bmatrix}\nonumber\\
       &~~~~~~~~~~~~~-\A^T\textbf{ad}^T_{V_0}-\bigg(\sum_{i=1}^n (\frac{\partial \A}{\partial q_i}) \dot{q}_i\bigg)^T\label{noninertial:Nhat0}\\
       &\hat{N}_\orbit=-\A^T\dot{P}_\orbit-\big(\sum_{i=1}^n (\frac{\partial \A}{\partial q_i}) \dot{q}_i\big)^TP_\orbit\nonumber\\
       &~~~~~~~~~~~~~~+\begin{bmatrix}\frac{\partial \A}{\partial q_1}\dot{q}_\m & \cdots &\frac{\partial \A}{\partial q_n}\dot{q}_\m\end{bmatrix}^TP_\orbit\nonumber\\
    &~~~~~~~~~~~~~~+\begin{bmatrix}(M_0^{-1}\frac{\partial {M}_0}{\partial q_1}M_0^{-1})^T(P_\orbit+\frac{P_0}{2})\\  \cdots\\ (M_0^{-1}\frac{\partial M_0}{\partial q_n} M_0^{-1})^T(P_\orbit+\frac{P_0}{2})\end{bmatrix}^T P_\orbit\label{noninertial:Nhatorbit}\\
    &\hat{F}_\m\!=-\frac{\partial u}{\partial q_\m}\! +\! f_\mathfrak{m}\!-\!\A^Tf_{0}\!+\!(J_{\E,\m}^T\!-\!\A^TJ_{\E,0}^T)f_\E \nonumber\\
    &\hat{F}_\orbit=-\dot{P}_\orbit+\textbf{ad}^T_{V_0}P_\orbit%-\Ad^T_{g^\orbit_0}\dot{P}^\orbit_\orbit
    \label{noninertial:f.mhat}%\ad^T_{V_0}{P}_\orbit-\ad_{V_0}^T\Ad_{g^\orbit_0}^TP^\orbit_\orbit
\end{align}
where the evolution of $P_\orbit(\theta,\dot{q}_\m)$ and $\dot{P}_\orbit(\theta,\dot{q}_\m)$ is found from Lemma \ref{noninertial:lemmaPorbit}.
Orbital pose $g^I_\orbit$ and velocity $\V_\orbit$ evolve in time according to Equations \eqref{eq:orbitpose} and \eqref{orbitvel} in Lemma \ref{orbitkinlemma}, respectively. Evolution of orbital true anomaly $\theta$ is found from Equation \eqref{noninertial:trueanomalyevolution} in Lemma \ref{orbitdyn}.
\end{theorem}
\begin{proof} 
The proof is provided in  Appendix \ref{app:prooftheorem} .
\end{proof}
For detailed calculation of the partial derivatives $\frac{\partial M_0}{\partial{q}_i}$, $\frac{\partial \hat{M}_\mathfrak{m}}{\partial{q}_i}$, and $\frac{\partial \A}{\partial{q}_i}$ in the LPKE in Theorem \ref{noninertialtheorem:EoM}, refer to \citep{moghaddam2022}.
\begin{remark}
    Describing the space manipulator's states and motion relative to the orbital frame ensures well-conditioned and stable EoM for numerical propagation. Directly integrating the motion in the inertial frame would lead to ill-conditioned EoM due to the large disparity in magnitudes between orbital motion (approximately $7.8 [km/s]$) and the manipulator’s relative motion (in the order of $1 [m/s]$). By incorporating Kepler’s equation into the formulation and adopting a quasi-inertial frame, these numerical issues are effectively mitigated.
\end{remark}

\begin{remark}
In the LPKE setting, the end–effector kinematics is  $V_e = J_{\E,0}(q)\,(V_\orbit + V_0) + J_{\E,\m}(q)\,\dot q,$ and, using the mechanical connection, the disturbed base motion is $V_0 = -\,\A(q)\,\dot q + M_0^{-1}P_0,$ where the bias $M_0^{-1}P_0$ collects  the analytic Keplerian forcing. Substituting gives $V_e = \big(J_{\E,\m} - J_{\E,0}\A\big)\dot q + J_{\E,0}M_0^{-1}P_0,$ so the total Jacobian $\mathbb{J}=J_{\E,\m}-J_{\E,0}\A$ governs the joint–rate contribution, while $J_{\E,0}M_0^{-1}P_0$ captures LPKE corrections due to orbital forcing. In the undisturbed, zero–momentum case with no symmetry–breaking orbital effects $V_\orbit\equiv 0$, $b_{\orbit}\equiv 0$, and $P_0=0$, this reduces to $V_e=\mathbb{J}(q)\,\dot q,$ which recovers exactly the Generalized Jacobian Matrix (GJM) of Umetani and Yoshida~\citep{umetani1987continuous}.
\end{remark}

%{states wrt orbit frame, motion wrt orbit is small compared to orbit motion. if you integrate the entire motion wrt to inertial entirely equations will be ill conditioned. this is the 3rd contrib. if we integrate kepler's eq into eom and choosing quasi-inertial. 1 motion is 7km/s the other motion is in order of 1m/s at least 3 orders of magnitude difference. p circ and p dot circ are analytically calculated at each step, instead of a full dynamics integration in time. also add to introduction: if you solve all, u end up with ill conditioned equationelsevier robotics and autonomous sys aim space robotics include motion of base in results} 

\section{Numerical Validation}\label{noninertial: Numerical Study}
We validate the proposed Lagrange–Poincaré–Kepler Equations (LPKE), derived in Theorem \ref{noninertialtheorem:EoM}, by comparing its numerical behavior against a previously validated benchmark Simscape Multibody model of a 7-DoF  space manipulator onboard a free-floating spacecraft moving within a non-inertial orbital framework \citep{moghaddam2022,moghaddam2021guidance}. All simulations are conducted using MATLAB/Simulink/Simscape (codebase available at: \url{https://github.com/BmMoghaddam/Orbital-paper}).

We perform a comparative study of three simulation approaches against the benchmark to establish the accuracy of the numerical simulation based on the LPKE in \eqref{noninertial:EoM} in capturing the effects of the moving orbital reference frame:
\begin{itemize}
\item \textbf{Benchmark}:  previously validated Simscape model simulating the manipulator in a moving orbital frame with high fidelity and precise numerical integration\citep{moghaddam2022}. The benchmark is placed on the exact same orbital trajectory as all other modes and uses identical joint configurations and time-steps, thereby serving as a high-accuracy reference for evaluating model performance.
\item \textbf{Mode I}: LPKE-based dynamics in Theorem \ref{noninertialtheorem:EoM} using Keplerian orbital propagation relative to the quasi-inertial frame, as per Lemma \ref{orbitkinlemma}. 

\item \textbf{Mode II}: LPKE-based dynamics with orbital propagation relative to the perifocal frame. 

\item \textbf{Mode III}: LPE model with fully coupled orbital dynamics and integration of orbital motion via fundamental Newton's gravitational equations. 
\end{itemize}
All simulation modes capture Coriolis and coupling effects introduced by orbital motion into the locked spacecraft-manipulator dynamics. Due to its physical modeling and validated multibody representation, the benchmark provides a ground truth for evaluating the LPKE approximations in the proposed modes.%\citep{future:moghaddam_disturbed}.

\subsection{Simulation Platform}

The \textit{Benchmark} model employs a previously validated Simscape Multibody model representing a 7-link,  7-DoF  KUKA LBR iiwa  robotic manipulator mounted on a  $10 kg$  cubic spacecraft (see Figure \ref{fig:simmodel})\cite{moghaddam2022}. The model incorporates a 6-DoF  joint to capture the relative motion of the end-effector with respect to the orbital frame, alongside seven sequential revolute joints parameterized using twist representations according to \eqref{expparam}.
\begin{figure}
    \centering
    \begin{subfigure}[b]{0.49\columnwidth}
        \centering
        \includegraphics[width=2.8cm]{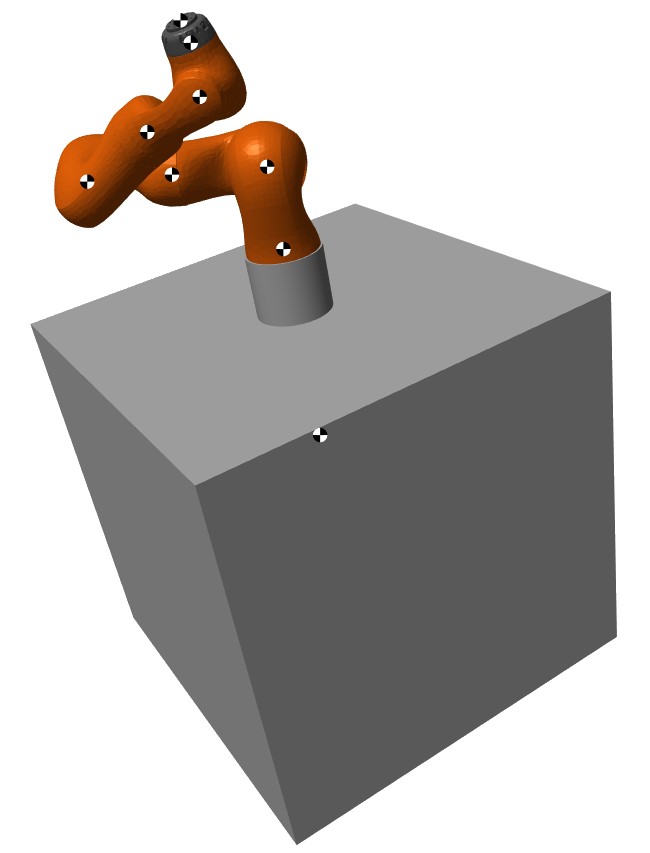}
        \caption{Simscape model}
        \label{fig:simmodel}
    \end{subfigure}
    \hfill
    \begin{subfigure}[b]{0.49\columnwidth}
        \centering
        \includegraphics[width=2cm]{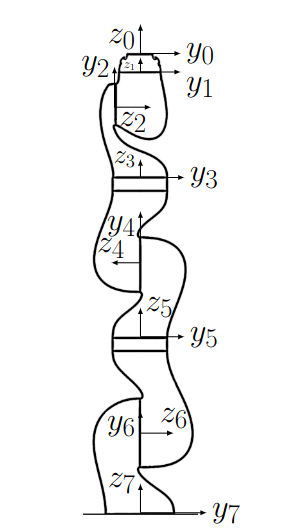}
        \caption{Arm Frame Assignment}
        \label{sim:frameassignment}
    \end{subfigure}
    \caption{Simscape model and the arm frame assignment}
\end{figure}
Table \ref{tab:inertia} presents the geometric and inertial characteristics of the simulated system, which are derived from the realistic specifications of a KUKA LBR iiwa 7-DoF  robotic manipulator mounted on a $10 kg$ cubic spacecraft.  While chaser spacecraft in on-orbit servicing (OOS) missions are typically heavier than their manipulators, we intentionally select a relatively low base mass (10~kg) here to stress–test base–arm coupling. This challenging mass ratio amplifies reaction dynamics and produces larger base–attitude excursions, allowing us to demonstrate that the LPKE formulation faithfully captures strong momentum coupling and large angular motion of the base. The objective is methodological—evaluating the formulation under edge-case conditions—rather than prescribing mission sizing.

\begin{table}[htbp]
  \centering
  \caption{Geometric and inertial properties of the simulated system}
\resizebox{!}{0.19\columnwidth}{\begin{tabular}{|c|cccc|}
\hline
 Body & {$l_i$} [$10^{-2}m$] & $\rho_{cm}$ [$10^{-2}m$] & $m_i$[$kg$] & $I_{xx} ,~ I_{yy},~I_{zz}$ [$kg\times m^2$] \\ \hline
  7 & $100$ & -$\begin{bmatrix}  0&0 &50 \end{bmatrix}$  & 10 &  $\begin{bmatrix}1670~ & 1670& 1670\end{bmatrix}$  \\
6    & $19$  & -$\begin{bmatrix}  0&0 &19 \end{bmatrix}  $     & 3.4525 &  $\begin{bmatrix}74.7~~ & 57.4~& 23.9   \end{bmatrix}$  \\
 5 & $21$      & -$\begin{bmatrix}  0&0 &21 \end{bmatrix} $   & 3.4822 &  $\begin{bmatrix}39.0 ~~& 27.9~ & 19.9  \end{bmatrix}$ \\
 4   & $19$    & -$\begin{bmatrix}  0&19 &0 \end{bmatrix}$      & 4.0562  &  $\begin{bmatrix} 104.2~  &78.3~ &34.1 \end{bmatrix}$   \\
 3      & $21$      & -$\begin{bmatrix}  0&0 &21 \end{bmatrix}$    &3.4822 &  $\begin{bmatrix}41.4~~ &24.8 ~& 23.4\end{bmatrix}$  \\
2    & $19$     & -$\begin{bmatrix}  0&19 &6 \end{bmatrix}$  & 2.1633 &  $\begin{bmatrix}
     16.37& 18.27& 12.1 
 \end{bmatrix}  $    \\
 1 & $8.1$    & -$\begin{bmatrix}  0&6 &8.1 \end{bmatrix}$     &  2.3466 &  $\begin{bmatrix}6.5~~~~ & 6.3~~~ & 4.5\end{bmatrix}$  \\
0 & $4.5$   & -$\begin{bmatrix}  0&0 &4.5 \end{bmatrix}$ & 3.129    &  $\begin{bmatrix}
     15.9~~& 15.9~~& 2.9  
 \end{bmatrix} $ \\\hline
\end{tabular}}
\label{tab:inertia}
  
\end{table}
 
The moments of inertia for links are computed based on their geometries and align with the publicly available properties of the  KUKA  LBR iiwa7 model \citep{KUKALBRiiwa}. Moments of inertia of the spacecraft are simply computed from $\mathcal{I}_7=\big({m_7 l_7^2}/{6}\big)\mathbb{I}_{3 \times 3}$, assuming uniform density and cubic geometry. Here $\mathcal{I}_7, m_7, l_7$ are inertia, mass and length of the spacecraft, respectively. Note that the spacecraft and the end-effector are indexed $7$ and $0$, respectively. All inertial properties (Table~\ref{tab:inertia}) are matched across simulation modes to ensure consistency. The manipulator is initialized as follows:

\begin{equation}
\begin{aligned}
q_\m &= \begin{bmatrix}0.2 & 1.2&0.3&2.5&0.5&1.5&0.6 \end{bmatrix}^T, \\
\dot{q}_\m &= 10^{-2}\times \begin{bmatrix} 0 & 1 & 0 & 1 & 0 & 0 & 0 \end{bmatrix}^T.
\end{aligned}
\end{equation}

we extracted the linear positions of the joints relative to Body $0$ (end-effector) in the upright configuration according to Figure \eqref{sim:frameassignment}:
\begin{align*}
    &\rho_1=\begin{bmatrix}0 &0 &-4.5 \end{bmatrix}^T, 
    &\rho_2=\begin{bmatrix}0 &-6 &-12.6 \end{bmatrix}^T,\\ &\rho_3=\begin{bmatrix}0 &0 &-31.6\end{bmatrix}^T, &~~~~
    \rho_4=\begin{bmatrix}0 &0 &-52.6 \end{bmatrix}^T,\\ &\rho_5=\begin{bmatrix}0 &0 &-71.6 \end{bmatrix}^T,
    &~~~~\rho_6=\begin{bmatrix}0 &0 &-92.6 \end{bmatrix}^T,\\
    &\rho_7=\begin{bmatrix}0 &0 &-111.6 \end{bmatrix}^T, &[cm]
\end{align*}\normalsize
and the axes of rotation of the revolute joints $1,\cdots,7$ in the shown upright configuration as:
\begin{align*}
    w_1=w_3=w_5=w_7=&\begin{bmatrix}0 & 0 & 1\end{bmatrix}^T\\
    w_2=-w_4=w_6=&\begin{bmatrix}0 & 1 & 0\end{bmatrix}^T
\end{align*}\normalsize
The \textit{benchmark} space manipulator model moves relative to a non-inertial orbital frame which itself moves on an orbital path relative to a central inertial perifocal frame. Orbital elements of this orbit are configured for a near-circular GEO orbit (Table~\ref{tab:orbitalparameter}).
\begin{table}[htbp]
\caption{Orbital Parameters}
  \centering
\begin{tabular}{|c|cccccc|}
\hline
 Property & $\Omega$ & $\omega$ & $i$ & $a ~ [km]$ & $e$ &$\bar{\theta}|_{t=0}$ \\ \hline
  Value & $0$ & $0$  & $0$ &  $42157.08431$ & $2\times 10^{-6}$ & $0$  \\\hline
\end{tabular}
\label{tab:orbitalparameter}
\end{table}
In Table \ref{tab:orbitalparameter}, $i$ is the inclination, which is the angular distance between the orbital plane and Earth's equator, $\Omega$ is the right ascension of the ascending node (intersection of the orbital plane and the equatorial plane), $\omega$ is the argument of periapsis defining the angle between the ascending node and the perigee of the orbit ellipse,  $e$ is the orbit eccentricity, $a$ is the semi-major axis of the orbit,$\bar{\theta}|_{t=0}$ is the true anomaly at the time of start of the simulation. As can be seen, through putting $\Omega=0,~\omega=0~\& ~ i=0$ the orbital plane is chosen such that it overlaps with the Earth's equator. This is valid since the orientation of the orbital plane is inconsequential to the shape and amount of orbital disturbing effects on the space manipulator. 
\begin{figure}
    \centering
    \includegraphics[width=0.99\columnwidth]{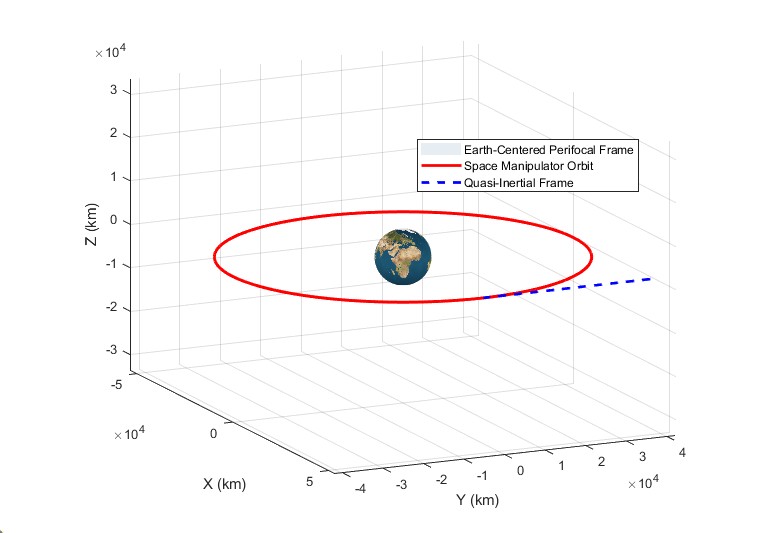}
    \caption{3D representation of the orbit and the quasi-inertial frame's paths around Earth}
    \label{fig:3d}
\end{figure}
Figure \ref{fig:3d} shows the three-dimensional orbital path of the spacecraft-manipulator system, visualized with respect to the Earth-centered inertial frame. This provides spatial context for the elliptical trajectory used in the simulation, along with the trajectory of the quasi-inertial frame. The orbit is chosen so that it represents the deviation of a chaser space manipulator's path from an orbital debris in GEO orbit in common vicinity operations. The offset between the initial velocity of the space manipulator orbit and a perfect GEO orbit is $7~[cm/s]$ , corresponding to the relative velocity of two objects on the same GEO orbit that are $|\rho^n_\t|=1.2684m$ or $\Delta \theta=2.3721\times 10^{-05} [rad]$ apart \citep{curtis2013orbital}, faithful to the initiation configuration of an in-orbit space manipulator operation. The orbital trajectory deviates from a reference GEO path by up to $74cm$ in the first $120s$ and by $15~m$ over a full orbital revolution. Figure \ref{fig:orbitalpaths} establishes a verification of: (i) the simulated orbital motion relative to the perifocal frame, (ii) the simulated motion of the quasi-inertial frame relative to the perifocal frame, and (iii) the relative path of the orbital frame to the quasi-inertial frame observed from the quasi-inertial frame.
\begin{figure}
    \centering
    \begin{minipage}{0.32\columnwidth}
    \includegraphics[width=2.9cm]{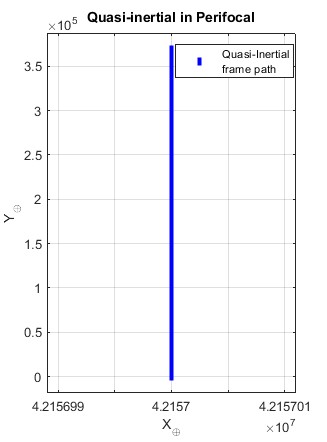}
    \end{minipage}
    \begin{minipage}{0.32\columnwidth}
    \includegraphics[width=2.9cm]{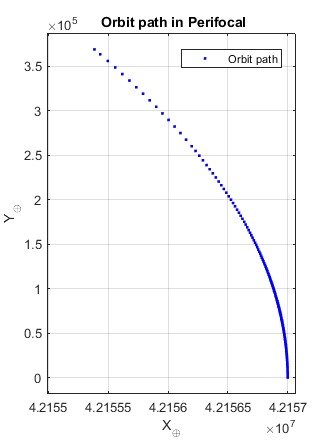}
  \end{minipage}    
  \begin{minipage}{0.32\columnwidth}
    \includegraphics[width=2.9cm]{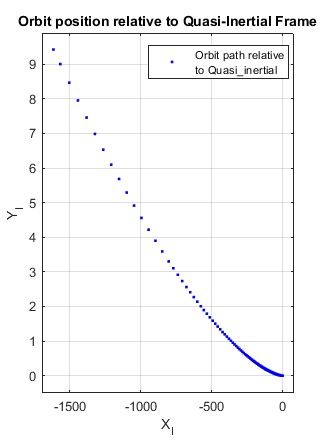}
  \end{minipage} 
  \caption{Path of the Orbital and Quasi-Inertial Frame and the relative position of orbit to the quasi-inertial frame}\label{fig:orbitalpaths}
\end{figure}

\textit{Simulation Mode I} employs an LPKE for the propagation of the motion of the space manipulator defined in Table \ref{tab:inertia}. This mode benefits from more moderate orbital velocities and distances (due to frame choice), which improves numerical conditioning. The twist vectors $\xi_i$ in Table \ref{Twists} are extracted from the joint axes of rotation and linear positions of joints in Figure \ref{sim:frameassignment}. The space manipulator is placed on the exact same orbital trajectory as all other modes and uses identical joint configurations and time-steps to the benchmark, 
\begin{table}[htbp]
\caption{Manipulator twists}
    \centering
    \resizebox{!}{1.6cm}{\begin{tabular}{cccccccc}
     $\xi_1$& $\xi_2$ & $\xi_3$ & $\xi_4$ & $\xi_5$ & $\xi_6$ & $\xi_7$ &{} \\\cmidrule{1-7}\\
    $\left[\begin{matrix}0\\0\\0\\0\\0\\1\end{matrix}  \right]$ & $\left[\begin{matrix}0.126\\0\\0\\0\\1\\0\end{matrix} \right]$ & $\left[\begin{matrix}0\\0\\0\\0\\0\\1\end{matrix} \right]$ & $\left[\begin{matrix}0\\-0.526\\0\\0\\-1\\0\end{matrix} \right]$&$\left[\begin{matrix}0\\0\\0\\0\\0\\1\end{matrix} \right]$ &  $\left[\begin{matrix}0\\0.926\\0\\0\\1\\0\end{matrix}  \right]$& $\left[\begin{matrix}0\\0\\0\\0\\0\\1\end{matrix}  \right]$
    \end{tabular}}\label{Twists}\vspace{0.5cm}
\end{table}

\textit{Simulation Mode II} also employs the LPKE formulation, but the orbital motion is expressed relative to the perifocal frame.  The Keplerian trajectory is computed in closed form using perifocal-frame dynamics, leading to significantly larger position and velocity magnitudes (in the order of hundreds to thousands of kilometers and velocities of several km/s). This greater disparity between orbital and manipulator dynamics introduces numerical stiffness into the integration process.

\textit{Simulation Mode III} employs the Lagrange–Poincaré Equations (LPE), where the orbital motion is not expressed through closed-form Keplerian solutions but is instead directly integrated using Newton's fundamental laws of gravitational dynamics. The spacecraft’s orbital state is dynamically coupled with the manipulator motion, resulting in a fully integrated model that captures mutual interactions through direct numerical propagation of gravitational acceleration. This includes disturbance terms $F_\orbit$ and $\hat{N}_\orbit$ from Theorem \ref{noninertialtheorem:EoM}, represented as external inputs. This mode adds dynamic complexity on top of Mode II’s ill-conditioning.

%\subsection{Simulation Scenario}
%For this simulation, the space manipulator initiates its motion from the initial configuration and initial joint velocities given by $$q_\m=\begin{bmatrix}0.2&1.5&0.3&1.8&0.5&0.3&0.6\end{bmatrix}^T$$ $$\dot{q}_\m=\begin{bmatrix}0&1&0&1&0&0&0\end{bmatrix}^T.$$

%The orbital frame's pose and velocity are kinematically propagated in time (can be analytically found from Kepler's equations \citep{curtis2013orbital}). The two orbital trajectories are identical between these simulated orbital space manipulators ($\rho^I_{\orbit_1}(t)=\rho^I_{\orbit_2}(t)$, $R^I_{\orbit_1}(t)=R^I_{\orbit_2}(t)$, $\theta_{\orbit_1}(t)=\theta_{\orbit_2}(t)$). 

\subsection{Simulation Results}
Figures \ref{fig:orbitalverification1} and \ref{fig:orbitalverification}  illustrate joint trajectory comparisons between the benchmark and LPKE-based simulations. In each subplot, the solid lines represent the benchmark joint trajectories, while the dashed lines show corresponding joint trajectories simulated using \textit{Modes I}, \textit{II}, and \textit{III}, respectively. Each mode simulates the same motion using its internal dynamics and with an integration step size equal to the corresponding benchmark. \textit{Mode I} exhibits near-exact agreement with \textit{Benchmark}, with maximum absolute joint trajectory deviation less than $10^{-11}$ radians throughout the simulation time scope, consistent with machine precision for double-precision integration obtained using a fixed-step Runge–Kutta (RK) integrator. This result substantiates the structural and numerical fidelity of the LPKE formulation. Thus, we establish the dynamic equivalency of the LPKE in Theorem \ref{noninertialtheorem:EoM} and Benchmark Simscape model. In contrast, \textit{Mode II} introduces propagation errors due to the large disparity in distances and velocities between orbital and internal manipulator motion inherent in the perifocal frame. These disparities lead to ill-conditioned dynamics matrices, resulting in numerical stiffness and degraded performance. This stiffness introduces integration difficulties due to rapidly varying dynamics embedded in slow orbital states, potentially degrading solver accuracy. \textit{Mode III} suffers from both the same disparity in scales and the added burden of tightly coupled manipulator-orbit integration, amplifying stiffness and solver instability. Although it directly integrates the full coupled orbital-manipulator dynamics, the combined complexity and scale mismatch further degrade accuracy and simulation speed. \textit{Modes II} and \textit{III} experience growing drift, reinforcing the long-term stability and structure-preserving benefits of the full LPKE approach in \textit{Mode I}.  Figure \ref{fig:orbitalverification} confirms that the LPKE model not only outperforms \textit{Modes II} and \textit{III} in accuracy and stability, but also faithfully reproduces the dynamics captured by the high-fidelity Simscape benchmark, while offering a closed-form and computationally efficient alternative suitable for model-based control and real-time motion planning. 
\begin{figure*}[htbp]
    \centering
    \begin{minipage}{0.68\columnwidth}\includegraphics[width=0.99\textwidth]{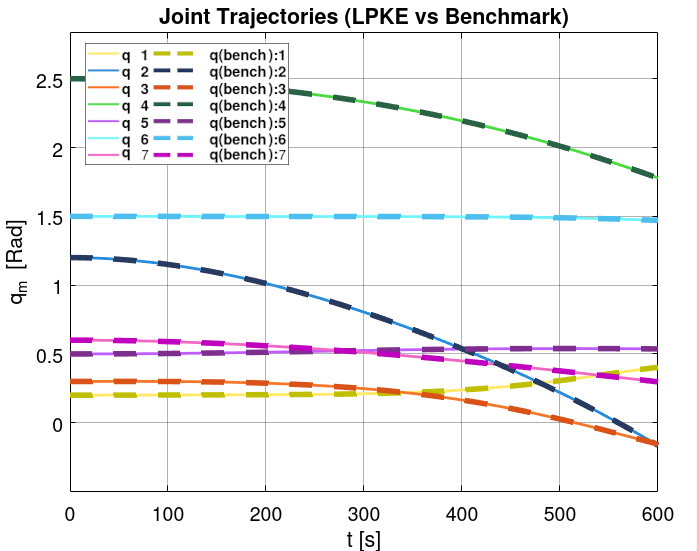}
    \end{minipage}
    \begin{minipage}{0.67\columnwidth}
    \includegraphics[width=0.99\textwidth]{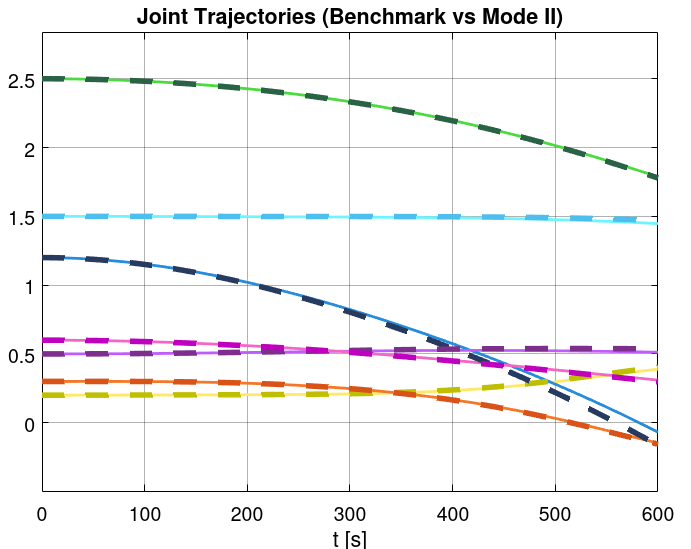}
    \end{minipage}
    \begin{minipage}{0.68\columnwidth}
    \includegraphics[width=0.99\textwidth]{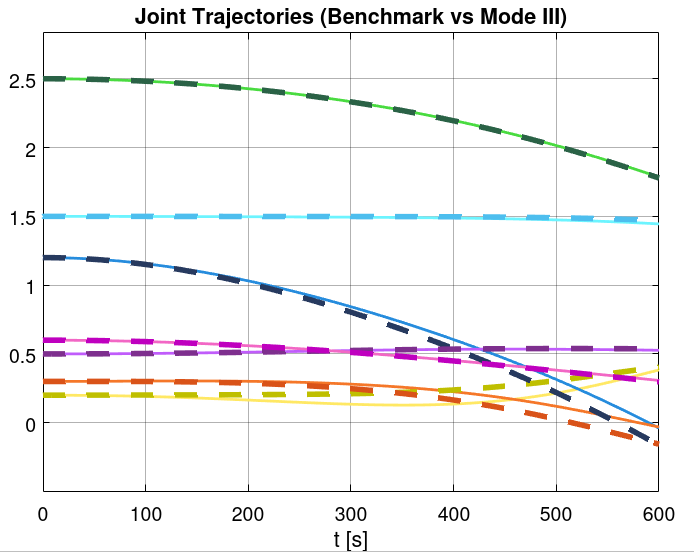}
    \end{minipage}
    \caption{Joint trajectory comparison of the benchmark Simscape simulation and: \textbf{Mode I} analytical LPKE propagation (left), \textbf{Mode II} LPKE formulation relative to the perifocal frame (middle), and \textbf{Mode III} direct integration of orbital differential equations into LPE (Right)}
    \label{fig:orbitalverification1}
\end{figure*}
\begin{figure*}[htbp]
    \centering
    \begin{minipage}{0.68\columnwidth}\includegraphics[width=0.99\textwidth]{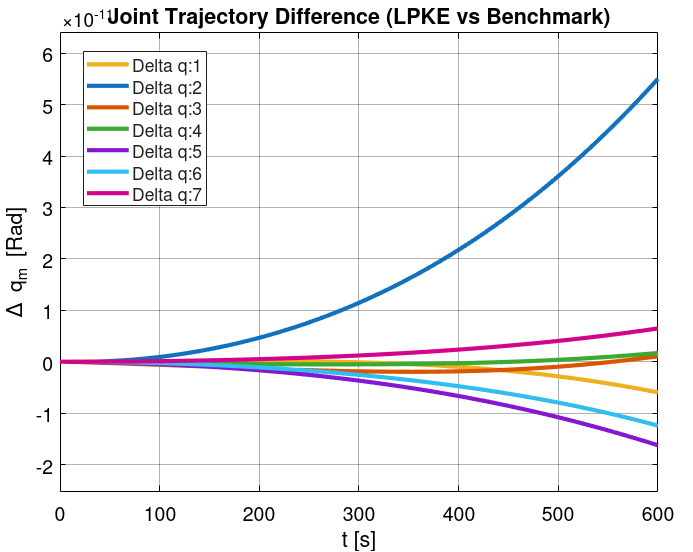}
    \end{minipage}
    \begin{minipage}{0.68\columnwidth}
    \includegraphics[width=0.99\textwidth]{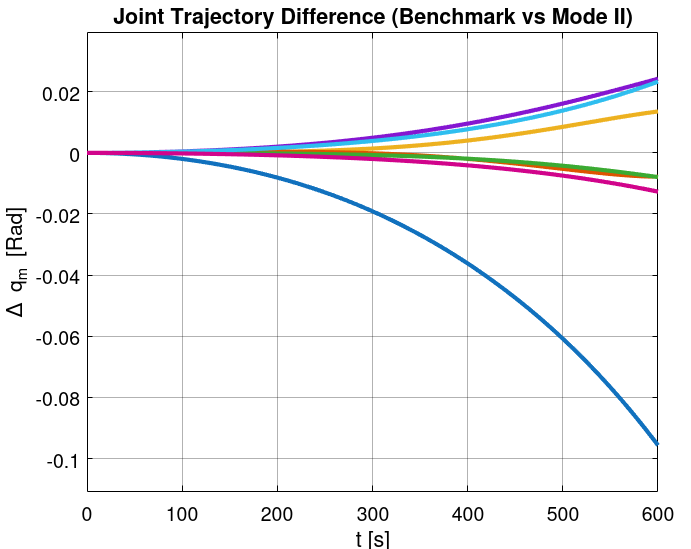}
    \end{minipage}
    \begin{minipage}{0.68\columnwidth}
    \includegraphics[width=0.99\textwidth]{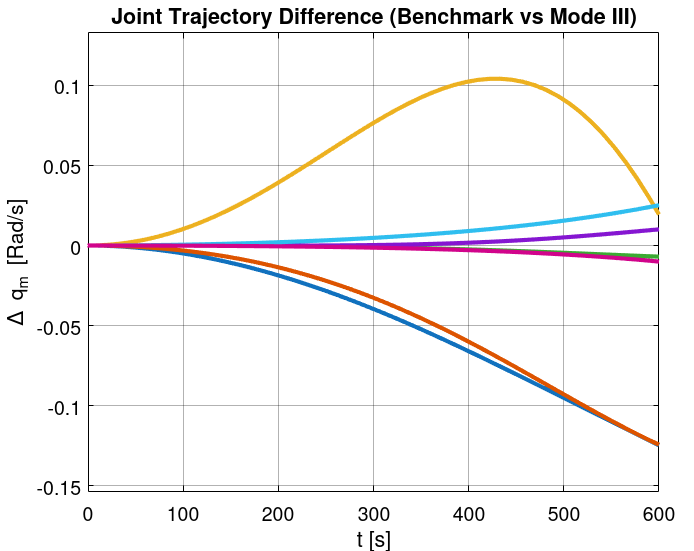}
    \end{minipage}
    \caption{Joint trajectory error comparison of the benchmark (Simscape) simulation and: \textbf{Mode I} analytical LPKE propagation (left), \textbf{Mode II} LPKE formulation relative to the perifocal frame (middle), and \textbf{Mode III} direct integration of orbital differential equations into LPE (Right)}
    \label{fig:orbitalverification}
\end{figure*}
\begin{figure*}[htbp]
    \centering
    \begin{minipage}{0.5\columnwidth}\includegraphics[width=0.99\textwidth]{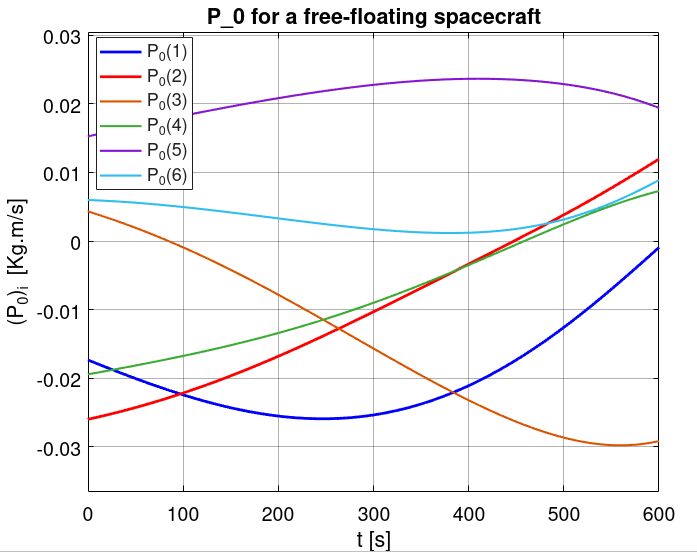}
    \end{minipage}
    \begin{minipage}{0.5\columnwidth}
    \includegraphics[width=0.99\textwidth]{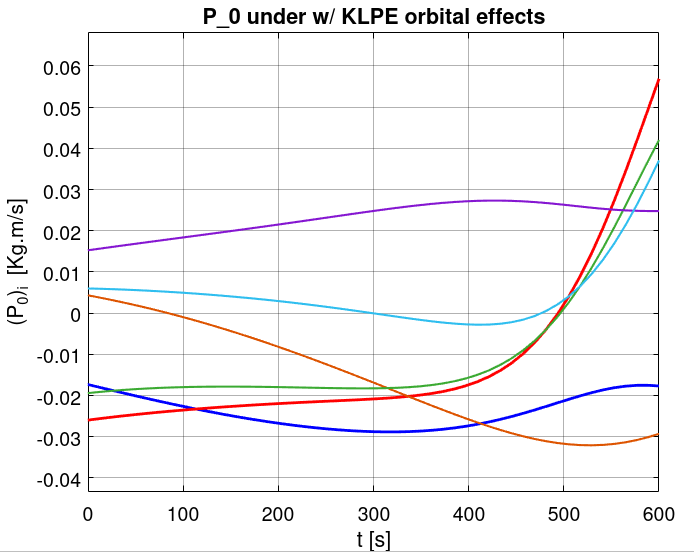}
    \end{minipage}
    \begin{minipage}{0.5\columnwidth}
    \includegraphics[width=0.99\textwidth]{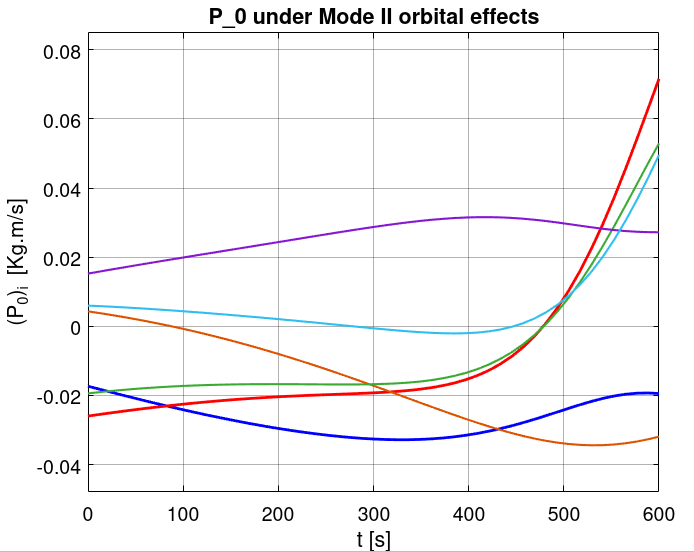}
    \end{minipage}
    \begin{minipage}{0.5\columnwidth}
    \includegraphics[width=0.99\textwidth]{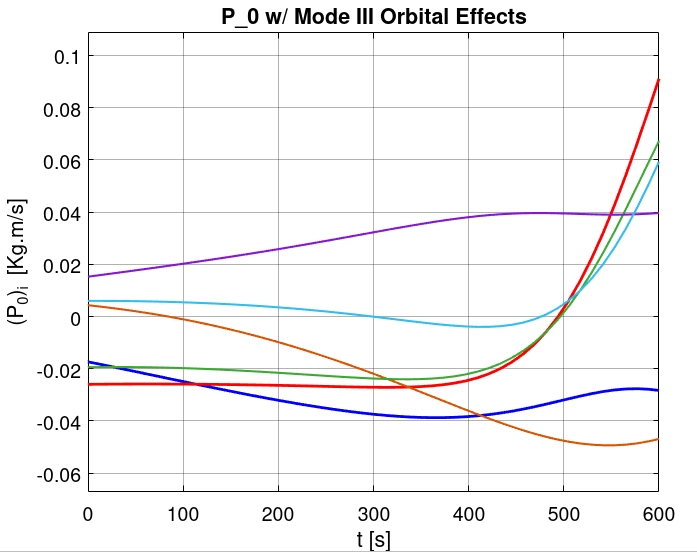}
    \end{minipage}
    \caption{Momentum $P_0$ evolution comparison of a free-floating space manipulator under no orbital disturbances (first plot on the left) and: \textbf{Mode I} analytical LPKE propagation (second plot), \textbf{Mode II} LPKE formulation relative to the perifocal frame (thirdplot), and \textbf{Mode III} direct integration of orbital differential equations into LPE (fourth plot)}
    \label{fig:momentumverification}
\end{figure*}

\subsubsection{Computational Efficiency}

Table~\ref{tab:comparison} summarizes the average computational cost of the three numerical approaches over 200 simulation runs, each with a 5-minute simulation horizon. All models are implemented in MATLAB using a fixed-step fourth-order RK method, ensuring consistent solver settings across all comparisons.
Among all modes, the LPKE formulation (\textit{Mode I}) is the most computationally efficient, with an average runtime of approximately 6.1 seconds per simulation. This performance stems from its closed-form integration of orbital effects using analytical Keplerian solutions, which avoids the stiffness and coupling challenges present in other models.
By contrast, \textit{Mode II} incurs nearly double the runtime (11.32 seconds on average). This increase is attributed to the need for additional evaluations of orbital coupling terms and their propagation within the inertial Lagrangian framework. While this approach remains consistent with a decoupled dynamic modeling philosophy, it introduces moderate overhead due to the exogenous updates.
\textit{Mode III}, which directly integrates Newtonian orbital dynamics into the manipulator's LPE, is the most computationally demanding of the three, with an average runtime of 11.59 seconds. This formulation requires simultaneous integration of the manipulator and orbital state dynamics, leading to stiffness and ill-conditioning in the combined system matrices. Consequently, this mode places a higher burden on numerical solvers and may be less suitable for real-time deployment.

\begin{remark}
Figure \ref{fig:momentumverification} illustrates the evolution of the linear and angular momentum vector $P_0$ across the three simulation modes. For the free--floating case, the apparent variation of $P_0$ is simply a rotated observation of the conserved total momentum $\mu$ when viewed in the spacecraft frame, consistent with theoretical expectations of momentum conservation. By contrast, in the disturbed orbital settings (Modes I, II and III), $P_0$ evolves nontrivially under the influence of symmetry--breaking orbital effects. Among these, the LPKE formulation most accurately follows the benchmark Simscape simulation of momentum evolution. This agreement confirms that the LPKE captures both the conservation structure of the free--floating regime and the disturbance--driven momentum exchange induced by orbital forcing.  
\end{remark}

Overall, the LPKE model achieves a favorable balance between fidelity and efficiency, providing a closed-form and scalable solution suitable for onboard applications. The dominant contributor to simulation slowdown is the magnitude difference between orbital and manipulator velocities. This explains why \textit{Mode I} is significantly faster than both \textit{Mode II} and \textit{Mode III}, whereas \textit{Mode II} and \textit{III} have comparable average runtimes despite their formulation differences.
\begin{table}[htbp]
    \centering
    \caption{Comparison of Numerical Efficiency of Models Over 200 Runs}
    \begin{tabular}{|lccc|}
        \hline
        {} & \textbf{Mode I} & \textbf{Mode II} & \textbf{Mode III}\\
        \hline
        Ave. Simulation Time (s) & 6.0996 & 11.3228 & 11.5892\\
        %Avg. Time per Step (ms) & XX.XX & YY.YY& \\
        %Number of Steps & XXXX & YYYY& \\
        \hline
    \end{tabular}
    \label{tab:comparison}
\end{table}

\begin{remark}
We further evaluated LPKE by repeating the simulations over 100 orbital eccentricities $e\in[0,0.2]$  and 100 altitudes spanning LEO to GEO, thereby inducing a broad range of orbit–attitude coupling disturbances.  We also varied the inclination over $i\in[0^\circ,30^\circ]$ and, as intuitively anticipated from the specialization by LPKE to planar orbital motion, observed no material change in behavior. Across all cases, LPKE achieved accuracy comparable to the benchmark while retaining clear computational advantages: Mode~II and Mode~III exhibited average excess runtimes of $79\%$ and $87\%$ relative to LPKE, respectively. Representative joint–trajectory comparisons for a LEO circular case and a highly elliptical LEO case ($e=0.2$) (see Figs.~\ref{fig:LEO} and \ref{fig:LEOe}, respectively) further illustrate the consistency of LPKE’s performance across various orbital conditions.
\begin{figure}
    \centering
    \begin{subfigure}[b]{0.49\columnwidth}
        \centering
        \includegraphics[width=1.07\columnwidth]{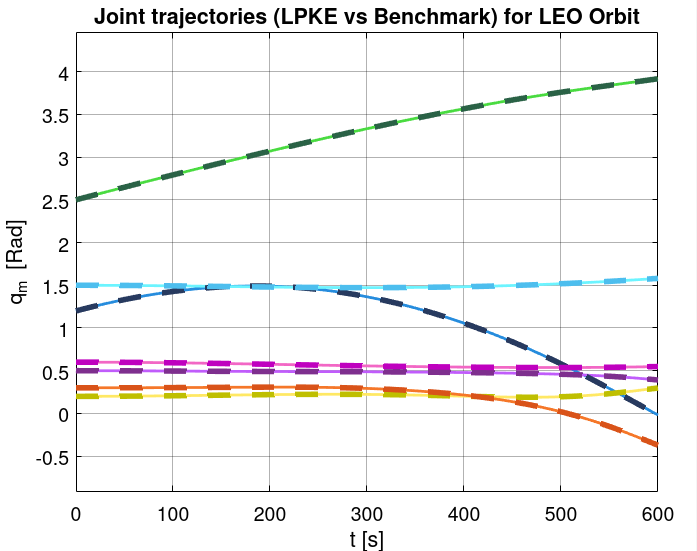}
        
    \end{subfigure}
    \hfill
    \begin{subfigure}[b]{0.49\columnwidth}
        \centering
        \includegraphics[width=1.07\columnwidth]{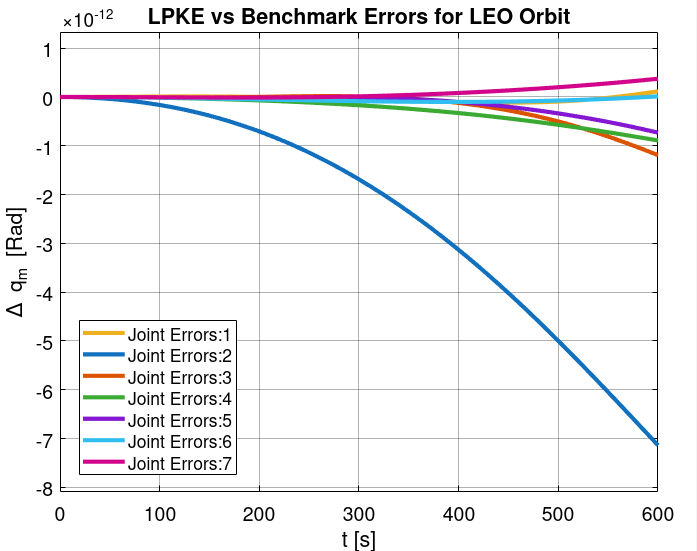}
    \end{subfigure}
    \caption{Joint trajectory evolution (Left) and Joint trajectory errors (Right)  comparison of LPKE vs Benchmark for space manipulator on LEO orbit}
        \label{fig:LEO}
\end{figure}

\begin{figure}
    \centering
    \begin{subfigure}[b]{0.49\columnwidth}
        \centering
        \includegraphics[width=1.07\columnwidth]{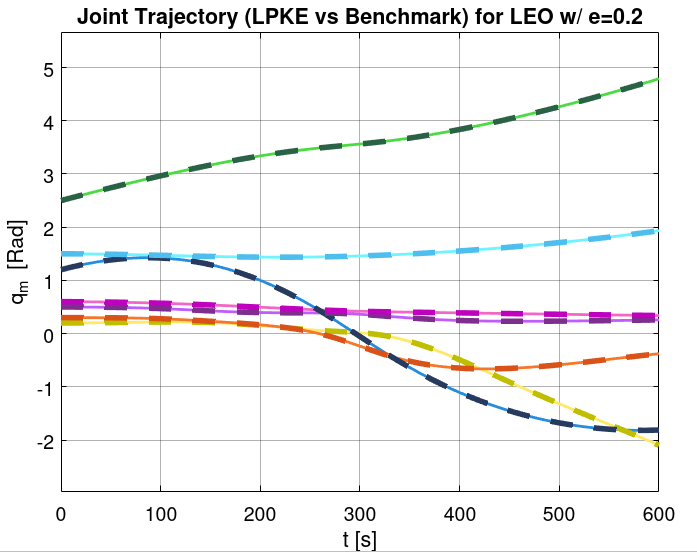}
        
    \end{subfigure}
    \hfill
    \begin{subfigure}[b]{0.49\columnwidth}
        \centering
        \includegraphics[width=1.07\columnwidth]{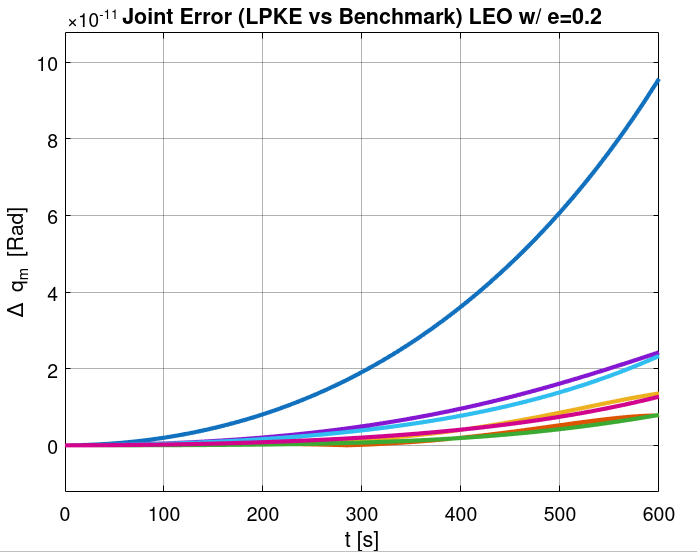}
    \end{subfigure}
    \caption{Joint trajectory evolution (Left) and Joint trajectory errors (Right)  comparison of LPKE vs Benchmark for space manipulator on LEO orbit with high eccentricity (e=0.2)}
        \label{fig:LEOe}
\end{figure}

\end{remark}

\subsubsection{Discussion}

The LPKE formulation (\textit{Mode I} ) offers both physical and computational advantages: Physically, it captures orbital coupling, Coriolis effects, and symmetry-breaking disturbances in closed form. Numerically, it avoids ill-conditioning typical of inertial-frame formulations caused by the large disparity in magnitudes between orbital and manipulator dynamics. Algorithmically, it supports structure-preserving integration and eliminates redundant coordinates in propagation, making it well-suited for embedded control applications.
In contrast, \textit{Modes II} and \textit{III}  fail to maintain accuracy over extended periods due to compounded numerical integration errors and poor conditioning of the full-system dynamic matrices.
%\subsubsection{Real-time Implementation}

The demonstrated behavior of the LPKE under different orbital frames carries direct implications for real-time simulation, model-based control, and planning. Simulations using LPKE in the quasi-inertial frame (\textit{Mode I}), are more computationally efficient and stable, suggesting their suitability for real-time applications, including onboard planning and predictive control. Its ability to integrate accurately at reduced computational cost makes it a promising candidate for real-time controllers and planners.
By contrast, the perifocal frame's use in \textit{Mode II} and the complex coupled integration in \textit{Mode III} lead to increased stiffness and computational overhead, reducing the feasibility of these approaches for real-time applications, especially in the presence of additional perturbations or real-world actuator and sensor constraints. In particular, the need to resolve large orbital dynamics alongside slow manipulator dynamics imposes a substantial computational burden. For applications requiring fast, onboard computation, LPKE in the quasi-inertial frame (\textit{Mode I}) offers a structurally sound and numerically practical solution.
%-----------------------------------------------------------------
%\FloatBarrier
\section{Conclusion} \label{conclusion}
In this work, we extended the LPE for spacecraft-manipulator systems from an inertial frame to an orbital frame, in order to capture the largest effects that orbital motion has on the safety of control of space manipulators. The LPKE provided equations that remain free of parametrization singularities, enabling the development of model-based control laws for space manipulators that avoid such singularities. By incorporating symmetry-breaking forces that stem from orbital motion into the structure of the LPE developed in \citep{moghaddam2022}, LPKE integrated the effects of orbital motion into the established LPE while maintaining the closed-form matrix EoM developed in \citep{moghaddam2022}, and similar structural matrices to that model. %By incorporating the coupling effects and forces introduced by orbital motion, the model offers clear physical interpretations and explicit closed-form equations.  
The proposed formalism effectively decoupled the orbital dynamics, locked spacecraft-manipulator dynamics, and equivalent fixed-base manipulator internal dynamics. Thus, it retained the model's utility for hardware-in-the-loop (HIL) simulations and singularity-free control development, avoiding the need for base spacecraft pose parametrization and preserving geometric properties as observed from the spacecraft. As future work, we will integrate the LPKE framework with full hardware-in-the-loop and model-based control stacks.  Orbital motion solely introduced disturbing forces to the LPE, leading to a CoM  oM shift up to $4m/min$ for the sample GEO spacecraft, a significant disturbance in microgravity operations. The decoupling of orbital, spacecraft, and manipulator dynamics, along with the  independence of the shape of EoM from orbit and spacecraft's pose, is integral to the development of advanced safe geometric control schemes and simulation frameworks.

%%%%%%%%%%%%%%%%%%%%%%%%%%%%%%%%%%%%%%%END
%%%%%%%%%%%%%%%%%%%%%%%%%%%%%%%%%%%%%%%END
%%%%%%%%%%%%%%%%%%%%%%%%%%%%%%%%%%%%%%%END
%%%%%%%%%%%%%%%%%%%%%%%%%%%%%%%%%%%%%%%END

\printcredits

\bibliographystyle{cas-model2-names}
\bibliography{cas-dc-template}{}

\begin{thebibliography}{88}
\expandafter\ifx\csname natexlab\endcsname\relax\def\natexlab#1{#1}\fi
\providecommand{\url}[1]{\texttt{#1}}
\providecommand{\href}[2]{#2}
\providecommand{\path}[1]{#1}
\providecommand{\DOIprefix}{doi:}
\providecommand{\ArXivprefix}{arXiv:}
\providecommand{\URLprefix}{URL: }
\providecommand{\Pubmedprefix}{pmid:}
\providecommand{\doi}[1]{\href{http://dx.doi.org/#1}{\path{#1}}}
\providecommand{\Pubmed}[1]{\href{pmid:#1}{\path{#1}}}
\providecommand{\bibinfo}[2]{#2}
\ifx\xfnm\relax \def\xfnm[#1]{\unskip,\space#1}\fi
%Type = Article
\bibitem[{Aikenhead et~al.(1983)Aikenhead, Daniell and
  Davis}]{aikenhead1983canadarm}
\bibinfo{author}{Aikenhead, B.A.}, \bibinfo{author}{Daniell, R.G.},
  \bibinfo{author}{Davis, F.M.}, \bibinfo{year}{1983}.
\newblock \bibinfo{title}{Canadarm and the space shuttle}.
\newblock \bibinfo{journal}{Journal of Vacuum Science \& Technology A: Vacuum,
  Surfaces, and Films} \bibinfo{volume}{1}, \bibinfo{pages}{126--132}.
%Type = Article
\bibitem[{Andiappane et~al.(2019)Andiappane, Durand and
  Dubanchet}]{andiappane2019mission}
\bibinfo{author}{Andiappane, S.}, \bibinfo{author}{Durand, G.},
  \bibinfo{author}{Dubanchet, V.}, \bibinfo{year}{2019}.
\newblock \bibinfo{title}{Mission and system design for {EROSS} project: The
  {E}uropean robotic orbital support services}.
\newblock \bibinfo{journal}{IAC 2019} .
%Type = Article
\bibitem[{Antonelli et~al.(2004)Antonelli, Caccavale and
  Chiaverini}]{antonelli2004adaptive}
\bibinfo{author}{Antonelli, G.}, \bibinfo{author}{Caccavale, F.},
  \bibinfo{author}{Chiaverini, S.}, \bibinfo{year}{2004}.
\newblock \bibinfo{title}{Adaptive tracking control of underwater
  vehicle-manipulator systems based on the virtual decomposition approach}.
\newblock \bibinfo{journal}{IEEE Transactions on Robotics and Automation}
  \bibinfo{volume}{20}, \bibinfo{pages}{594--602}.
%Type = Article
\bibitem[{Aspragathos and Dimitros(1998)}]{aspragathos1998comparative}
\bibinfo{author}{Aspragathos, N.A.}, \bibinfo{author}{Dimitros, J.K.},
  \bibinfo{year}{1998}.
\newblock \bibinfo{title}{A comparative study of three methods for robot
  kinematics}.
\newblock \bibinfo{journal}{IEEE Transactions on Systems, Man, and Cybernetics,
  Part B (Cybernetics)} \bibinfo{volume}{28}, \bibinfo{pages}{135--145}.
%Type = Book
\bibitem[{Ball(1998)}]{ball1998treatise}
\bibinfo{author}{Ball, R.S.}, \bibinfo{year}{1998}.
\newblock \bibinfo{title}{A Treatise on the Theory of Screws}.
\newblock \bibinfo{publisher}{Cambridge University Press}.
%Type = Article
\bibitem[{Benninghoff et~al.(2014)Benninghoff, Boge and
  Rems}]{benninghoff2014autonomous}
\bibinfo{author}{Benninghoff, H.}, \bibinfo{author}{Boge, T.},
  \bibinfo{author}{Rems, F.}, \bibinfo{year}{2014}.
\newblock \bibinfo{title}{Autonomous navigation for on-orbit servicing}.
\newblock \bibinfo{journal}{KI-K{\"u}nstliche Intelligenz}
  \bibinfo{volume}{28}, \bibinfo{pages}{77--83}.
%Type = Inproceedings
\bibitem[{Bloch and Croach(1994)}]{bloch1994reduction}
\bibinfo{author}{Bloch, A.}, \bibinfo{author}{Croach, P.},
  \bibinfo{year}{1994}.
\newblock \bibinfo{title}{Reduction of euler lagrange problems for constrained
  variational problems and relation with optimal control problems}, in:
  \bibinfo{booktitle}{Proceedings of 33rd IEEE Conference on Decision and
  Control}, pp. \bibinfo{pages}{2584--2590}.
%Type = Incollection
\bibitem[{Bloch(2003)}]{blochbook}
\bibinfo{author}{Bloch, A.M.}, \bibinfo{year}{2003}.
\newblock \bibinfo{title}{Nonholonomic mechanics}, in:
  \bibinfo{booktitle}{Nonholonomic mechanics and control}.
  \bibinfo{publisher}{Springer}, pp. \bibinfo{pages}{207--276}.
%Type = Article
\bibitem[{Bloch et~al.(1996)Bloch, Krishnaprasad, Marsden and
  Murray}]{bloch1996nonholonomic}
\bibinfo{author}{Bloch, A.M.}, \bibinfo{author}{Krishnaprasad, P.},
  \bibinfo{author}{Marsden, J.E.}, \bibinfo{author}{Murray, R.M.},
  \bibinfo{year}{1996}.
\newblock \bibinfo{title}{Nonholonomic mechanical systems with symmetry}.
\newblock \bibinfo{journal}{Archive for Rational Mechanics and Analysis}
  \bibinfo{volume}{136}, \bibinfo{pages}{21--99}.
%Type = Article
\bibitem[{Bloch et~al.(2001)Bloch, Leonard and Marsden}]{bloch2001controlled}
\bibinfo{author}{Bloch, A.M.}, \bibinfo{author}{Leonard, N.E.},
  \bibinfo{author}{Marsden, J.E.}, \bibinfo{year}{2001}.
\newblock \bibinfo{title}{Controlled lagrangians and the stabilization of
  euler--poincar{\'e} mechanical systems}.
\newblock \bibinfo{journal}{International Journal of Robust and Nonlinear
  Control: IFAC-Affiliated Journal} \bibinfo{volume}{11},
  \bibinfo{pages}{191--214}.
%Type = Article
\bibitem[{Bloch et~al.(1992)Bloch, Reyhanoglu and
  McClamroch}]{bloch1992control}
\bibinfo{author}{Bloch, A.M.}, \bibinfo{author}{Reyhanoglu, M.},
  \bibinfo{author}{McClamroch, N.H.}, \bibinfo{year}{1992}.
\newblock \bibinfo{title}{Control and stabilization of nonholonomic dynamic
  systems}.
\newblock \bibinfo{journal}{IEEE Transactions on Automatic control}
  \bibinfo{volume}{37}, \bibinfo{pages}{1746--1757}.
%Type = Article
\bibitem[{Borri et~al.(2000)Borri, Trainelli and
  Bottasso}]{borri2000representations}
\bibinfo{author}{Borri, M.}, \bibinfo{author}{Trainelli, L.},
  \bibinfo{author}{Bottasso, C.L.}, \bibinfo{year}{2000}.
\newblock \bibinfo{title}{On representations and parameterizations of motion}.
\newblock \bibinfo{journal}{Multibody system dynamics} \bibinfo{volume}{4},
  \bibinfo{pages}{129--193}.
%Type = Inproceedings
\bibitem[{Brockett(1984)}]{brockett1984robotic}
\bibinfo{author}{Brockett, R.W.}, \bibinfo{year}{1984}.
\newblock \bibinfo{title}{Robotic manipulators and the product of exponentials
  formula}, in: \bibinfo{booktitle}{Mathematical theory of networks and
  systems}, \bibinfo{organization}{Springer}. pp. \bibinfo{pages}{120--129}.
%Type = Article
\bibitem[{Carr(2022)}]{carr2022coupled}
\bibinfo{author}{Carr, C.M.}, \bibinfo{year}{2022}.
\newblock \bibinfo{title}{Coupled orbit-attitude dynamics and control of a
  cubesat equipped with a robotic manipulator} .
%Type = Incollection
\bibitem[{Cendra et~al.(2001)Cendra, Marsden and Ratiu}]{cendra2001geometric}
\bibinfo{author}{Cendra, H.}, \bibinfo{author}{Marsden, J.E.},
  \bibinfo{author}{Ratiu, T.S.}, \bibinfo{year}{2001}.
\newblock \bibinfo{title}{Geometric mechanics, {L}agrangian reduction, and
  nonholonomic systems}, in: \bibinfo{booktitle}{Mathematics unlimited—2001
  and beyond}. \bibinfo{publisher}{Springer}, pp. \bibinfo{pages}{221--273}.
%Type = Article
\bibitem[{Chhabra and Emami(2014a)}]{chhabra2014generalized}
\bibinfo{author}{Chhabra, R.}, \bibinfo{author}{Emami, M.R.},
  \bibinfo{year}{2014}a.
\newblock \bibinfo{title}{A generalized exponential formula for forward and
  differential kinematics of open-chain multi-body systems}.
\newblock \bibinfo{journal}{Mechanism and Machine Theory} \bibinfo{volume}{73},
  \bibinfo{pages}{61--75}.
%Type = Article
\bibitem[{Chhabra and Emami(2014b)}]{chhabra2014nonholonomic}
\bibinfo{author}{Chhabra, R.}, \bibinfo{author}{Emami, M.R.},
  \bibinfo{year}{2014}b.
\newblock \bibinfo{title}{Nonholonomic dynamical reduction of open-chain
  multi-body systems: a geometric approach}.
\newblock \bibinfo{journal}{Mechanism and Machine Theory} \bibinfo{volume}{82},
  \bibinfo{pages}{231--255}.
%Type = Article
\bibitem[{Chhabra and Emami(2015)}]{chhabra2015symplectic}
\bibinfo{author}{Chhabra, R.}, \bibinfo{author}{Emami, M.R.},
  \bibinfo{year}{2015}.
\newblock \bibinfo{title}{Symplectic reduction of holonomic open-chain
  multi-body systems with constant momentum}.
\newblock \bibinfo{journal}{Journal of Geometry and Physics}
  \bibinfo{volume}{89}, \bibinfo{pages}{82--110}.
%Type = Article
\bibitem[{Clohessy and Wiltshire(1960)}]{clohessy1960terminal}
\bibinfo{author}{Clohessy, W.}, \bibinfo{author}{Wiltshire, R.S.},
  \bibinfo{year}{1960}.
\newblock \bibinfo{title}{{T}erminal guidance for satellite rendezvous}.
\newblock \bibinfo{journal}{Journal of the Aerospace Sciences}
  \bibinfo{volume}{27}, \bibinfo{pages}{653--658}.
%Type = Book
\bibitem[{Curtis()}]{curtis2013orbital}
\bibinfo{author}{Curtis, H.}, .
\newblock \bibinfo{title}{Orbital mechanics for engineering students}.
\newblock \bibinfo{publisher}{Butterworth-Heinemann (2013)}.
%Type = Article
\bibitem[{Da~Fonseca et~al.(2017)Da~Fonseca, Goes, Seito, da~Silva~Duarte and
  de~Oliveira}]{da2017attitude}
\bibinfo{author}{Da~Fonseca, I.M.}, \bibinfo{author}{Goes, L.C.},
  \bibinfo{author}{Seito, N.}, \bibinfo{author}{da~Silva~Duarte, M.K.},
  \bibinfo{author}{de~Oliveira, {\'E}.J.}, \bibinfo{year}{2017}.
\newblock \bibinfo{title}{Attitude dynamics and control of a spacecraft like a
  robotic manipulator when implementing on-orbit servicing}.
\newblock \bibinfo{journal}{Acta Astronautica} \bibinfo{volume}{137},
  \bibinfo{pages}{490--497}.
%Type = Article
\bibitem[{De~Stefano et~al.(2021)De~Stefano, Mishra, Giordano, Lampariello and
  Ott}]{de2021relative}
\bibinfo{author}{De~Stefano, M.}, \bibinfo{author}{Mishra, H.},
  \bibinfo{author}{Giordano, A.M.}, \bibinfo{author}{Lampariello, R.},
  \bibinfo{author}{Ott, C.}, \bibinfo{year}{2021}.
\newblock \bibinfo{title}{A relative dynamics formulation for
  hardware-in-the-loop simulation of on-orbit robotic missions}.
\newblock \bibinfo{journal}{IEEE Robotics and Automation Letters}
  \bibinfo{volume}{6}, \bibinfo{pages}{3569--3576}.
%Type = Inproceedings
\bibitem[{Dubanchet et~al.(2020)Dubanchet, Romero, Gregertsen, Austad, Gancet,
  Natusiewicz, Vinals, Guerra, Rekleitis, Paraskevas
  et~al.}]{dubanchet2020eross}
\bibinfo{author}{Dubanchet, V.}, \bibinfo{author}{Romero, J.B.},
  \bibinfo{author}{Gregertsen, K.N.}, \bibinfo{author}{Austad, H.},
  \bibinfo{author}{Gancet, J.}, \bibinfo{author}{Natusiewicz, K.},
  \bibinfo{author}{Vinals, J.}, \bibinfo{author}{Guerra, G.},
  \bibinfo{author}{Rekleitis, G.}, \bibinfo{author}{Paraskevas, I.}, et~al.,
  \bibinfo{year}{2020}.
\newblock \bibinfo{title}{{EROSS} project--{E}uropean autonomous robotic
  vehicle for on-orbit servicing}, in: \bibinfo{booktitle}{International
  Symposium on Artificial Intelligence, Robotics and Automation in
  Space,(i-SAIRAS’20). USA: Pasadena, California}.
%Type = Inproceedings
\bibitem[{Dubancheta et~al.(2021)Dubancheta, Bejar~Romero, Paraskevas,
  Lopez~Negro, Cuffolo, Maye, Rodr{\'\i}guez~Reina, Romero~Manrique, Alonso,
  Torralbo~Dezainde et~al.}]{dubancheta2021eross}
\bibinfo{author}{Dubancheta, V.}, \bibinfo{author}{Bejar~Romero, J.},
  \bibinfo{author}{Paraskevas, I.}, \bibinfo{author}{Lopez~Negro, P.},
  \bibinfo{author}{Cuffolo, A.}, \bibinfo{author}{Maye, F.},
  \bibinfo{author}{Rodr{\'\i}guez~Reina, A.}, \bibinfo{author}{Romero~Manrique,
  P.}, \bibinfo{author}{Alonso, M.}, \bibinfo{author}{Torralbo~Dezainde, S.},
  et~al., \bibinfo{year}{2021}.
\newblock \bibinfo{title}{{EROSS} project--ground validation of an autonomous
  gnc architecture towards future {E}uropean servicing missions}, in:
  \bibinfo{booktitle}{Proceedings of the 72nd International Astronautical
  Congress (IAC)}, pp. \bibinfo{pages}{25--29}.
%Type = Article
\bibitem[{Dubowsky and Papadopoulos(1993)}]{dubowsky1993kinematics}
\bibinfo{author}{Dubowsky, S.}, \bibinfo{author}{Papadopoulos, E.},
  \bibinfo{year}{1993}.
\newblock \bibinfo{title}{The kinematics, dynamics, and control of free-flying
  and free-floating space robotic systems}.
\newblock \bibinfo{journal}{IEEE Transactions on robotics and automation}
  \bibinfo{volume}{9}, \bibinfo{pages}{531--543}.
%Type = Article
\bibitem[{Duindam and Stramigioli(2008)}]{duindam2008singularity}
\bibinfo{author}{Duindam, V.}, \bibinfo{author}{Stramigioli, S.},
  \bibinfo{year}{2008}.
\newblock \bibinfo{title}{Singularity-free dynamic equations of open-chain
  mechanisms with general holonomic and nonholonomic joints}.
\newblock \bibinfo{journal}{IEEE Transactions on Robotics}
  \bibinfo{volume}{24}, \bibinfo{pages}{517--526}.
%Type = Article
\bibitem[{Flores-Abad et~al.(2014)Flores-Abad, Ma, Pham and
  Ulrich}]{flores2014review}
\bibinfo{author}{Flores-Abad, A.}, \bibinfo{author}{Ma, O.},
  \bibinfo{author}{Pham, K.}, \bibinfo{author}{Ulrich, S.},
  \bibinfo{year}{2014}.
\newblock \bibinfo{title}{A review of space robotics technologies for on-orbit
  servicing}.
\newblock \bibinfo{journal}{Progress in aerospace sciences}
  \bibinfo{volume}{68}, \bibinfo{pages}{1--26}.
%Type = Inproceedings
\bibitem[{From et~al.(2011)From, Pettersen and
  Gravdahl}]{from2011singularityspace}
\bibinfo{author}{From, P.}, \bibinfo{author}{Pettersen, K.},
  \bibinfo{author}{Gravdahl, J.}, \bibinfo{year}{2011}.
\newblock \bibinfo{title}{A singularity free formulation of the dynamically
  equivalent manipulator mapping of space manipulators}, in:
  \bibinfo{booktitle}{AIAA SPACE Conference \& Exposition}, p.
  \bibinfo{pages}{7344}.
%Type = Article
\bibitem[{From(2012a)}]{from2012explicit}
\bibinfo{author}{From, P.J.}, \bibinfo{year}{2012}a.
\newblock \bibinfo{title}{An explicit formulation of singularity-free dynamic
  equations of mechanical systems in lagrangian form---part one: Single rigid
  bodies}.
\newblock \bibinfo{journal}{Modeling, Identification and Control}
  \bibinfo{volume}{33}, \bibinfo{pages}{45--60}.
%Type = Article
\bibitem[{From(2012b)}]{from2012explicit2}
\bibinfo{author}{From, P.J.}, \bibinfo{year}{2012}b.
\newblock \bibinfo{title}{An explicit formulation of singularity-free dynamic
  equations of mechanical systems in lagrangian form---part two: Multibody
  systems}.
\newblock \bibinfo{journal}{Modeling, Identification and Control}
  \bibinfo{volume}{33}, \bibinfo{pages}{61--68}.
%Type = Article
\bibitem[{From et~al.(2010a)From, Duindam, Pettersen, Gravdahl and
  Sastry}]{from2010singularity}
\bibinfo{author}{From, P.J.}, \bibinfo{author}{Duindam, V.},
  \bibinfo{author}{Pettersen, K.Y.}, \bibinfo{author}{Gravdahl, J.T.},
  \bibinfo{author}{Sastry, S.}, \bibinfo{year}{2010}a.
\newblock \bibinfo{title}{Singularity-free dynamic equations of
  vehicle--manipulator systems}.
\newblock \bibinfo{journal}{Simulation Modelling Practice and Theory}
  \bibinfo{volume}{18}, \bibinfo{pages}{712--731}.
%Type = Article
\bibitem[{From et~al.(2010b)From, Pettersen and
  Gravdahl}]{from2010singularityauv}
\bibinfo{author}{From, P.J.}, \bibinfo{author}{Pettersen, K.Y.},
  \bibinfo{author}{Gravdahl, J.T.}, \bibinfo{year}{2010}b.
\newblock \bibinfo{title}{Singularity-free dynamic equations of auv-manipulator
  systems}.
\newblock \bibinfo{journal}{IFAC Proceedings Volumes} \bibinfo{volume}{43},
  \bibinfo{pages}{31--36}.
%Type = Article
\bibitem[{Gibbs and Sachdev(2002)}]{gibbs2002canada}
\bibinfo{author}{Gibbs, G.}, \bibinfo{author}{Sachdev, S.},
  \bibinfo{year}{2002}.
\newblock \bibinfo{title}{Canada and the international space station program:
  overview and status}.
\newblock \bibinfo{journal}{Acta Astronautica} \bibinfo{volume}{51},
  \bibinfo{pages}{591--600}.
%Type = Article
\bibitem[{Gibson and Hunt(1990a)}]{gibson1990geometry}
\bibinfo{author}{Gibson, C.}, \bibinfo{author}{Hunt, K.},
  \bibinfo{year}{1990}a.
\newblock \bibinfo{title}{Geometry of screw systems—1: Screws: genesis and
  geometry}.
\newblock \bibinfo{journal}{Mechanism and Machine theory} \bibinfo{volume}{25},
  \bibinfo{pages}{1--10}.
%Type = Article
\bibitem[{Gibson and Hunt(1990b)}]{gibson1990geometry2}
\bibinfo{author}{Gibson, C.}, \bibinfo{author}{Hunt, K.},
  \bibinfo{year}{1990}b.
\newblock \bibinfo{title}{Geometry of screw systems—2: classification of
  screw systems}.
\newblock \bibinfo{journal}{Mechanism and Machine Theory} \bibinfo{volume}{25},
  \bibinfo{pages}{11--27}.
%Type = Article
\bibitem[{Hamel(1904)}]{hamel1904lagrange}
\bibinfo{author}{Hamel, G.}, \bibinfo{year}{1904}.
\newblock \bibinfo{title}{Die lagrange-eulerschen gleichungen der mechanik.
  zeitschr. f}.
\newblock \bibinfo{journal}{Math. und Phys} \bibinfo{volume}{50}.
%Type = Book
\bibitem[{Hamel(2013)}]{hamel2013theoretische}
\bibinfo{author}{Hamel, G.}, \bibinfo{year}{2013}.
\newblock \bibinfo{title}{Theoretische Mechanik: eine einheitliche
  Einf{\"u}hrung in die gesamte Mechanik}.
\newblock \bibinfo{publisher}{Springer-Verlag}.
%Type = Article
\bibitem[{Herv{\'e}(1999)}]{herve1999lie}
\bibinfo{author}{Herv{\'e}, J.M.}, \bibinfo{year}{1999}.
\newblock \bibinfo{title}{The lie group of rigid body displacements, a
  fundamental tool for mechanism design}.
\newblock \bibinfo{journal}{Mechanism and Machine Theory} \bibinfo{volume}{34},
  \bibinfo{pages}{719--730}.
%Type = Article
\bibitem[{Huang et~al.(2017)Huang, Tang, Li, Liang, Li and
  Pang}]{huang2017vehicle}
\bibinfo{author}{Huang, H.}, \bibinfo{author}{Tang, Q.}, \bibinfo{author}{Li,
  H.}, \bibinfo{author}{Liang, L.}, \bibinfo{author}{Li, W.},
  \bibinfo{author}{Pang, Y.}, \bibinfo{year}{2017}.
\newblock \bibinfo{title}{Vehicle-manipulator system dynamic modeling and
  control for underwater autonomous manipulation}.
\newblock \bibinfo{journal}{Multibody System Dynamics} \bibinfo{volume}{41},
  \bibinfo{pages}{125--147}.
%Type = Article
\bibitem[{Huang et~al.(2006)Huang, Xu and Liang}]{huang2006tracking}
\bibinfo{author}{Huang, P.}, \bibinfo{author}{Xu, Y.}, \bibinfo{author}{Liang,
  B.}, \bibinfo{year}{2006}.
\newblock \bibinfo{title}{Tracking trajectory planning of space manipulator for
  capturing operation}.
\newblock \bibinfo{journal}{International Journal of Advanced Robotic Systems}
  \bibinfo{volume}{3}, \bibinfo{pages}{31}.
%Type = Article
\bibitem[{Janousek et~al.(2021)}]{janousek2021eros}
\bibinfo{author}{Janousek, M.}, et~al., \bibinfo{year}{2021}.
\newblock \bibinfo{title}{The {E}uropean {EROSS} on-orbit servicing mission}.
\newblock \bibinfo{journal}{Acta Astronautica} \bibinfo{volume}{186},
  \bibinfo{pages}{123--133}.
%Type = Article
\bibitem[{Jiang et~al.(2017)Jiang, Wang and Xu}]{jiang2017integrated}
\bibinfo{author}{Jiang, L.}, \bibinfo{author}{Wang, Y.}, \bibinfo{author}{Xu,
  S.}, \bibinfo{year}{2017}.
\newblock \bibinfo{title}{Integrated 6-dof orbit-attitude dynamical modeling
  and control using geometric mechanics}.
\newblock \bibinfo{journal}{International Journal of Aerospace Engineering}
  \bibinfo{volume}{2017}, \bibinfo{pages}{4034328}.
%Type = Book
\bibitem[{Kaplan(1976)}]{kaplan1976modern}
\bibinfo{author}{Kaplan, M.H.}, \bibinfo{year}{1976}.
\newblock \bibinfo{title}{Modern spacecraft dynamics and control}.
\newblock \bibinfo{publisher}{Wiley}.
%Type = Article
\bibitem[{Kawano et~al.(2001)Kawano, Mokuno, Kasai and
  Suzuki}]{kawano2001result}
\bibinfo{author}{Kawano, I.}, \bibinfo{author}{Mokuno, M.},
  \bibinfo{author}{Kasai, T.}, \bibinfo{author}{Suzuki, T.},
  \bibinfo{year}{2001}.
\newblock \bibinfo{title}{Result of autonomous rendezvous docking experiment of
  engineering test satellite-vii}.
\newblock \bibinfo{journal}{Journal of Spacecraft and Rockets}
  \bibinfo{volume}{38}, \bibinfo{pages}{105--111}.
%Type = Inproceedings
\bibitem[{King(2001)}]{king2001space}
\bibinfo{author}{King, D.}, \bibinfo{year}{2001}.
\newblock \bibinfo{title}{Space servicing: past, present and future}, in:
  \bibinfo{booktitle}{the 6th International Symposium on Artificial
  Intelligence and Robotics \& Automation in Space: i-SAIRAS, Canadian Space
  Agency, St-Hubert, Quebec, Canada, 18-22 June 2001}, pp.
  \bibinfo{pages}{18--22}.
%Type = Misc
\bibitem[{KUKA-Roboter-GmbH(2015)}]{KUKALBRiiwa}
\bibinfo{author}{KUKA-Roboter-GmbH}, \bibinfo{year}{2015}.
\newblock \bibinfo{title}{Lbr iiwa 7 r800, lbr iiwa 14 r820 specification}.
\newblock
  \bibinfo{note}{\url{https://www.oir.caltech.edu/twiki_oir/pub/Palomar/ZTF/KUKARoboticArmMaterial/Spez_LBR_iiwa_en.pdf,Accessed:
  5-6-2023},}.
%Type = Article
\bibitem[{Lee et~al.(2005)Lee, Kim, Park, Kim and Bobrow}]{lee2005newton}
\bibinfo{author}{Lee, S.H.}, \bibinfo{author}{Kim, J.}, \bibinfo{author}{Park,
  F.C.}, \bibinfo{author}{Kim, M.}, \bibinfo{author}{Bobrow, J.E.},
  \bibinfo{year}{2005}.
\newblock \bibinfo{title}{Newton-type algorithms for dynamics-based robot
  movement optimization}.
\newblock \bibinfo{journal}{IEEE Transactions on robotics}
  \bibinfo{volume}{21}, \bibinfo{pages}{657--667}.
%Type = Book
\bibitem[{Lynch and Park(2017)}]{lynch2017modern}
\bibinfo{author}{Lynch, K.M.}, \bibinfo{author}{Park, F.C.},
  \bibinfo{year}{2017}.
\newblock \bibinfo{title}{Modern robotics}.
\newblock \bibinfo{publisher}{Cambridge University Press}.
%Type = Article
\bibitem[{Mao and Wang(2018)}]{mao2018reachable}
\bibinfo{author}{Mao, Q.}, \bibinfo{author}{Wang, S.}, \bibinfo{year}{2018}.
\newblock \bibinfo{title}{Reachable relative motion design of space manipulator
  actuated microgravity platform}.
\newblock \bibinfo{journal}{Journal of Spacecraft and Rockets}
  \bibinfo{volume}{55}, \bibinfo{pages}{1564--1576}.
%Type = Article
\bibitem[{Marsden and Weinstein(1974)}]{marsden1974reduction}
\bibinfo{author}{Marsden, J.}, \bibinfo{author}{Weinstein, A.},
  \bibinfo{year}{1974}.
\newblock \bibinfo{title}{Reduction of symplectic manifolds with symmetry}.
\newblock \bibinfo{journal}{Reports on mathematical physics}
  \bibinfo{volume}{5}, \bibinfo{pages}{121--130}.
%Type = Book
\bibitem[{Marsden et~al.(2007)Marsden, Misiolek, Ortega, Perlmutter and
  Ratiu}]{marsden2007hamiltonian}
\bibinfo{author}{Marsden, J.E.}, \bibinfo{author}{Misiolek, G.},
  \bibinfo{author}{Ortega, J.P.}, \bibinfo{author}{Perlmutter, M.},
  \bibinfo{author}{Ratiu, T.S.}, \bibinfo{year}{2007}.
\newblock \bibinfo{title}{Hamiltonian reduction by stages}.
\newblock \bibinfo{publisher}{Springer}.
%Type = Book
\bibitem[{Marsden et~al.(1990)Marsden, Montgomery and
  Ratiu}]{marsden1990reduction}
\bibinfo{author}{Marsden, J.E.}, \bibinfo{author}{Montgomery, R.},
  \bibinfo{author}{Ratiu, T.S.}, \bibinfo{year}{1990}.
\newblock \bibinfo{title}{Reduction, symmetry, and phases in mechanics}.
\newblock \bibinfo{publisher}{American Mathematical Soc.}
%Type = Article
\bibitem[{Marsden and Ratiu(1986)}]{marsden1986reduction}
\bibinfo{author}{Marsden, J.E.}, \bibinfo{author}{Ratiu, T.},
  \bibinfo{year}{1986}.
\newblock \bibinfo{title}{Reduction of poisson manifolds}.
\newblock \bibinfo{journal}{Letters in mathematical Physics}
  \bibinfo{volume}{11}, \bibinfo{pages}{161--169}.
%Type = Book
\bibitem[{Marsden and Ratiu(2013)}]{marsden2013introduction}
\bibinfo{author}{Marsden, J.E.}, \bibinfo{author}{Ratiu, T.S.},
  \bibinfo{year}{2013}.
\newblock \bibinfo{title}{Introduction to mechanics and symmetry: a basic
  exposition of classical mechanical systems}.
\newblock \bibinfo{publisher}{Springer Science \& Business Media}.
%Type = Article
\bibitem[{Mishra et~al.(2022)Mishra, Garofalo, Giordano, De~Stefano, Ott and
  Kugi}]{mishra2022reduced}
\bibinfo{author}{Mishra, H.}, \bibinfo{author}{Garofalo, G.},
  \bibinfo{author}{Giordano, A.M.}, \bibinfo{author}{De~Stefano, M.},
  \bibinfo{author}{Ott, C.}, \bibinfo{author}{Kugi, A.}, \bibinfo{year}{2022}.
\newblock \bibinfo{title}{Reduced euler-lagrange equations of floating-base
  robots: Computation, properties, \& applications}.
\newblock \bibinfo{journal}{IEEE Transactions on Robotics}
  \bibinfo{volume}{39}, \bibinfo{pages}{1439--1457}.
%Type = Inproceedings
\bibitem[{Mishra et~al.(2020)Mishra, Giordano, De~Stefano, Lampariello and
  Ott}]{mishra2020inertia}
\bibinfo{author}{Mishra, H.}, \bibinfo{author}{Giordano, A.M.},
  \bibinfo{author}{De~Stefano, M.}, \bibinfo{author}{Lampariello, R.},
  \bibinfo{author}{Ott, C.}, \bibinfo{year}{2020}.
\newblock \bibinfo{title}{Inertia-decoupled equations for hardware-in-the-loop
  simulation of an orbital robot with external forces}, in:
  \bibinfo{booktitle}{IEEE/RSJ International Conference on Intelligent Robots
  and Systems (IROS)}, pp. \bibinfo{pages}{1879--1886}.
%Type = Article
\bibitem[{Moghaddam and Chhabra(2021)}]{moghaddam2021guidance}
\bibinfo{author}{Moghaddam, B.M.}, \bibinfo{author}{Chhabra, R.},
  \bibinfo{year}{2021}.
\newblock \bibinfo{title}{On the guidance, navigation and control of in-orbit
  space robotic missions: A survey and prospective vision}.
\newblock \bibinfo{journal}{Acta Astronautica} \bibinfo{volume}{184},
  \bibinfo{pages}{70--100}.
%Type = Article
\bibitem[{Monazzah~Moghaddam and Chhabra(2023)}]{moghaddam2022}
\bibinfo{author}{Monazzah~Moghaddam, B.}, \bibinfo{author}{Chhabra, R.},
  \bibinfo{year}{2023}.
\newblock \bibinfo{title}{Singularity-free lagrange-poincar\'{e} equations on
  lie groups for vehicle-manipulator systems}.
\newblock \bibinfo{journal}{IEEE Transactions on Robotics}
  \bibinfo{volume}{40}, \bibinfo{pages}{1393--1409}.
%Type = Article
\bibitem[{Muralidharan and Emami(2019)}]{muralidharan2019rendezvous}
\bibinfo{author}{Muralidharan, V.}, \bibinfo{author}{Emami, M.R.},
  \bibinfo{year}{2019}.
\newblock \bibinfo{title}{Rendezvous and attitude synchronization of a space
  manipulator}.
\newblock \bibinfo{journal}{The Journal of the Astronautical Sciences}
  \bibinfo{volume}{66}, \bibinfo{pages}{100--120}.
%Type = Article
\bibitem[{Murray et~al.(1994)Murray, Li and Sastry}]{murray1994mathematical}
\bibinfo{author}{Murray, R.}, \bibinfo{author}{Li, Z.},
  \bibinfo{author}{Sastry, S.}, \bibinfo{year}{1994}.
\newblock \bibinfo{title}{A mathematical introduction to robotic manipulation}.
\newblock \bibinfo{journal}{CRC press, Boca Raton, FL} .
%Type = Article
\bibitem[{Murray(1997)}]{murray1997nonlinear}
\bibinfo{author}{Murray, R.M.}, \bibinfo{year}{1997}.
\newblock \bibinfo{title}{Nonlinear control of mechanical systems: A lagrangian
  perspective}.
\newblock \bibinfo{journal}{Annual Reviews in Control} \bibinfo{volume}{21},
  \bibinfo{pages}{31--42}.
%Type = Article
\bibitem[{Nenchev et~al.(1992)Nenchev, Umetani and
  Yoshida}]{nenchev1992analysis}
\bibinfo{author}{Nenchev, D.}, \bibinfo{author}{Umetani, Y.},
  \bibinfo{author}{Yoshida, K.}, \bibinfo{year}{1992}.
\newblock \bibinfo{title}{Analysis of a redundant free-flying
  spacecraft/manipulator system}.
\newblock \bibinfo{journal}{IEEE Transactions on Robotics and Automation}
  \bibinfo{volume}{8}, \bibinfo{pages}{1--6}.
%Type = Article
\bibitem[{Olguin-Diaz et~al.(2013)Olguin-Diaz, Arechavaleta, Jarquin and
  Parra-Vega}]{olguin2013passivity}
\bibinfo{author}{Olguin-Diaz, E.}, \bibinfo{author}{Arechavaleta, G.},
  \bibinfo{author}{Jarquin, G.}, \bibinfo{author}{Parra-Vega, V.},
  \bibinfo{year}{2013}.
\newblock \bibinfo{title}{A passivity-based model-free force--motion control of
  underwater vehicle-manipulator systems}.
\newblock \bibinfo{journal}{IEEE Transactions on Robotics}
  \bibinfo{volume}{29}, \bibinfo{pages}{1469--1484}.
%Type = Inproceedings
\bibitem[{Padial et~al.(2012)Padial, Hammond, Augenstein and Rock}]{Padial}
\bibinfo{author}{Padial, J.}, \bibinfo{author}{Hammond, M.},
  \bibinfo{author}{Augenstein, S.}, \bibinfo{author}{Rock, S.M.},
  \bibinfo{year}{2012}.
\newblock \bibinfo{title}{Tumbling target reconstruction and pose estimation
  through fusion of monocular vision and sparse-pattern range data}, in:
  \bibinfo{booktitle}{International Conference on Multisensor Fusion and
  Integration for Intelligent Systems (MFI), Hamburg, Germany, 13-15 Sept.
  2012}, \bibinfo{organization}{{IEEE}}. pp. \bibinfo{pages}{419--425}.
%Type = Article
\bibitem[{Park(1994)}]{park1994computational}
\bibinfo{author}{Park, F.C.}, \bibinfo{year}{1994}.
\newblock \bibinfo{title}{Computational aspects of the product-of-exponentials
  formula for robot kinematics}.
\newblock \bibinfo{journal}{IEEE Transactions on Automatic Control}
  \bibinfo{volume}{39}, \bibinfo{pages}{643--647}.
%Type = Inproceedings
\bibitem[{Park and Bobrow(1994)}]{park1994recursive}
\bibinfo{author}{Park, F.C.}, \bibinfo{author}{Bobrow, J.E.},
  \bibinfo{year}{1994}.
\newblock \bibinfo{title}{A recursive algorithm for robot dynamics using lie
  groups}, in: \bibinfo{booktitle}{Proceedings of the IEEE international
  conference on robotics and automation}, pp. \bibinfo{pages}{1535--1540}.
%Type = Article
\bibitem[{Park et~al.(1995)Park, Bobrow and Ploen}]{park1995lie}
\bibinfo{author}{Park, F.C.}, \bibinfo{author}{Bobrow, J.E.},
  \bibinfo{author}{Ploen, S.R.}, \bibinfo{year}{1995}.
\newblock \bibinfo{title}{A lie group formulation of robot dynamics}.
\newblock \bibinfo{journal}{The International journal of robotics research}
  \bibinfo{volume}{14}, \bibinfo{pages}{609--618}.
%Type = Book
\bibitem[{Prussing and Conway(1993)}]{prussing1993orbit}
\bibinfo{author}{Prussing, J.E.}, \bibinfo{author}{Conway, B.A.},
  \bibinfo{year}{1993}.
\newblock \bibinfo{title}{Orbital mechanics}.
\newblock \bibinfo{publisher}{Oxford University Press}.
%Type = Inproceedings
\bibitem[{Roa et~al.(2024)Roa, Beyer, Rodr{\'\i}guez, Stelzer, de~Stefano,
  Lutze, Mishra, Elhardt, Grunwald, Dubanchet et~al.}]{roa2024eross}
\bibinfo{author}{Roa, M.A.}, \bibinfo{author}{Beyer, A.},
  \bibinfo{author}{Rodr{\'\i}guez, I.}, \bibinfo{author}{Stelzer, M.},
  \bibinfo{author}{de~Stefano, M.}, \bibinfo{author}{Lutze, J.P.},
  \bibinfo{author}{Mishra, H.}, \bibinfo{author}{Elhardt, F.},
  \bibinfo{author}{Grunwald, G.}, \bibinfo{author}{Dubanchet, V.}, et~al.,
  \bibinfo{year}{2024}.
\newblock \bibinfo{title}{{Eross}: In-orbit demonstration of {E}uropean robotic
  orbital support services}, in: \bibinfo{booktitle}{2024 IEEE Aerospace
  Conference}, \bibinfo{organization}{IEEE}. pp. \bibinfo{pages}{1--9}.
%Type = Article
\bibitem[{Sakawa(1999)}]{sakawa1999trajectory}
\bibinfo{author}{Sakawa, Y.}, \bibinfo{year}{1999}.
\newblock \bibinfo{title}{Trajectory planning of a free-flying robot by using
  the optimal control}.
\newblock \bibinfo{journal}{Optimal Control Applications and Methods}
  \bibinfo{volume}{20}, \bibinfo{pages}{235--248}.
%Type = Article
\bibitem[{Sarkar and Podder(2001)}]{sarkar2001coordinated}
\bibinfo{author}{Sarkar, N.}, \bibinfo{author}{Podder, T.K.},
  \bibinfo{year}{2001}.
\newblock \bibinfo{title}{Coordinated motion planning and control of autonomous
  underwater vehicle-manipulator systems subject to drag optimization}.
\newblock \bibinfo{journal}{IEEE Journal of Oceanic Engineering}
  \bibinfo{volume}{26}, \bibinfo{pages}{228--239}.
%Type = Article
\bibitem[{Scheurle and J(1993)}]{scheurle1993reduced}
\bibinfo{author}{Scheurle, J.M.}, \bibinfo{author}{J}, \bibinfo{year}{1993}.
\newblock \bibinfo{title}{The reduced euler--lagrange equations}.
\newblock \bibinfo{journal}{Fields Institute Comm} \bibinfo{volume}{1},
  \bibinfo{pages}{139--164}.
%Type = Book
\bibitem[{Selig(2005)}]{selig2005geometric}
\bibinfo{author}{Selig, J.M.}, \bibinfo{year}{2005}.
\newblock \bibinfo{title}{Geometric fundamentals of robotics}.
\newblock \bibinfo{publisher}{Springer}.
%Type = Article
\bibitem[{Sheridan(1993)}]{sheridan1993}
\bibinfo{author}{Sheridan, T.}, \bibinfo{year}{1993}.
\newblock \bibinfo{title}{Space teleoperation through time delay: review and
  prognosis}.
\newblock \bibinfo{journal}{IEEE Transactions on Robotics and Automation}
  \bibinfo{volume}{9}, \bibinfo{pages}{592--606}.
%Type = Article
\bibitem[{Shi et~al.(2022)Shi, Liu, Sun and Yue}]{shi2022coupled}
\bibinfo{author}{Shi, K.}, \bibinfo{author}{Liu, C.}, \bibinfo{author}{Sun,
  Z.}, \bibinfo{author}{Yue, X.}, \bibinfo{year}{2022}.
\newblock \bibinfo{title}{Coupled orbit-attitude dynamics and trajectory
  tracking control for spacecraft electromagnetic docking}.
\newblock \bibinfo{journal}{Applied Mathematical Modelling}
  \bibinfo{volume}{101}, \bibinfo{pages}{553--572}.
%Type = Article
\bibitem[{Sincarsin and Hughes(1983)}]{sincarsin1983gravitational}
\bibinfo{author}{Sincarsin, G.}, \bibinfo{author}{Hughes, P.},
  \bibinfo{year}{1983}.
\newblock \bibinfo{title}{Gravitational orbit-attitude coupling for very large
  spacecraft}.
\newblock \bibinfo{journal}{Celestial mechanics} \bibinfo{volume}{31},
  \bibinfo{pages}{143--161}.
%Type = Book
\bibitem[{Stramigioli(2001)}]{stramigioli2001modeling}
\bibinfo{author}{Stramigioli, S.}, \bibinfo{year}{2001}.
\newblock \bibinfo{title}{Modeling and IPC control of interactive mechanical
  systems—A coordinate-free approach}.
\newblock \bibinfo{publisher}{Springer}.
%Type = Inproceedings
\bibitem[{Stramigioli and Duindam(2007)}]{stramigioli2007geometric}
\bibinfo{author}{Stramigioli, S.}, \bibinfo{author}{Duindam, V.},
  \bibinfo{year}{2007}.
\newblock \bibinfo{title}{On geometric dynamics of rigid multi-body systems},
  in: \bibinfo{booktitle}{PAMM: Proceedings in Applied Mathematics and
  Mechanics}, \bibinfo{organization}{Wiley Online Library}. pp.
  \bibinfo{pages}{3030001--3030002}.
%Type = Article
\bibitem[{Stramigioli et~al.(2002)Stramigioli, Maschke and
  Bidard}]{stramigioli2002geometry}
\bibinfo{author}{Stramigioli, S.}, \bibinfo{author}{Maschke, B.},
  \bibinfo{author}{Bidard, C.}, \bibinfo{year}{2002}.
\newblock \bibinfo{title}{On the geometry of rigid-body motions: The relation
  between lie groups and screws}.
\newblock \bibinfo{journal}{Proceedings of the Institution of Mechanical
  Engineers, Part C: Journal of Mechanical Engineering Science}
  \bibinfo{volume}{216}, \bibinfo{pages}{13--23}.
%Type = Inproceedings
\bibitem[{Umetani and Yoshida(1987)}]{umetani1987continuous}
\bibinfo{author}{Umetani, Y.}, \bibinfo{author}{Yoshida, K.},
  \bibinfo{year}{1987}.
\newblock \bibinfo{title}{Continuous path control of space manipulators mounted
  on omv}, in: \bibinfo{booktitle}{Acta Astronautica}, pp.
  \bibinfo{pages}{981--986}.
%Type = Inproceedings
\bibitem[{Vafa and Dubowsky(1987)}]{vafa1987dynamics}
\bibinfo{author}{Vafa, Z.}, \bibinfo{author}{Dubowsky, S.},
  \bibinfo{year}{1987}.
\newblock \bibinfo{title}{On the dynamics of manipulators in space using the
  virtual manipulator approach}, in: \bibinfo{booktitle}{Proceedings of the
  IEEE Int. Conf. on Robotics and Automation (ICRA)},
  \bibinfo{address}{Raleigh, NC, USA}. pp. \bibinfo{pages}{579--586}.
%Type = Inproceedings
\bibitem[{Viswanathan et~al.(2012)Viswanathan, Sanyal and
  Holguin}]{viswanathan2012dynamics}
\bibinfo{author}{Viswanathan, S.P.}, \bibinfo{author}{Sanyal, A.},
  \bibinfo{author}{Holguin, L.}, \bibinfo{year}{2012}.
\newblock \bibinfo{title}{Dynamics and control of a six degrees of freedom
  ground simulator for autonomous rendezvous and proximity operation of
  spacecraft}, in: \bibinfo{booktitle}{AIAA guidance, navigation, and control
  conference}, p. \bibinfo{pages}{4926}.
%Type = Article
\bibitem[{Wee et~al.(1997)Wee, Walker and McClamroch}]{wee1997articulated}
\bibinfo{author}{Wee, L.B.}, \bibinfo{author}{Walker, M.W.},
  \bibinfo{author}{McClamroch, N.H.}, \bibinfo{year}{1997}.
\newblock \bibinfo{title}{An articulated-body model for a free-flying robot and
  its use for adaptive motion control}.
\newblock \bibinfo{journal}{IEEE Transactions on Robotics and Automation}
  \bibinfo{volume}{13}, \bibinfo{pages}{264--277}.
%Type = Article
\bibitem[{Xu et~al.(2024)Xu, Hu, Song, Zhao, Zhang and Deng}]{xu2024orbit}
\bibinfo{author}{Xu, M.}, \bibinfo{author}{Hu, W.}, \bibinfo{author}{Song, Y.},
  \bibinfo{author}{Zhao, B.}, \bibinfo{author}{Zhang, P.},
  \bibinfo{author}{Deng, Z.}, \bibinfo{year}{2024}.
\newblock \bibinfo{title}{Orbit-attitude-rotating coupling dynamics of space
  manipulator assembled with camera}.
\newblock \bibinfo{journal}{Journal of Vibration Engineering \& Technologies} ,
  \bibinfo{pages}{1--14}.
%Type = Article
\bibitem[{Zhang et~al.(2013)Zhang, Ye, Jiang, Zhu, Ji and Hu}]{zhang2013output}
\bibinfo{author}{Zhang, W.}, \bibinfo{author}{Ye, X.}, \bibinfo{author}{Jiang,
  L.}, \bibinfo{author}{Zhu, Y.}, \bibinfo{author}{Ji, X.},
  \bibinfo{author}{Hu, X.}, \bibinfo{year}{2013}.
\newblock \bibinfo{title}{Output feedback control for free-floating space
  robotic manipulators base on adaptive fuzzy neural network}.
\newblock \bibinfo{journal}{Aerospace Science and Technology}
  \bibinfo{volume}{29}, \bibinfo{pages}{135--143}.
%Type = Article
\bibitem[{Zong and Emami(2020)}]{zong2020concurrent}
\bibinfo{author}{Zong, L.}, \bibinfo{author}{Emami, M.R.},
  \bibinfo{year}{2020}.
\newblock \bibinfo{title}{Concurrent base-arm control of space manipulators
  with optimal rendezvous trajectory}.
\newblock \bibinfo{journal}{Aerospace Science and Technology}
  \bibinfo{volume}{100}, \bibinfo{pages}{105822}.
%Type = Article
\bibitem[{Zong et~al.(2020)Zong, Luo and Wang}]{zong2020optimal}
\bibinfo{author}{Zong, L.}, \bibinfo{author}{Luo, J.}, \bibinfo{author}{Wang,
  M.}, \bibinfo{year}{2020}.
\newblock \bibinfo{title}{Optimal concurrent control for space manipulators
  rendezvous and capturing targets under actuator saturation}.
\newblock \bibinfo{journal}{IEEE Transactions on Aerospace and Electronic
  Systems} \bibinfo{volume}{56}, \bibinfo{pages}{4841--4855}.
%Type = Article
\bibitem[{Zong et~al.(2021)Zong, Luo and Wang}]{zong2021optimal}
\bibinfo{author}{Zong, L.}, \bibinfo{author}{Luo, J.}, \bibinfo{author}{Wang,
  M.}, \bibinfo{year}{2021}.
\newblock \bibinfo{title}{Optimal detumbling trajectory generation and
  coordinated control after space manipulator capturing tumbling targets}.
\newblock \bibinfo{journal}{Aerospace Science and Technology}
  \bibinfo{volume}{112}, \bibinfo{pages}{106626}.

\end{thebibliography}

%\begin{appendices}
\appendix
\section{Proof of Lemma \ref{orbitdyn}}\label{app:orbitkinematics}
For an orbit with true anomay $\theta$, the Eccentric anomaly is found from $E=2 tan^{-1}\left(\frac{1+e_\orbit}{1-e_\orbit} tan(\frac{1}{2}\theta)\right)$\citep{curtis2013orbital}.
The orbital parameter \textit{Mean Anomaly} $\mathcal{MA}_\orbit(t)$ corresponds to the fraction of the orbital period $T_\orbit$ that has elapsed since the space-manipulator system passed the periapsis of the orbit. The period $T_\orbit$ is the time it takes for the space manipulator to cross the same position relative the Earth-centered frame $\oplus$. Thus,
\begin{align}
    \mathcal{MA}_\orbit(t)=\mathcal{MA}_\orbit(t=0)+\frac{\sqrt{\mu_\orbit^3}}{\mathcal{G}m_\oplus}t,\label{meananomaly}
\end{align}
where $t$ represents the time past the initiation of the mission/model, $\mathcal{G}m_\oplus$ is the gravitational constant of the Earth, commonly approximated with $\mathcal{G}m_\oplus=3.986004418\times 10^{14} m^3s^{-2}$, $\mu_\orbit$ is the constant angular momentum of the orbit as defined in Section \ref{sec:orbitmotion}, and $\mathcal{MA}_\orbit(t=0))$ is the Mean Anomaly at the beginning of the mission. In orbital mechanics, this Mean Anomaly $\mathcal{MA}_\orbit(t)$ is related to the orbit eccentricity $e_\orbit$ and eccentric anomaly $E$ through the Mean Anomaly map $\textbf{MA}(E): \mathbb{R} \rightarrow \mathbb{R}$ as $\mathcal{MA}_\orbit(t)=\textbf{MA}(E)=E-e_\orbit \sin(E)$.
By defining $\textbf{MA}^{-1}(t)$ as the inverse of this map:
\begin{equation}
    E(t)=\textbf{MA}^{-1}(\mathcal{MA}(t)). \label{eccentricanomaly}
\end{equation}
which at each time $t$ provides the Eccentric Anomaly. Finally, the true anomaly of the undisturbed orbit is found from \citep{curtis2013orbital}:
\begin{equation}
    \theta(t)=2 \tan^{-1}\Big(\frac{1-e_\orbit}{1+e_\orbit}\tan(\frac{E(t)}{2})\Big).\label{trueanomaly}
\end{equation}
Knowing the initial true anomaly $\bar{\theta}$ and as a result the initial Eccentric anomaly $E(\theta)$ and Mean Anomaly $\mathcal{MA}(t=0)=E(\bar{\theta}-e_\orbit sin(E(\bar{\theta}))$, by substituting $\mathcal{MA}(t)$ from \eqref{meananomaly} into \eqref{eccentricanomaly} and the resulting $E(t)$ from \eqref{eccentricanomaly} into \eqref{trueanomaly}, we can find \eqref{noninertial:trueanomalyevolution}.\qed

\section{Proof of Lemma \ref{orbitkinlemma}} \label{app:relposevel}
Knowing $\theta(t)$ at each instant, the absolute pose of the orbital frame ${g}^\oplus_\orbit(\theta) \in \textbf{SE}(2)$ and the quasi-inertial frame ${g}^\oplus_I \in \textbf{SE}(2)$ with respect to the perifocal inertial frame is found as \citep{curtis2013orbital}:
\begin{equation}
    {g}^\oplus_\orbit(\theta)=\begin{bmatrix} R^\oplus_\orbit(\theta) & {}^\oplus p^\oplus_\orbit(\theta)\\ \mathbb{O}_{1\times3} & 1\end{bmatrix} \quad \& \quad {g}^\oplus_I=\begin{bmatrix} R^\oplus_I(\bar{\theta}) & {}^\oplus p^\oplus_I(\bar{\theta})\\ \mathbb{O}_{1\times3} & 1\end{bmatrix}, %=\iota^\oplus_\orbit\breve{g}^\oplus_\orbit(\theta) %\iota^\oplus_\orbit
\end{equation}
where %the inclusion map $\iota^\oplus_\orbit$ was specified in \eqref{inclusionmaporbit}, 
$R^\oplus_\orbit(\theta)={0.65cm}{\begin{bmatrix}\cos(\theta) & \sin(\theta)\\ -\sin(\theta) & \cos(\theta)\end{bmatrix}}$ is the relative rotation of the orbital frame to the perifocal frame, the linear position ${}^\oplus p^\oplus_\orbit(\theta)=\frac{\mu_\odot ^2/ \mathcal{G}m_\oplus}{1+e_\odot \cos(\theta )}\begin{bmatrix}\cos(\theta)& \sin(\theta)\end{bmatrix}^T$ is found from the well-known Kepler's equations \citep{curtis2013orbital} for the evolution of the orbital radius \citep{curtis2013orbital}. %is found from \eqref{orbitp}, \eqref{orbitp1}, \eqref{orbitp2}, and \eqref{orbitp3}, and  ${}^\oplus r^\oplus_\orbit(\theta)$ in \eqref{orbitp3}  as:
Thus, the absolute pose of the orbital frame with respect to the quasi-inertial frame ${g}^I_\orbit(\theta) \in \textbf{se}(2)$ is found as \citep{curtis2013orbital}:
\begin{equation}
    {g}^I_\orbit(\theta)=(g^\oplus_I(t))^{-1}g^\oplus_\orbit(\theta)=\begin{bmatrix} (R^\oplus_I)^{-1}R^\oplus_\orbit(\theta) & {}^\orbit p^I_\orbit(\theta)\\ \mathbb{O}_{1\times2} & 1\end{bmatrix}, %=\iota^\oplus_\orbit\breve{g}^\oplus_\orbit(\theta) %\iota^\oplus_\orbit
\end{equation}
where
%\begin{equation*} {}^\oplus p^I_\orbit(\theta)=\frac{\mu_\odot ^2/ \mathcal{G}m_\oplus}{1+e_\odot \cos(\theta )}\begin{bmatrix}\cos(\theta)-\cos(\bar{\theta})& \sin(\theta)-\sin(\bar{\theta})&0\end{bmatrix}^T\end{equation*}Therefore,
\begin{align}
    {}^\orbit p^I_\orbit(\theta)&=R^\oplus_\orbit({}^\oplus p^\oplus_\orbit(\theta)-{}^\oplus p^\oplus_I(t))\nonumber\\
    &= R^\oplus_\orbit({}^\oplus p^\oplus_\orbit(\theta)-{}^\oplus p^\oplus_I(\bar{\theta})-\bar{v} t)\nonumber\\
    &=(\frac{\mu_\odot ^2}{\mathcal{G}m_\oplus})\begin{bmatrix}\cos(\theta) & \sin(\theta) \\ -\sin(\theta) & \cos(\theta) \end{bmatrix}\nonumber\\
    &~~~~~~~~~~~~~~~\begin{bmatrix}\frac{\cos(\theta)}{1+e_\odot \cos(\theta )}-\frac{\cos(\bar{\theta})}{1+e_\odot \cos(\bar{\theta} )}-\bar{v}_x t\\ \frac{\sin(\theta)}{1+e_\odot \cos(\theta )}-\frac{\sin(\bar{\theta})}{1+e_\odot \cos(\bar{\theta} )}-\bar{v}_y t\end{bmatrix}%\nonumber\\
    %&=(\frac{\mu_\odot ^2}{\mathcal{G}m_\oplus})\begin{bmatrix}\frac{1}{1+e_\odot \cos(\theta )}-\frac{\cos(\theta-\bar{\theta})}{1+e_\odot \cos(\bar{\theta} )}\\ \frac{\sin(\theta-\bar{\theta})}{1+e_\odot \cos(\bar{\theta} )}\end{bmatrix}
\end{align}
where $\bar{v}={}^\oplus v^\oplus_I={}^\oplus v^\oplus_\orbit(t=0)=\begin{bmatrix}\bar{v}_x&\bar{v}_y\end{bmatrix}^T$ is constant perifocal velocity of the quasi-inertial frame.
%with $P_\orbit$ the locked an undisturbed orbital momentum in \eqref{noninertial:orbitalmomentum}, $e_\orbit$ the eccentricity (a geometric property describing the degree of ellipticity of the orbit), $\mathcal{G}m_\oplus=GM_\oplus$ as the gravitational constant of the earth, $G$ as the global gravitaional constant, and $M_\oplus$ as the mass of the Earth.
%Since a closed-form exact solution for $g^I_\orbit$ exists, the variation $\delta g^I_\orbit$ will be nonexistent. Thus:\begin{equation*}\eta_\orbit=(g^I_\orbit)^{-1}\delta g^I_\orbit=0.\end{equation*}
Similar to the configuration of the orbit, the orbital velocity of an undisturbed orbit with the constant orbital momentum $\mu_\orbit$ with respect to the Earth-fixed inertial frame represented in the orbital frame ${}^\oplus \V^\oplus_\orbit=\iota_\orbit^T{}^{\oplus}V^{\oplus}_\orbit$ is a function of the true anomaly $\theta$, which itself is exclusively dependent on the time parameter $t$ through \eqref{noninertial:trueanomalyevolution}:
\begin{align}
     &{}^{\oplus}\V_{\orbit}^{\oplus}=\frac{\mathcal{G}m_\oplus}{\mu_\orbit}\begin{bmatrix}  - \sin(\theta)\\e_\orbit +\cos(\theta)\\\frac{\mathcal{G}m_\oplus (1+e_\orbit \cos(\theta))^2}{\mu_\orbit^2}\end{bmatrix} \nonumber\\
     &{}^{\orbit}\V_{\orbit}^{\oplus}=\frac{\mathcal{G}m_\oplus}{\mu_\orbit}\begin{bmatrix}  e_\orbit \sin(\theta)\\1+e_\orbit\cos(\theta)\\\frac{\mathcal{G}m_\oplus (1+e_\orbit \cos(\theta))^2}{\mu_\orbit^2}\end{bmatrix},
\end{align}
where $e_\orbit$ is the eccentricity of the orbit's ellipse, $\mathcal{G}m_\oplus$ is the gravitational constant of the Earth \citep{curtis2013orbital}, and the angular velocity of the orbital frame $\dot{\theta}=\frac{\mu_\orbit}{|\oplus p^\oplus_\orbit|^2}=\frac{(\mathcal{G}m_\oplus)^2 (1+e_\orbit \cos(\theta))^2}{\mu_\orbit^3}$\citep{curtis2013orbital} is the rate of change of the true anomaly .
The quasi-inertial frame $I$ moves with a constant linear velocity with respect to the perifocal frame $\oplus$ equal to the initial linear velocity of the orbital frame ${}^\oplus v_I^\oplus={}^{\oplus}\bar{v}^{\oplus}_\orbit={}^{\oplus}{v}^{\oplus}_\orbit(t=0)$ and ${}^\oplus w_I^\oplus=0$:
\begin{align}
    %{}^I\V^\oplus_I=\frac{\mathcal{G}m_\oplus}{\mu_\orbit}\begin{bmatrix}  1+e_\orbit\cos(\bar{\theta})\\e_\orbit \sin(\bar{\theta})\\0\end{bmatrix} \quad \& \quad
    {}^\orbit \V^\oplus_I&=R_{z_\orbit}(\theta-\bar{\theta}){}^I\V^\oplus_I\nonumber\\
    %&=\frac{\mathcal{G}m_\oplus}{\mu_\orbit}\begin{bmatrix}\cos(\theta-\bar{\theta})& \sin(\theta-\bar{\theta})& 0\\ -\sin(\theta-\bar{\theta})& \cos(\theta-\bar{\theta})& 0\\ 0& 0& 1\end{bmatrix}\begin{bmatrix}  e_\orbit \sin(\bar{\theta})\\1+e_\orbit\cos(\bar{\theta})\\0\end{bmatrix}
&=\frac{\mathcal{G}m_\oplus}{\mu_\orbit}\begin{bmatrix}  cos(\theta)+e_\orbit\cos(\theta-\bar{\theta})\\ -\sin(\theta)-e_\orbit \sin(\theta-\bar{\theta})\\0\end{bmatrix}
\end{align}
The orbit's velocity with respect to the quasi-inertial frame can be found:
\begin{equation}
    {}^\oplus\V_{\orbit}^I(\theta)={}^\oplus\V^\oplus_\orbit(\theta)-{}^\oplus\V_I^\oplus=\frac{\mathcal{G}m_\oplus }{\mu_\orbit}\begin{bmatrix}  -\sin(\theta)+\sin(\bar{\theta})\\ \cos(\theta)-\cos(\bar{\theta})\\\frac{\mathcal{G}m_\oplus} {\mu_\orbit^2}(1+e_\orbit \cos(\theta))^2\end{bmatrix}, \label{orbitvel2}
\end{equation}
\begin{align}
    {}^\orbit\V_{\orbit}^I(\theta)&=R_{z_\orbit}(\theta)\left({}^\oplus\V^\oplus_\orbit(\theta)-{}^\oplus\V_I^\oplus\right)\nonumber\\
    &=\frac{\mathcal{G}m_\oplus }{\mu_\orbit}\begin{bmatrix}  -\sin(\theta-\bar{\theta})\\ 1-\cos(\theta-\bar{\theta})\\\frac{\mathcal{G}m_\oplus} {\mu_\orbit^2}(1+e_\orbit \cos(\theta))^2\end{bmatrix}. \label{orbitvel3}
\end{align} \qed
\section{Proof of Lemma \ref{noninertial:lemmaPorbit}}\label{app:lemmaPorbit}
The  momentum $P_\orbit=M_0 \Ad_{g^\orbit_0}\iota_0 {}^\orbit\V^I_\orbit$ in Equation \eqref{noninertial:momentum} is found from:
\begin{align}
    P_\orbit
    &=M_0  \frac{\mathcal{G}m_\oplus }{\mu_\orbit}\Ad_{g^\orbit_0}\iota_\orbit\begin{bmatrix}  -\sin(\theta-\bar{\theta})\\ 1-\cos(\theta-\bar{\theta})\\\frac{\mathcal{G}m_\oplus} {\mu_\orbit^2}(1+e_\orbit \cos(\theta))^2\end{bmatrix}.
\end{align}
Taking the derivative of the above equation results:
\begin{align}
    \dot{P}_\orbit(\theta,q_\m)&=\frac{\mathcal{G}m_\oplus}{\mu_\orbit}\Bigg(\frac{dM_0}{dt}\Ad_{g^\orbit_0}\iota_\orbit \V_\orbit\nonumber\\
    &~+M_0\frac{d(\Ad_{g^\orbit_0})}{dt}\iota_\orbit \V_\orbit+M_0\Ad_{g^\orbit_0}\iota_\orbit \dot{\V}_\orbit\Bigg)\nonumber\\
    &=\frac{\mathcal{G}m_\oplus}{\mu_\orbit}\Bigg(\bigg(\sum^n_{i=1}\frac{\partial M_0}{\partial q_i}\dot{q}_i\bigg)\Ad_{g^\orbit_0}\iota_\orbit \V_\orbit\nonumber\\
    &~+M_0\left(\ad_{V_0}\Ad_{g^\orbit_0}\right)\iota_\orbit \V_\orbit\nonumber\\
    &+M_0\Ad_{g^\orbit_0}\iota_\orbit \frac{\mathcal{G}m_\orbit}{\mu_\orbit^2}\left(1+e \cos (\theta)\right)^2\nonumber\\
    &~~~~~~~~~~~~~~~~~~~\begin{bmatrix}
        -\cos(\theta-\bar{\theta})\\1+\sin(\theta-\bar{\theta})\\ -\frac{2 \mathcal{G}m_\oplus
        e_\orbit sin(\theta)(1+e_\orbit cos(\theta))}{\mu_\orbit^2}
        \end{bmatrix}\Bigg) \label{app:Pdot}
\end{align}
This completes the proof. \qed
%------------------------------------------------------------------------------------
\section{Proof of Lemma \ref{noninertial:varlemma}}\label{app:variations}
To show the equivalence of the first terms appearing in \eqref{noninertial:hamiltonforL} and \eqref{noninertial:hamiltonforlcm}, we need to compute the variation $\delta (\OOs)$ induced on $\q^\vee_0 \times T_{q_\m}Q_\m $ by the variation $\delta g^\orbit_0.$ This variation can be found from the chain rule as $\delta (\OOs)= \iota_\orbit \delta \V_\orbit + \delta V_0+ \delta (\A\dot{q}_\mathfrak{m})$, using the definition in \eqref{noninertial:lockedorbitalvelocitywrtorbit}.
We know from \citep{moghaddam2022} that for $\delta V_0=\delta \left((g^\orbit_0)^{-1}\dot{g}^I_\orbit \right)$ we have $\delta \hat{V}_0 =(\dot{{\eta_0}}+ \ad_{V_0} {\eta_0})^\wedge$, where the $\ad_{V_0}: \q^0_0 \rightarrow \q^0_0$ operator for the Lie algebra $\q_0^0 \in \g^0_0$ is defined in \citep{moghaddam2022}. Knowing 
that $\delta \V_\orbit=0$ and $\delta g^I_\orbit=0$ for a known kinematic evolution of the orbital frame in \eqref{orbitdyn} and \eqref{orbitvel}, we can find $\delta (\OOs)= \dot{\eta_0}+\ad_{V_0} {\eta_0} + \A\delta \dot{q}_\m+\delta \A\dot{q}_\m$.
We also know from \citep{moghaddam2022} that
\begin{align*}
        \left<F_0,\delta g_0^\orbit\right>&=\left<F_{0},T_{\mathbb{I}}L_{g^\orbit_0}\eta_0\right>\\
        &=\left<T_{\mathbb{I}}^*L_{g_0^\orbit}F_{0},\eta_0\right>=\left<F_{\eta_0},\eta_0\right>
%        \label{noninertial:F0}
\end{align*}
This completes the proof. \qed

\section{Proof of Theorem \ref{noninertialtheorem:EoM}} \label{app:prooftheorem}
The EoM of a space-manipulator system with the Lagrangian in \eqref{noninertial:lagrangianfirst} between two fixed configurations $q_0=q(t_0)$ and $q_f=q(t_F)$ under the effect of the applied force $F=(F_0,F_m)\in T^*Q$ can be formed using the Lagrange-d'Alambert principle
for a variation of the type $\delta q$ with fixed endpoints, i.e., $\delta q(t_0)=\delta q(t_f)=0$. 
Based on Lemma \ref{noninertial:varlemma}, the reduced Lagrangian $\ell^{cm}$ then has to satisfy the Lagrange-d'Alembert principle in \eqref{noninertial:hamiltonforlcm} for variations defined in \eqref{noninertial:varomega0} and forces defined in \eqref{noninertial:f.mhat} \citep{blochbook}.
Using the chain rule, and recognizing that $l^{cm}=k^{cm}(\OOs,q_\m,\dot{q}_\m)-u^{cm}(q_\m) $, the Lagrange-d'Alambert principle in \eqref{noninertial:hamiltonforlcm} can be re-written as:
\begin{align}
    &\int \Bigg(\underbrace{\left<\frac{\partial k^{cm}}{\partial\OOs}, \delta (\OOs)\right>}_{(I)}\nonumber\\
    & ~~
    +\underbrace{\left<\frac{\partial k^{cm}}{\partial q_\mathfrak{m}}, \delta q_\mathfrak{m}\right>+\left<\frac{\partial k^{cm}}{\partial\dot{q}_\mathfrak{m}}, \delta \dot{q}_\mathfrak{m}\right>}_{(II)}\Bigg)dt\nonumber\\
    &-\int (\underbrace{\left<F_{\eta_0},\eta_0\right>}_{(III)} +\underbrace{\left<F_\m+\frac{\partial u}{\partial q_\m},\delta q_\m\right>}_{(IV)})dt=0, %+ \frac{dU}{dg^I_0} g^I_0\delta\eta
        \label{noninertial:variationequation}
\end{align}
The term $(I)$ in \ref{noninertial:variationequation} can been expanded as
\begin{align}
    (I) %&= \int   \left(\left<P_\orbit+P_0,\dot{\eta}_0+ \ad_{V_0} \eta_0\right> + \left<P_\orbit+P_0,\delta (\A \dot{q}_\m)\right>\right)dt\nonumber\\
    =\int   \Bigg(&\underbrace{\left<P_0+P_\orbit,\dot{\eta}_0+ \ad_{V_0} \eta_0\right>}_{(I.A)} \nonumber\\
    &+\underbrace{\left<P_0+P_\orbit,\delta (\A \dot{q}_\m)\right>}_{(I.B)}\Bigg)dt %+\underbrace{\left<P_\orbit,\delta (\A \dot{q}_\m)\right>}_{(I.C)}
    \label{noninertial:I}
\end{align}
As detailed in \citep{moghaddam2022} the term $(I.A)$ can be expanded as: %is expanded using the definition of the variation $\delta \OOs$ in  Lemma \ref{noninertial:varlemma}:
\begin{equation}
    (I.A) = \int   \left<\underbrace{-\dot{P}_0+\textbf{ad}^T_{V_0}P_0}_{(I.A.1)}\underbrace{-\dot{P}_\orbit+\textbf{ad}^T_{V_0}P_\orbit}_{\hat{F}_\orbit},\eta_0\right>%+\A^T F_\orbit
    %\left<\frac{\partial k^{cm}}{\partial (\OOs)},\delta (\A\dot{q}_\m)\right>
    dt.
    \label{noninertial:I.A}
\end{equation}
Collecting the term $(I.A.1)$  from \eqref{noninertial:I.A} and the disturbing terms $\hat{F}_\orbit$ and $F_{\eta_0}$ from \eqref{noninertial:I.A} and \eqref{noninertial:variationequation}, respectively,  and noticing the arbitrariness of variation $\eta_0$ 
provides the Euler-Poincar\'{e} formula for the dynamics of the spacecraft in \eqref{noninertial:EoM}.
Similarly, the term $(I.B)$ can be expanded as detailed in \citep{moghaddam2022}: 
\begin{align}
(I.B)%&=\int\left(\left<P_0+P_\orbit,\delta\A \dot{q}_\m+\A \delta \dot{q}_\m\right>\right)dt\nonumber\\
    &=\int\left(\left<P_0+P_\orbit,\delta\A \dot{q}_\m+\A \delta \dot{q}_\m\right>\right)dt\nonumber\\
    %&=\int\left(\left<P_0+P_\orbit,\delta\A \dot{q}_\m\right>+\left<\A^T(P_0+P_\orbit),\delta \dot{q}_\m\right>\right)dt\nonumber\\
    &=\int\Bigg(\underbrace{\left<P_0+P_\orbit,\delta\A \dot{q}_\m\right>}_{(I.B.1)}\nonumber\\
    &~~~~~~~-\underbrace{\left<\dot{\A}^T(P_0+P_\orbit),\delta q_\m\right>}_{(I.B.2)}\nonumber\\
    & ~~~~~~~-\underbrace{\left<\A^T(\dot{P}_0+\dot{P}_\orbit),\delta q_\m\right>}_{(I.B.3)}\Bigg)dt
\end{align}
where again, the terms $(I.B.1)$, $(I.B.2)$, and $(I.B.3)$ can be expanded according to \citep{moghaddam2022} as:
\begin{align}
    (I.B.1)&=\int \left<P_0+P_\orbit,\begin{bmatrix}\frac{\partial \A}{\partial q_1} \delta q_\m \cdots \frac{\partial \A}{\partial q_n} \delta q_\m\end{bmatrix}\dot{q}_\m\right>dt\nonumber\\
    &=\int \left<P_0+P_\orbit,\begin{bmatrix}\frac{\partial \A}{\partial q_1} \dot{q}_\m \cdots \frac{\partial \A}{\partial q_n} \dot{q}_\m\end{bmatrix}\delta {q}_\m\right>dt\nonumber\\
    &=\int \left<\begin{bmatrix}\dot{q}_\m ^T(\frac{\partial \A}{\partial q_1})^T \\\vdots \\\dot{q}_\m^T(\frac{\partial \A}{\partial q_n})^T \end{bmatrix}(P_0+P_\orbit),\delta q_\m\right>dt\\
    %&=\int \left<\begin{bmatrix}\dot{q}_\m^T\frac{\partial \A}{\partial q_1}^T \\ \vdots \\\dot{q}_\m^T\frac{\partial \A}{\partial q_n}^T\end{bmatrix}P_0+\begin{bmatrix}\dot{q}_\m^T\frac{\partial \A}{\partial q_1}^T \\ \vdots \\\dot{q}_\m^T\frac{\partial \A}{\partial q_n}^T\end{bmatrix}P_\orbit,\delta q_\m\right>dt
    (I.B.2)%&=\int \left<-\left(\sum_{i=1}^n (\frac{\partial \A}{\partial q_i}) \dot{q}_i\right)^T(P_0+P_\orbit),\delta q_\m\right>dt\nonumber\\
    &=\int \Bigg<-\left(\sum_{i=1}^n (\frac{\partial \A}{\partial q_i}) \dot{q}_i\right)^TP_0\nonumber\\
    &~~~~~~~~~~~~~~-\left(\sum_{i=1}^n (\frac{\partial \A}{\partial q_i}) \dot{q}_i\right)^TP_\orbit,\delta q_\m\Bigg>dt\\
    (I.B.3)%&=-\int\left<\A^T(\dot{P}_0+\dot{P}_\orbit),\delta q_\m\right>dt\nonumber\\
    %&= -\int\left<\A^T(\textbf{ad}^T_{V_0}P_0+\textbf{ad}^T_{V_0}P_\orbit-F_\eta),\delta q_\m\right>dt\nonumber\\
    &=\int\big<-\A^T\textbf{ad}^T_{V_0}P_0 - \A^T\textbf{ad}^T_{V_0}P_\orbit \nonumber\\
    &~~~~~~~~~~+\A^TF_\eta,\delta q_\m\big>dt
\end{align}
%\subsection[Mass, Coriolis, and Coupling]{Calculation of Mass Matrix $\hat{M}_\m$, Coriolis matrix $\hat{C}_\m$ and Coupling matrix $\hat{N}_\m$}
The term $(II)$ in \eqref{noninertial:variationequation} has also been rigorously expanded in \citep{moghaddam2022} as:
\begin{align}
    (II) %&=\int \left<\frac{\partial K_{cm}}{\partial q_\m}-\frac{d}{dt}\frac{\partial K_{cm}}{\partial \dot{q}_\m},\delta q_\m\right>dt \nonumber\\
    =&\int \big<\frac{1}{2}\frac{d}{dt}\left(\dot{q}_\m^T \hat{M}_\mathfrak{m}\dot{q}_\mathfrak{m}\right)\nonumber\\
    &~~~~~~~~~-\frac{1}{2}\frac{d}{dt}\left((P_0+P_\orbit)^T {M}_0^{-1}(P_0+P_\orbit)\right)\nonumber\\
    &~~~~~~~~~~~~~~~~~~~~~~~~~~+\dot{\hat{M}}_\m \dot{q}_\m+\hat{M}_\m \ddot{q}_\m,\delta q_\m\big>dt\nonumber\\
    =&\int \Bigg<\underbrace{\frac{1}{2}\begin{bmatrix}\frac{\partial \hat{M}_\mathfrak{m}}{\partial q_1}\dot{q}_\m & \cdots & \frac{\partial \hat{M}_\mathfrak{m}}{\partial q_n}\dot{q}_\m\\\end{bmatrix}^T}_{\hat{C}_1}\dot{q}_\mathfrak{m}\nonumber\\
    &~~~~~~~-\underbrace{\frac{1}{2}\begin{bmatrix}(P_0+P_\orbit)^T\frac{\partial {M}_0^{-1}}{\partial q_1}\\\vdots \\ (P_0+P_\orbit)^T\frac{\partial {M}_0^{-1}}{\partial q_n}\end{bmatrix}(P_0+P_\orbit)}_{\hat{N}_1}\nonumber\\
   & ~~~~~~~~~+\underbrace{\left(\sum_{i=1}^{n}\frac{\partial \hat{M}_\mathfrak{m}}{\partial{q}_i}\dot{q}_i\right)}_{\hat{C}_2}+\hat{M}_\m \ddot{q}_\m,\delta q_\m\Bigg>dt\nonumber\\
    &=\int \Big<{\hat{M}}_\mathfrak{m}\ddot{q}_\mathfrak{m}\nonumber\\
    &~~~~~~~~~~+\underbrace{\left((\sum_{i=1}^{n}\frac{\partial \hat{M}_\mathfrak{m}}{\partial{q}_i}\dot{q}_i)+\frac{1}{2}\begin{bmatrix}\dot{q}_\m^T\frac{\partial \hat{M}_\mathfrak{m}}{\partial q_1}^T\\ \vdots \\ \dot{q}_\m^T\frac{\partial \hat{M}_\mathfrak{m}}{\partial q_n}^T\\\end{bmatrix}\right)}_{\hat{C}_\m}\dot{q}_\mathfrak{m}\nonumber\\
    &~~~~~~~~~~-\underbrace{\frac{1}{2}\begin{bmatrix}\frac{\partial \hat{M}_0^{-1}}{\partial q_1}P_0\cdots \frac{\partial \hat{M}_0^{-1}}{\partial q_n}P_0\\\end{bmatrix}^TP_0}_{\hat{N}_{0,2}}\nonumber\\
    &~~~~~~~~~~-\underbrace{\frac{1}{2}\begin{bmatrix}\frac{\partial {M}_0^{-1}}{\partial q_1}P_\orbit\cdots \frac{\partial {M}_0^{-1}}{\partial q_n}P_\orbit\\\end{bmatrix}^TP_\orbit}_{\hat{N}_{\orbit,1}}\nonumber\\
    &~~~~~~~~~~ -\underbrace{\begin{bmatrix}\frac{\partial {M}_0^{-1}}{\partial q_1}P_0\cdots \frac{\partial {M}_0^{-1}}{\partial q_n}P_0\\\end{bmatrix}^TP_\orbit}_{\hat{N}_{\orbit,2}},\delta q_\m\Big>dt\label{noninertial:N2}
\end{align}
%\begin{multline}(II)=\int \Big<{\hat{M}}_\mathfrak{m}\ddot{q}_\mathfrak{m}+\underbrace{\left((\sum_{i=1}^{n}\frac{\partial \hat{M}_\mathfrak{m}}{\partial{q}_i}\dot{q}_i)+\frac{1}{2}\begin{bmatrix}\frac{\partial \hat{M}_\mathfrak{m}}{\partial q_1}\dot{q}_\m \cdots \frac{\partial \hat{M}_\mathfrak{m}}{\partial q_n}\dot{q}_\m\\\end{bmatrix}^T\right)}_{\hat{C}_\m}\dot{q}_\mathfrak{m}\\-\underbrace{\frac{1}{2}\begin{bmatrix}\frac{\partial \hat{M}_0^{-1}}{\partial q_1}P_0\cdots \frac{\partial \hat{M}_0^{-1}}{\partial q_n}P_0\\\end{bmatrix}^TP_0}_{\hat{N}_{0,2}}-\underbrace{\frac{1}{2}\begin{bmatrix}\frac{\partial {M}_0^{-1}}{\partial q_1}P_\orbit\cdots \frac{\partial {M}_0^{-1}}{\partial q_n}P_\orbit\\\end{bmatrix}^TP_\orbit}_{\hat{N}_{\orbit,1}},\delta q_\m\Big>dt\label{noninertial:N2}\end{multline}
Collecting the terms $\hat{M}_\m$ and $\hat{C}_\m$ here completes the calculation of $\hat{M}_\m$ and $\hat{C}$ in \eqref{noninertial:EoM} and \eqref{noninertial:coriolis}. The $\hat{N}_0$ matrix as presented in \eqref{noninertial:Nhat0} is formed by collecting the terms $\hat{N}_{0,1}$ and $\hat{N}_{0,2}$  from \eqref{noninertial:I.A} and \eqref{noninertial:N2}.
%\subsection{Calculation of the Euler-Poincar\'{e} Equation in \eqref{noninertial:EoM}} 
Similar to the term $(I.A)$, the term $(I.B)$ in \eqref{noninertial:I} can be expanded using the definition of $P_\orbit$ in \eqref{noninertial:momentum}:
\begin{align}
    (I.B) = %\int (\left<P_\orbit,\dot{\eta_0}\right>+ \left<P_\orbit,\textbf{ad}_{V_0}\eta_0\right>)dt\nonumber\\
        %&= \int \bigg(-\left<P_\orbit,\eta_0\right>+\left<\textbf{ad}^T_{V_0}P_\orbit,\eta_0\right>\bigg)dt\nonumber\\
        \int \left<\underbrace{-\dot{P}_\orbit+\textbf{ad}^T_{V_0}P_\orbit}_{\hat{F}_\orbit},\eta_0\right> dt,%-\dot{\A}_{\orbit,0}^T
        \label{noninertial:I.B}
\end{align}
where $\dot{P}_\orbit$ is given in \eqref{app:lemmaPorbit}. 
Collecting the terms $(I.A.1)$, $(I.B)$ and $(III)$ from \eqref{noninertial:I.A}, \eqref{noninertial:I.B} and \eqref{noninertial:variationequation}, respectively, expanding the forcing term based on \eqref{f.m}, and noticing the arbitrariness of variation $\eta_0$ 
provides the Euler-Poincar\'{e} formula for the dynamics of the spacecraft in the first line of \eqref{noninertial:EoM}. 
%\subsection{Calculation of $\hat{N}_\orbit$ in \eqref{noninertial:EoM}}
The term $(I.C)$ in \eqref{noninertial:I.A} is expanded similar to the derivations in \citep{moghaddam2022}: %$\hat{N}_1$ in \eqref{noninertial:I.A}:
\begin{align}
        (I.C)%&=\int \left<P_\orbit,\A\delta \dot{q}_\mathfrak{m}+\delta \A\dot{q}_\mathfrak{m}\right> dt\\
        %&=\int (\left<\A^TP_\orbit,\delta \dot{q}_\mathfrak{m}\right>+\left<P_\orbit,\delta \A\dot{q}_\mathfrak{m}\right>) dt\\
        %&= \int (-\left<\dot{P},\A\delta q_\mathfrak{m}\right>-\left<P,\dot{\A}\delta q_\mathfrak{m}\right>+\left<P, \delta \A \dot{q}_\mathfrak{m}\right>)dt\\
        &= \int (-\left<\A^T\dot{P}_\orbit,\delta q_\mathfrak{m}\right>-\left<\dot{\A}^TP_\orbit,\delta q_\mathfrak{m}\right>\nonumber\\
        &~~~~~+\left<P_\orbit, \delta \A\dot{q}_\mathfrak{m}\right>)dt\nonumber\\
        &= \int \big<-\A^T\dot{P}_\orbit-\left(\sum_{i=1}^n (\frac{\partial \A}{\partial q_i}) \dot{q}_i\right)^TP_\orbit\nonumber\\
        &~~~~~+\underbrace{\begin{bmatrix}
       \frac{\partial \A}{\partial q_1}\dot{q}_\m & \cdots &\frac{\partial \A}{\partial q_n}\dot{q}_\m\end{bmatrix}^TP_\orbit}_{\hat{N}_{\orbit,3}},\delta q_\mathfrak{m}\big>dt%\nonumber\\
       %&= \int \Bigg<\underbrace{-\A^T\ad_{V_0}^TP_\orbit-\A^T\Ad^T_{g^\orbit_0}\dot{P}^\orbit_\orbit-\left(\sum_{i=1}^n (\frac{\partial \A}{\partial q_i}) \dot{q}_i\right)^TP_\orbit}_{\hat{N}_{\orbit,2}}\nonumber\\
       %& ~~~~~~~~~~~~~~~~~~~~~~~~~~~~~~~~~~~~~~~~~~+\underbrace{\begin{bmatrix}\frac{\partial \A}{\partial q_1}\dot{q}_\m & \cdots &\frac{\partial \A}{\partial q_n}\dot{q}_\m\end{bmatrix}^TP_\orbit}_{\hat{N}_{\orbit,3}},\delta q_\mathfrak{m}\Bigg>dt.
       \label{noninertial:I.C2} %\hat{N}_2 \delta q_\m
\end{align}
where the evolution of $\dot{P}_\orbit$ can be found from Lemma \ref{noninertial:lemmaPorbit}. The $\hat{N}_\orbit$ matrix as presented in \eqref{noninertial:Nhatorbit} is found by collecting $\hat{N}_{\orbit,1}$, $\hat{N}_{\orbit,2}$, and $\hat{N}_{\orbit,3}$ from \eqref{noninertial:N2}, and \eqref{noninertial:I.C2}, respectively.
%\end{appendices}

\end{document}